\documentclass[11pt,twoside,a4paper]{article}

\usepackage[utf8]{inputenc}     %
\usepackage[T1]{fontenc}        %
\usepackage[english]{babel}     %
\usepackage{microtype}          %

\usepackage{lmodern}            

\usepackage[leqno]{mathtools}   %
\usepackage{amssymb}            %
\usepackage{amsthm}             %
\usepackage{nicefrac}           %

\usepackage{bm}                 %
\usepackage{dsfont}             %
\usepackage{bbm}                %
\usepackage{mathrsfs}           %
\usepackage{euscript}           %
\usepackage[bb=px]{mathalfa}    %

\usepackage[dvipsnames]{xcolor} %
\usepackage{geometry}
\usepackage{graphicx}
\usepackage{booktabs}           %
\usepackage{enumitem}
\setlist[enumerate]{itemsep=0pt}
\usepackage{fancyhdr}
\usepackage{titlesec}
\usepackage{authblk}            %

\usepackage[compress]{natbib}
\setcitestyle{authoryear, circle, citesep={,}, aysep={,}, yysep={,}}

\usepackage[ruled,linesnumbered]{algorithm2e}

\usepackage{tikz}
\usetikzlibrary{positioning, fit, tikzmark, calc}
\usetikzlibrary{decorations.pathreplacing, calligraphy}

\colorlet{inlinkcolor}{purple}
\colorlet{exlinkcolor}{purple}
\colorlet{citecolor}{teal!50!blue}

\usepackage{multirow,makecell,array,longtable,calc,etoolbox,float,caption,subcaption}
\usepackage[normalem]{ulem}
\AtBeginDocument{\setcounter{secnumdepth}{3}\setcounter{tocdepth}{2}}

\usepackage[colorlinks=true,linkcolor=inlinkcolor,citecolor=citecolor,urlcolor=exlinkcolor,linktoc=page,breaklinks=true,plainpages=false,pagebackref=true,backref=page]{hyperref}
\renewcommand*{\backrefalt}[4]{%
    \ifcase #1 \footnotesize{(Not cited.)}%
    \or        \footnotesize{(Cited on page~#2.)}%
    \else      \footnotesize{(Cited on pages~#2.)}%
    \fi}

\newtheorem{theorem}{Theorem}[section]
\newtheorem{proposition}[theorem]{Proposition}
\newtheorem{lemma}[theorem]{Lemma}
\newtheorem{corollary}[theorem]{Corollary}

\newtheorem{innercustomproposition}{Proposition}
\newenvironment{customproposition}[1]
{\renewcommand\theinnercustomproposition{#1}\innercustomproposition}
{\endinnercustomproposition}
\newtheorem{innercustomlemma}{Lemma}
\newenvironment{customlemma}[1]
{\renewcommand\theinnercustomlemma{#1}\innercustomlemma}
{\endinnercustomlemma}
\newtheorem{innercustomtheorem}{Theorem}
\newenvironment{customtheorem}[1]
{\renewcommand\theinnercustomtheorem{#1}\innercustomtheorem}
{\endinnercustomtheorem}

\theoremstyle{definition}
\newtheorem{definition}{Definition}
\newtheorem{assumption}{Assumption}

\theoremstyle{remark}
\newtheorem{remark}[theorem]{Remark}
\newtheorem{example}{Example}

\numberwithin{equation}{section} %

\newcommand{\setheader}[2]{%
    \pagestyle{fancy}
    \fancyhf{}
    \fancyhead[LO]{#1}
    \fancyhead[RO]{}
    \fancyhead[LE]{}
    \fancyhead[RE]{#2}
    \renewcommand{\headrulewidth}{0.4pt}
    \fancyfoot[C]{\thepage}
    \renewcommand{\headrule}{\color{gray!75}\hrule width\textwidth height 1.0pt}
}

\makeatletter
\renewcommand{\maketitle}{
    \begin{center}
        {\Large\bfseries\@title\par}
        \vskip 1.5em
        {\@author\par}
        \vskip 1em
        {\@date\par}
    \end{center}
    \vskip 1.5em
}
\makeatother

\makeatletter
\renewenvironment{abstract}{
    \small\noindent{\bfseries Abstract.}
    \ignorespaces
}{\par\vspace{0.5em}}
\makeatother

\newcommand{\keywords}[1]{
    \noindent{\small\bfseries Keywords:} {\small #1}
    \par\noindent\textcolor{gray}{\rule{\linewidth}{0.5pt}}
}

\titleformat{\section}
    {\normalfont\large\bfseries\centering}
    {\thesection}{1em}{}
\titleformat{\subsection}[runin]
    {\normalfont\normalsize\bfseries}
    {\thesubsection}{1em}{}[.]
\titleformat{\subsubsection}[runin]
    {\normalfont\normalsize\itshape}
    {\thesubsubsection}{1em}{}[.]
\titleformat{\paragraph}[runin]
    {\normalfont\normalsize\bfseries}
    {\theparagraph}{1em}{}[.]

\titlespacing{\section}
    {0pt}{2.0ex plus 0.5ex minus 0.2ex}{1.0ex plus 0.1ex}
\titlespacing{\subsection}
    {0pt}{1.0ex plus 0.1ex minus 0.1ex}{1.0ex plus 0.1ex}
\titlespacing{\subsubsection}
    {0pt}{1.0ex plus 0.1ex minus 0.1ex}{1.0ex plus 0.1ex}
\titlespacing{\paragraph}
    {0pt}{0.5ex plus 0.1ex minus 0.1ex}{0.5ex plus 0.1ex}

\def\d{\,\mathrm{d}}
\def\dt{\,\mathrm{d}t}
\def\ds{\,\mathrm{d}s}
\def\dd{\mathrm{d}}
\def\what{\widehat}
\newcommand{\given}[1][]{\mkern 2mu #1\vert \mkern 2mu }

\def\bbone{{\mathbbm{1}}}

\def\bzero{{\mathbf{0}}}

\def\vepsilon{{\bm{\varepsilon}}}

\def\vb{\mathbf{b}}

\def\ve{\mathbf{e}}

\def\vg{\mathbf{g}}

\def\vn{\mathbf{n}}

\def\vs{\mathbf{s}}

\def\vu{\mathbf{u}}
\def\vv{\mathbf{v}}

\def\vx{\mathbf{x}}
\def\vy{\mathbf{y}}
\def\vz{\mathbf{z}}
\def\mA{\mathbf{A}}
\def\mB{\mathbf{B}}

\def\mI{\mathbf{I}}

\def\mP{\mathbf{P}}

\def\mX{\mathbf{X}}
\def\mY{\mathbf{Y}}
\def\mZ{\mathbf{Z}}

\def\calA{{\mathcal{A}}}

\def\calF{{\mathcal{F}}}

\def\calI{{\mathcal{I}}}
\def\calJ{{\mathcal{J}}}

\def\calM{{\mathcal{M}}}
\def\calN{{\mathcal{N}}}

\def\calP{{\mathcal{P}}}

\def\calT{{\mathcal{T}}}

\def\calW{{\mathcal{W}}}
\def\calX{{\mathcal{X}}}

\def\calZ{{\mathcal{Z}}}

\def\bbE{{\mathbb{E}}}
\def\bbF{{\mathbb{F}}}

\def\bbN{{\mathbb{N}}}

\def\bbP{{\mathbb{P}}}
\def\bbQ{{\mathbb{Q}}}
\def\bbR{{\mathbb{R}}}

\def\scrH{{\mathscr{H}}}

\def\frakR{{\mathfrak{R}}}

\DeclareMathOperator*{\argmax}{arg\,max}
\DeclareMathOperator*{\argmin}{arg\,min}

\def\kl{\mathrm{KL}}
\def\iid{\mathrm{i.i.d.}}
\def\cov{\mathrm{Cov}}
\def\var{\mathrm{Var}}
\def\refer{\mathrm{ref}}
\def\supp{\mathrm{supp}}

\def\op{\mathrm{op}}

\DeclareMathOperator{\pr}{Pr}

\mathtoolsset{showonlyrefs}
\allowdisplaybreaks
\setheader{Inference-Time Alignment for Diffusion Models}{Chang, Duan, Jiao, Li, Xu, and Yang}
\title{Inference-Time Alignment for Diffusion Models via Variationally Stable Doob's Matching}

\author[1]{Jinyuan Chang}
\author[2]{Chenguang Duan}
\author[3]{Yuling Jiao}
\author[4]{Ruoxuan Li}
\author[5]{Yi Xu}
\author[5]{Jerry Zhijian Yang}

\affil[1]{Joint Laboratory of Data Science and Business Intelligence, Southwestern University of Finance and Economics, \texttt{changjinyuan@swufe.edu.cn}}
\affil[2]{Weierstrass Institute
for Applied Analysis and Stochastics, \texttt{duan@wias-berlin.de}}
\affil[3]{School of Artificial Intelligence, Wuhan University, \texttt{yulingjiaomath@whu.edu.cn}}
\affil[4]{State Key Laboratory of Mathematical Sciences, Academy of Mathematics and Systems Science, Chinese Academy of Sciences, \texttt{ruoxuanli.math@amss.ac.cn}}
\affil[5]{School of Mathematics and Statistics, Wuhan University, \texttt{\{yixu.math,zjyang.math\}@whu.edu.cn}}
\date{\today}

\begin{document}

\thispagestyle{plain}
\maketitle
\noeqref{eq:Xt:Q:brownian,eq:lemma:error:kl:5,eq:lemma:error:kl:6}

\begin{abstract}
Inference-time alignment for diffusion models aims to adapt a pre-trained reference diffusion model toward a target distribution without retraining the reference score network, thereby preserving the generative capacity of the reference model while enforcing desired properties at inference time. A central mechanism for achieving such alignment is guidance, which modifies the sampling dynamics through an additional drift term. In this work, we introduce variationally stable Doob's matching, a novel framework for provable guidance estimation grounded in Doob's $h$-transform. Our approach formulates guidance as the gradient of the logarithm of an underlying Doob's $h$-function and employs gradient-regularized regression to simultaneously estimate both the $h$-function and its gradient, resulting in a consistent estimator of the guidance. Theoretically, we establish non-asymptotic convergence rates for the estimated guidance. Moreover, we analyze the resulting controllable diffusion processes and prove non-asymptotic convergence guarantees for the generated distributions in the Wasserstein-2 distance. Finally, we show that variationally stable guidance estimators are adaptive to unknown low dimensionality, effectively mitigating the curse of dimensionality under low-dimensional linear subspace assumptions. Numerical experiments demonstrate the effectiveness of our proposed method and show the advantages compared to the baselines.
\end{abstract}

\par

\section{Introduction}\label{section:introduction}

\par Diffusion models~\citep{Sohl2015Deep,Ho2020Denoising,song2019Generative,song2021scorebased} have emerged as powerful generative tools for sampling from data distributions, achieving remarkable success across diverse domains, including text-to-image and text-to-video generation~\citep{Ramesh2021Zero}, Bayesian inverse problems~\citep{chung2023diffusion,song2023Loss,chen2025solving,chang2025provable}, and scientific applications~\citep{Bao2024Score,Li2025State,si2025latentensf,ding2025nonlinear,uehara2025reward}. Recent years have witnessed the development of large-scale diffusion models pre-trained on vast datasets. Despite the robust capabilities of these reference models in capturing the training distribution, the target distributions of downstream generative tasks rarely align perfectly with this reference distribution. For example, in conditional generative learning~\citep{Dhariwal2021Diffusion,ho2021classifierfree}, reference diffusion models generate samples from a mixture of distributions, whereas downstream tasks require sampling from specific constituent distributions. Similarly, in posterior sampling contexts~\citep{chung2023diffusion,song2023Loss,chen2025solving,chang2025provable,purohit2024posterior,martin2025pnpflow}, the reference distribution serves as the prior, while the target is the posterior defined by tilting the prior with a measurement likelihood. Furthermore, downstream tasks frequently impose additional constraints, such as human preferences or safety considerations~\citep{Domingo2024Stochastic,domingo2025adjoint,uehara2025inference,kim2025testtime,sabour2025test,denker2025iterative,ren2025drift}, which must be satisfied without compromising the model's generative quality.

\par To bridge the gap between reference and target distributions, researchers have proposed numerous alignment methods~\citep{Xu2023ImageReward,lee2023aligning,Fan2023DPOK,clark2024directly,Domingo2024Stochastic,domingo2025adjoint,uehara2024finetuning,uehara2025inference}. These strategies generally fall into two categories: fine-tuning and inference-time alignment. Fine-tuning approaches involve retraining the reference score network via supervised learning~\citep{lee2023aligning}, reinforcement learning~\citep{Fan2023DPOK,black2024training,clark2024directly,Uehara2024Feedback}, or classifier-free fine-tuning~\citep{ho2021classifierfree,Zhang2023Adding,yuan2023rewarddirected}. Despite its conceptual simplicity, fine-tuning presents significant limitations. First, it often requires a substantial collection of high-quality samples from the target distribution, which may be unavailable in practical scenarios like posterior sampling. Second, the computational cost of retraining score networks can be prohibitive, particularly for large-scale models with billions of parameters~\citep{uehara2025inference}. Third, fine-tuning is vulnerable to over-optimization~\citep{Gao2023Scaling,Rafailov2024Scaling,kim2025testtime}, where the network overfits to limited target samples or preferences, causing it to ``forget'' the valuable prior information encoded in the reference model. This degradation undermines the fundamental advantage of leveraging pre-trained models.

\par In contrast, inference-time alignment~\citep{uehara2025inference,kim2025testtime,sabour2025test,denker2025iterative,ren2025drift,pachebat2025iterative} eliminates the need to retrain the reference diffusion model. These methods offer substantial computational advantages and preserve the generative capacity of the underlying model. The core technique is guidance~\citep{jiao2025unified}, which incorporates target information by introducing an additional drift term to the reference diffusion model. Within the framework of Doob's $h$-transform~\citep{Rogers2000Diffusions,Sarkka2019applied,Chewi2025log}, the score function for the target tilted distribution decomposes into the sum of the reference score and a guidance term, where the guidance is defined as the gradient of the log-Doob's $h$-function~\citep{Heng2024Diffusion,Tang2024stochastic,Denker2024DEFT,denker2025iterative}. This relationship has also been investigated through the lens of classifier guidance~\citep{Dhariwal2021Diffusion}, stochastic optimal control~\citep{han2024stochastic,tang2025finetuning}, and Bayes' rule~\citep{chung2023diffusion,song2023Loss}. The primary challenge lies in accurately estimating this guidance.

\par Guidance estimation methods can be categorized into two main approaches: approximation and learning. The approximation approach, exemplified by diffusion posterior sampling~\citep{chung2023diffusion} and loss-guided diffusion~\citep{song2023Loss}, relies on heuristic approximations that often lead to inconsistencies with the underlying mathematical formulation. The learning approach~\citep{Dhariwal2021Diffusion,Tang2024stochastic,Denker2024DEFT,denker2025iterative} attempts to learn the necessary components by deep neural networks. Classifier guidance~\citep{Dhariwal2021Diffusion} learns Doob's $h$-function via a classifier, but this is effective primarily for discrete labels. Furthermore, the convergence of the plug-in gradient estimator of the classifier is not guaranteed, potentially undermining guidance reliability~\citep[Section~3.2.2]{mou2025rlfinetuning}. To mitigate this,~\citet{Tang2024stochastic} estimate Doob's $h$-function and its gradient using separate neural networks, increasing training complexity. While~\citet{Denker2024DEFT} directly learn the guidance, their method requires samples from the target tilted distribution.~\citet{denker2025iterative} attempt to address this data requirement via iterative retraining.

\par To address these limitations, we introduce variationally stable Doob's matching, a novel framework for provable guidance estimation in inference-time alignment. We propose a gradient-regularized regression method that simultaneously estimates Doob's $h$-function and its gradient, yielding a consistent estimator of the guidance. When combined with the pre-trained reference score, our method enables efficient sampling from the target distribution without the need for computationally expensive fine-tuning or access to target distribution samples. We also evaluate our proposed method on numerical examples, comparing it with the baselines.

\subsection{Contributions}

Our main contributions are summarized as follows:
\begin{enumerate}[label=(\roman*)]
\item We introduce \emph{variationally stable Doob's matching}, a novel guidance estimation framework for controllable diffusion models grounded in Doob's $h$-transform. The Doob $h$-function is estimated via a gradient-regularized least-squares regression approach. The plug-in gradient of the logarithm of the resulting $h$-function estimator yields an estimator for the Doob's guidance. Additionally, this method is derivative-free, meaning it does not require access to the gradient of the weight function between the target tilted distribution and the reference distribution.
\item We establish non-asymptotic convergence rates for variationally stable Doob's matching, showing that the proposed method guarantees convergence of both the $h$-function estimator and its gradient. These results provide rigorous theoretical guarantees for Doob's guidance estimation (Theorem~\ref{theorem:rate:guidance}). Moreover, we derive non-asymptotic convergence rates for the induced controllable diffusion models, thereby establishing rigorous guarantees for the generated distributions in the Wasserstein-2 distance (Theorem~\ref{theorem:rate:controllable:diffusion}). Additionally, we obtain convergence rates that depend only on the intrinsic dimension, thereby mitigating the curse of dimensionality under low-dimensional linear subspace assumptions (Theorem~\ref{theorem:rate:guidance:low} and Corollary~\ref{corollary:rate:controllable:diffusion:low}).
\item We demonstrate the effectiveness of our method on three tasks: reward-guided generation, transfer learning for diffusion models, and Bayesian inverse problems. In reward-guided generation, it achieves a lower Wasserstein-1 distance to the target distribution than the baselines. In inverse problems, it outperforms a diffusion posterior sampling baseline on Gaussian deblurring, motion deblurring, and inpainting in PSNR, SSIM, and LPIPS.
\end{enumerate}

\subsection{Organization}
The remainder of this paper is organized as follows. In Section~\ref{section:preliminaries}, we provide a brief introduction to diffusion models. In Section~\ref{section:method}, we propose the stochastic dynamics of controllable diffusion models within the framework of Doob's $h$-transform. In Section~\ref{section:doob:matching}, we present a practical algorithm to simultaneously estimate Doob's $h$-function and its gradient. Section~\ref{section:convergence} establishes non-asymptotic error bounds for both the estimation of the $h$-function and the induced controllable diffusion models. Section~\ref{section:experiments} presents numerical experiments comparing our proposed method with the baselines. Finally, concluding remarks are provided in Section~\ref{section:conclusion}. Detailed proofs of theoretical results are deferred to the appendix.

\paragraph{Notation}
For a vector $\vx$, let $|\vx|_2$ and $|\vx|_\infty$ denote its Euclidean and maximum norms, respectively; for a scalar, $|\cdot|$ denotes absolute value. We use double bars for function norms. For a probability density $p$ and $1\le r<\infty$, set
\[
\|f\|_{L^r(p)}=\left\{\int |f(\vx)|^r p(\vx)\,\mathrm{d}\vx\right\}^{1/r}\,,
\]
where $|f(\vx)|$ is the absolute value for scalar-valued $f$ and the Euclidean norm for vector-valued $f$. We write $\|f\|_\infty$ for the supremum norm and use
\[
\|f\|_{H^1(p)}^2=\|f\|_{L^2(p)}^2+\|\nabla f\|_{L^2(p)}^2\,.
\]
We write $B(\vx,R)$ for the open Euclidean ball of radius $R$ centered at $\vx$, $|A|_{\mathrm{op}}$ for the matrix operator norm induced by $|\cdot|_2$, and $|A|_0$ for the number of nonzero entries of a vector or matrix. Let $P$ and $Q$ be two probability measures, the total variation distance between them is defined by $\mathrm{TV}(P,Q)=\sup_E|P(E)-Q(E)|$, where the supremum is over measurable sets. We write $a\lesssim b$ if $a\le Cb$ for a constant $C>0$, with its permitted dependence specified in context, and $a\asymp b$ when both inequalities hold.

\section{Preliminaries on Diffusion Models}
\label{section:preliminaries}

\par This section introduces basic notations of diffusion models~\citep{song2021scorebased}.

\subsection{Forward and time-reversal process}

We consider the diffusion model for a reference distribution $p_0$, which need not have a density. The forward process is the Ornstein--Uhlenbeck process, whose positive-time marginals have smooth, strictly positive densities:
\begin{equation}\label{eq:forward:base}
\d\mX_{t}=-\mX_{t}\dt+\sqrt{2}\d\mB_{t}\,, \quad t\in(0,T) \, , \quad \mX_{0}\sim p_{0}\,,
\end{equation}
where $\mB_{t}$ is a $d$-dimensional standard Brownian motion, and $T>0$ is the terminal time. The transition distribution of the forward process can be expressed as:
\begin{equation}\label{eq:forward:solution}
\operatorname{Law}(\mX_{t}\given\mX_{0}=\vx_{0})=\mathcal{N}(\mu_{t}\vx_{0},\sigma_{t}^{2}\mI_{d})\,,
\end{equation}
where the mean and variance coefficients are given, respectively, as $\mu_{t}=e^{-t}$ and $\sigma_{t}^{2}=1-e^{-2t}$. The forward process~\eqref{eq:forward:base} is commonly referred to as the variance-preserving (VP) SDE~\citep{song2021scorebased} as $\mu_{t}^{2}+\sigma_{t}^{2}=1$ for each $t\in(0,T]$. Denote by $p_{t}$ the marginal density of $\mX_{t}$ for each $t\in(0,T]$, which satisfies 
\begin{equation}\label{eq:marginal:density}
p_{t}(\vx_{t})=\int\varphi_{d}(\vx_{t};\mu_{t}\vx_{0},\sigma_{t}^{2}\mI_{d})p_{0}(\d\vx_{0})\,,
\end{equation}
\begingroup\emergencystretch=2em where $\varphi_{d}(\cdot;\mu_{t}\vx_{0},\sigma_{t}^{2}\mI_{d})$ denotes the density function of the Gaussian distribution $\calN(\mu_{t}\vx_{0},\sigma_{t}^{2}\mI_{d})$. The corresponding time-reversal process~\citep{Anderson1982Reverse} of~\eqref{eq:forward:base} is characterized by: \par\endgroup
\begin{align}\label{eq:base:reversal}
\d\mX_{t}^{\leftarrow}=\Big\{\mX_{t}^{\leftarrow}+2\underbrace{\nabla\log p_{T-t}(\mX_{t}^{\leftarrow})}_{\text{reference score}}\Big\}\dt+\sqrt{2}\d\mB_{t}\,, \quad t\in(0,T)\,, \quad \mX_{0}^{\leftarrow}\sim p_{T}\,.
\end{align}
Here $\mathbf {B}_{t}$ denotes Brownian motion with respect to the reversed filtration, distinct from that in \eqref{eq:forward:base}. The process $(\mathbf {X}_t^{\leftarrow})_{0\le t\le T}$ has the same law as $(\mathbf {X}_{T-t})_{0\le t\le T}$~\citep{Anderson1982Reverse}.

\subsection{Path measure and filtration}
\label{section:preliminaries:path}

We formally define the probability space for the time-reversal process \eqref{eq:base:reversal}. Let $\Omega:=C([0,T],\mathbb {R}^d)$, equipped with the uniform topology, and let $\mathcal {B}(\Omega)$ be its Borel $\sigma$-algebra. Since the forward OU process \eqref{eq:forward:base} has continuous paths, its reversed path $(\mathbf {X}_{T-t})_{0\le t\le T}$ induces a probability measure $\mathbb {P}$ on $(\Omega,\mathcal {B}(\Omega))$~\citep[Theorems 2.1.2--2.1.3]{Chafai2025Stochastic}. For each path $\omega\in\Omega$, define the coordinate process by $\mathbf {X}_t^{\leftarrow}(\omega)=\omega(t)$, $0\le t\le T$. Let $\mathbb {F}:=(\mathcal {F}_t)_{0\le t\le T}$ be the usual augmentation under $\mathbb {P}$ of its natural filtration $\sigma(\mathbf {X}_s^{\leftarrow}:0\le s\le t)$, and set $\mathcal {F}=\mathcal {F}_{T}$. We use $\mathbb {P}$ also for its completion on $\mathcal {F}$. We work on $(\Omega,\mathcal {F},\mathbb {F},\mathbb {P})$, where $\mathbf {X}_0^{\leftarrow}\sim p_T$.

For every $0<\delta<T$, the OU coefficients are globally Lipschitz with linear growth, and Gaussian smoothing makes $p_{s}$ and $\nabla p_{s}$ bounded for $s\in[\delta,T]$.~\citet[Theorem 2.1]{HaussmannPardoux1986Time} and L\'evy's characterization therefore give \eqref{eq:base:reversal} on $[0,T-\delta]$, driven by a common $\mathbb {F}$-Brownian motion $\mathbf {B}$ under $\mathbb {P}$. We use its continuous extension to $T$, with \eqref{eq:base:reversal} restricted to $t<T$.

\subsection{Training phase: score matching}

\par In generative learning, the exact reference score $\nabla\log p_{t}$ in the time-reversal process~\eqref{eq:base:reversal} is intractable. One can estimate the reference score using samples from the reference distribution $p_{0}$ via standard techniques such as implicit score matching~\citep{hyvarinen2005Estimation}, sliced score matching~\citep{song2020Sliced}, and denoising score matching~\citep{Vincent2011connection}. Let $\hat{\vs}:[0,T)\times\bbR^{d}\to\bbR^{d}$ denote a measurable estimator for the prior score, that is,
\begin{equation}\label{eq:base:score:estimator}
\|\hat{\vs}(t,\cdot)-\nabla\log p_{T-t}\|_{L^{2}(p_{T-t})}\leq\varepsilon\,,
\end{equation}
for a small tolerence $\varepsilon\ll 1$. Considerable research has established theoretical guarantees for this score estimation~\citep{Tang2024Adaptivity,Oko2023Diffusion,fu2024unveil,ding2025characteristic}, leveraging standard techniques from non-parametric regression with deep neural networks~\citep{bauer2019deep,schmidt2020nonparametric,kohler2021rate,Jiao2023deep}.

\subsection{Inference phase: sampling}
\label{section:preliminaries:inference}

\par Given a reference score estimator $\hat{\vs}$ in~\eqref{eq:base:score:estimator}, the inference phase of diffusion models aims to generate samples by simulating the time-reversal process with estimated score. Since the explicit solution of the time-reversal process is intractable, we employ an exponential integrator~\citep{Hochbruck2005Explicit,Hochbruck2010Exponential,Lu2022DPMv1,zhang2023fast}. This approach is well-suited for solving the time-reversal process due to the semi-linearity of the drift term of the SDE in~\eqref{eq:base:reversal}.

Let $K\in\bbN$, $K\geq1$, denote the number of discretization steps, and let $0<T_{0}<T$ be an early-stopping time. We define a sequence of uniform time points $t_{k}\coloneq kh$ for $k=0,\dots,K$, where the step size is $h\coloneq(T-T_{0})/K$. In each time sub-interval, the exponential integrator approximates the score function by its value at the left endpoint:
\begin{equation}\label{eq:reverse:score:estimator:expint}
\d\what{\mX}_{t}^{\leftarrow}=\big\{\what{\mX}_{t}^{\leftarrow}+2\hat{\vs}(kh,\what{\mX}_{kh}^{\leftarrow})\big\}\dt+\sqrt{2}\d\mB_{t}\,, \quad t\in[kh, (k+1)h)\,, \,\, \what{\mX}_{0}^{\leftarrow}\sim\calN(\bzero,\mI_{d})\,,
\end{equation}
where $0\leq k\leq K-1$. The resulting linear approximation to the original semi-linear SDE has the following explicit solution:
\begin{equation*}
\what{\mX}_{(k+1)h}^{\leftarrow}=\exp(h)\what{\mX}_{kh}^{\leftarrow}+2\phi^{2}(h)\hat{\vs}(kh,\what{\mX}_{kh}^{\leftarrow})+\phi(2h)\bm{\xi}_{k}\,, \quad 0 \leq k \leq K-1\,,
\end{equation*}
where $\phi(z)\coloneq\sqrt{\exp(z)-1}$, and $\bm{\xi}_{0},\ldots,\bm{\xi}_{K-1}$ are i.i.d. standard Gaussian random variables.

\begin{remark}[Initialization]
Note that the true initial distribution of the time-reversal process~\eqref{eq:base:reversal} is $p_{T}$, rather than the $\calN(\bzero,\mI_{d})$ used in~\eqref{eq:reverse:score:estimator:expint}. We adopt the standard normal distribution because sampling from $\calN(\bzero,\mI_{d})$ is significantly more computationally tractable. This approximation is justified by the fact that, when $p_{0}$ has finite second moment, $p_{T}$ converges to $\calN(\bzero, \mI_{d})$ exponentially in KL-divergence as $T\to\infty$~\citep{Bakry2014Analysis,Vempala2019Rapid}; thus, the Gaussian initialization is valid for a sufficiently large terminal time $T$.
\end{remark}

\section{Controllable Diffusion Models and Doob's Transform}
\label{section:method}

\par In this section, we propose controllable diffusion models for sampling from a target distribution, defined as the reference distribution tilted by a weight function. We utilize the theory of measure change for diffusion processes on the filtered space $(\Omega,\calF,\bbF,\bbP)$.

\subsection{Problem setup}

\par We assume access to a pre-trained reference diffusion model~\eqref{eq:reverse:score:estimator:expint} that generates samples approximately from the reference distribution $p_{0}$. We aim to sample from a tilted distribution $q_{0}$, defined by reweighting the reference distribution with a known Borel measurable weight function $w:\bbR^{d}\to\bbR_{\geq 0}$:
\begin{equation}\label{eq:weight}
q_{0}(\d\vx)\coloneq\frac{w(\vx)}{Z}p_{0}(\d\vx)\,, \quad\text{where}~Z\coloneq\int w(\vx)p_{0}(\d\vx)\in(0,\infty)\,.
\end{equation}
Let $q_{s}$ denote the time-$s$ marginal of the forward Ornstein--Uhlenbeck process started from $q_{0}$; it has a density for $s>0$. Our goal is to derive a new diffusion process that generates samples from the tilted distribution $q_{0}$ directly by introducing a drift correction to the pre-trained reference diffusion model~\eqref{eq:reverse:score:estimator:expint}.

\par This problem encompasses a wide range of application scenarios.

\begin{example}[Bayesian inverse problems]\label{example:bip}
Bayesian inverse problems play a critical role in scientific computing~\citep{Stuart2010Inverse,Kantas2014Sequential,ding2025nonlinear}, image science~\citep{chung2023diffusion,mardani2024a,purohit2024posterior,chang2025provable}, and medical imaging~\citep{song2022solving}. In Bayesian inverse problems, we aim to recover an unknown signal $\mX_{0}\in\bbR^d$ from noisy measurements $\mY\in\bbR^{m}$, which are linked by the following measurement model:
\begin{equation}\label{eq:inverse:problem}
\mY=\calA(\mX_{0})+\vn\,,
\end{equation}
where $\mathcal{A}:\bbR^{d}\to\bbR^{m}$ is a known measurement operator, and $\vn\in\bbR^{m}$ represents measurement noise, independent of $\mX_{0}$, with a known distribution. The Bayesian approach incorporates prior knowledge about $\mX_{0}$ in the form of a prior distribution $p_{0}$. Given observed measurements $\mY=\vy$, the goal of Bayesian inverse problems is to sample from the posterior distribution:
\begin{equation}\label{eq:applications:inverse:posterior}
q_{0}(\d\vx):=p_{\mX_{0}\given\mY}(\d\vx\given\vy)=\frac{w(\vx)}{Z}p_{0}(\d\vx)\,,
\end{equation}
where $w(\vx)\coloneq p_{\mY\given\mX_{0}}(\vy\given\vx)$ is a likelihood determined by the measurement model~\eqref{eq:inverse:problem}, and $Z$ is a partition function to ensure $q_{0}$ is a valid probability density. For example, for a Gaussian noise $\vn\sim\mathcal{N}(\bzero,\sigma^{2}\mI_{m})$, it holds that
\begin{equation}\label{eq:applications:inverse:guidance}
w(\vx)=(2\pi\sigma^{2})^{-m/2}\exp\Big\{-\frac{1}{2\sigma^2}|\vy - \mathcal{A}(\vx)|_2^2\Big\}\,.
\end{equation}
In Bayesian inverse problems, one typically has a reference diffusion model pre-trained on the prior distribution, and aims to sample from the posterior distribution~\eqref{eq:applications:inverse:posterior} without retraining the reference model.
\end{example}

\begin{example}[Reward-guided generation]
In the reward-guided generation~\citep{domingo2025adjoint,uehara2025inference,kim2025testtime,sabour2025test,denker2025iterative,ren2025drift}, human preferences and constraints can be encoded into a reward function $r:\bbR^{d}\to\bbR$. For instance, in text-to-image generation, the reward function $r$ quantifies how well the generated data aligns with the input prompt. In practice, such reward function can be learned from the human feedback or preference data~\citep{Stiennon2020Learning,ouyang2022training,lee2023aligning}. For the sake of simplicity, we assume throughout this work that the reward function has already been given. A naive approach to reward-guided generation is to maximize the expected reward $\max_{\pi\in\calP}\bbE_{\pi}(r)$, where $\calP(\mathbb{R}^{d})$ is the set of probability measures on $\bbR^{d}$. However, solely maximizing the expected reward may lead to over-optimization and degenerate solutions~\citep{kim2025testtime}. To mitigate this, entropy regularization~\citep{uehara2024finetuning,tang2025finetuning} is incorporated into the objective functional, yielding the following optimization problem:
\begin{equation}\label{eq:max:reward:entropy}
q_{0}\coloneq\argmax_{\pi\in\calP(\mathbb{R}^{d})}\big\{\bbE_{\pi}(r)-\alpha\kl(\pi\| p_{0})\big\}\,,
\end{equation}
where $\alpha>0$ is a regularization parameter, and $p_{0}$ is the reference distribution, i.e., the distribution corresponding to the pre-trained reference diffusion model. This objective comprises two components: the expected reward, which captures human preferences, and the KL-divergence term, which prevents the distribution from deviating excessively from the reference diffusion model. The closed-form solution to this optimization problem~\eqref{eq:max:reward:entropy} is given by~\citep{Rafailov2023Direct}:
\begin{equation}\label{eq:applications:align:target}
q_{0}(\d\vx)=\frac{w(\vx)}{Z}p_{0}(\d\vx)\,, \quad w(\vx):=\exp\bigg\{\frac{r(\vx)}{\alpha}\bigg\}\,,
\end{equation}
where $Z$ is the partition function to ensure $q_{0}$ is a valid probability density. In reward-guided generation, the central objective is to incorporate preferences exclusively during the inference phase, thereby avoiding the substantial computational cost of retraining the large-scale reference diffusion model.
\end{example}

\begin{example}[Transfer learning for diffusion models]
Diffusion models have achieved remarkable success in image generation. However, their performance critically depends on the availability of large-scale training data. In the scenarios of few-shot generation, training diffusion models solely on limited samples typically results in poor generative performance. To retain strong generative capability under data scarcity, a common strategy is to transfer expressive diffusion models pre-trained on large datasets to the target domain~\citep{Ouyang2024Transfer,wang2024Bridging,zhong2025domain,bahram2026dogfit}. Formally, transfer learning aims to adapt a model pre-trained on a large-scale source distribution to a target distribution of much smaller size and diversity of samples. However, directly fine-tuning a large pre-trained diffusion model using only limited target samples often leads to severe overfitting~\citep{wang2024Bridging}. To address this issue, transfer learning approaches for diffusion models typically train a lightweight guidance network on the limited target data and combine it with the pre-trained reference score network, yielding a modified diffusion model capable of sampling from the target distribution~\citep{Ouyang2024Transfer,zhong2025domain,bahram2026dogfit}. Concretely,~\citet{Ouyang2024Transfer} estimates the Radon--Nikodym derivative between the target and source distributions,
\begin{equation}\label{eq:applications:transfer:density:ratio}
w(\vx):=\frac{\d q_{0}(\vx)}{\d p_{0}(\vx)}\,,
\end{equation}
using limited samples from the target distribution. For simplicity, we assume throughout this work that the density ratio is known. Under this assumption, the objective of transfer learning for diffusion models is to sample from the target distribution $q_{0}$ by leveraging the pre-trained reference diffusion model together with the estimated density ratio.
\end{example}

\subsection{Controllable diffusion models with Doob's \texorpdfstring{$h$}{h}-transform}

\par In this subsection, we achieve inference-time alignment by incorporating guidance into the reference diffusion models via Doob's $h$-transform~\citep{Rogers2000Diffusions,Sarkka2019applied,Heng2024Diffusion,Chewi2025log}. We begin by constructing a target path measure $\bbQ$ on the filtered probability space $(\Omega,\calF,\bbF)$. We assume that $\bbQ$ is absolutely continuous with respect to the reference path measure $\bbP$ defined in Section~\ref{section:preliminaries:path}. By the Radon--Nikodym theorem, this relationship is characterized by the existence of an $\bbF$-adapted process $(L_{t})_{0\leq t\le T}$ with $L_{t}\geq 0$, such that for every $t\in[0,T]$:
\begin{equation}\label{eq:time:dependent:likelihood}
L_{t}\coloneq\frac{\d\bbQ}{\d\bbP}\Big|_{\calF_{t}}\,.
\end{equation}
The process $L_{t}$ is known as the Radon--Nikodym derivative process, representing the likelihood ratio between the target and reference measures conditioned on the filtration $\calF_{t}$.

\par We require the marginal distribution of the terminal state $\mX_T^{\leftarrow}$ under $\bbQ$ to coincide with the tilted distribution $q_0$ defined in~\eqref{eq:weight}. To this end, define the target path measure $\bbQ$ by
\begin{equation}\label{eq:path:Q}
L_T=\frac{\d\bbQ}{\d\bbP}\Big|_{\calF_T}=\frac{\d\bbQ}{\d\bbP}=\frac{w(\mX_T^{\leftarrow})}{\bbE^{\bbP}\big\{w(\mX_T^{\leftarrow})\big\}}\,.
\end{equation}
Here $\calF_{T}=\calF$ by Section~\ref{section:preliminaries:path}, and the denominator is $Z$, ensuring that $\bbQ$ is a valid probability measure, i.e., $\bbE^{\bbP}(L_T)=1$. This construction gives the desired terminal distribution: for any Borel set $A\subseteq\bbR^{d}$,
\[
\bbQ(\mX_{T}^{\leftarrow}\in A)=\frac{1}{Z}\int_{A} w(\vx)p_{0}(\d\vx)=q_{0}(A)\,.
\]
The same tilt also changes the initial distribution of the reversed process. Using the forward transition density in~\eqref{eq:forward:solution} and $w(\vx_{0})p_{0}(\d\vx_{0})/Z=q_{0}(\d\vx_{0})$, we obtain
\[
\bbQ(\mX_{0}^{\leftarrow}\in A)=\int_{A}\int\varphi_{d}(\vx;\mu_{T}\vx_{0},\sigma_{T}^{2}\mI_{d})\,\frac{w(\vx_{0})}{Z}\,p_{0}(\d\vx_{0})\,\d\vx=q_{T}(A)\,.
\]
Thus, under $\bbQ$, the reversed process starts from $q_{T}$ and ends at $q_{0}$.

\par The rest of our derivation proceeds in two steps. First, we characterize the stochastic dynamics of $\mX_t^{\leftarrow}$ under the target measure $\bbQ$. Second, by weak uniqueness, we identify a controllable diffusion process whose terminal distribution is the target tilted distribution $q_0$.

\paragraph{Dynamics under the target path measure}

\par We first characterize the time-dependent likelihood ratio $L_t$ and its stochastic dynamics.

\begin{proposition}[Doob's $h$-function]
\label{proposition:doob:transform}
The Radon--Nikodym derivative process $L_t$, as defined in~\eqref{eq:time:dependent:likelihood}, admits the representation
\[
L_t=\bbE^{\bbP}\bigg(\frac{\d\bbQ}{\d\bbP}\given[\Big]\calF_t\bigg)=\frac{h^*(t,\mX_t^{\leftarrow})}{Z}\,,\qquad 0\le t\le T\,,
\]
where $Z=\bbE^{\bbP}\big\{w(\mX_T^{\leftarrow})\big\}$ and Doob's $h$-function is defined by
\begin{equation}\label{eq:h:function}
h^*(t,\vx)\coloneq\bbE^{\bbP}\big\{w(\mX_T^{\leftarrow})\given\mX_t^{\leftarrow}=\vx\big\}\,.
\end{equation}
The processes $h^{*}(t,\mX_{t}^{\leftarrow})$ and $L_t$ are uniformly integrable $\bbP$-martingales on $[0,T]$. For $t<T$, $h^{*}$ admits a positive $C^{1}(0,T;C^{2}(\mathbb{R}^{d}))$ version, and
\[
\d h^{*}(t,\mX_{t}^{\leftarrow})=\sqrt{2}\nabla h^{*}(t,\mX_{t}^{\leftarrow})^{\top}\d\mB_{t}\,,\qquad
\d L_{t}=\sqrt{2}L_{t}\nabla\log h^*(t,\mX_t^{\leftarrow})^{\top}\d\mB_t\,.
\]
\end{proposition}

\par The proof of Proposition~\ref{proposition:doob:transform} is provided in Appendix~\ref{appendix:section:method}. We now use the dynamics of $L_t$ to derive those of $\mX_t^{\leftarrow}$ under the target measure $\bbQ$.

\begin{proposition}
\label{proposition:controllable:sde:Q}
Let the reference process $\mX_t^{\leftarrow}$ satisfy the SDE~\eqref{eq:base:reversal} under the path measure $\bbP$, and let $\bbQ$ be the target path measure defined by~\eqref{eq:path:Q}. Assume that the Novikov condition holds:
\begin{equation}\label{eq:Novikov}  
\bbE^{\bbP}\bigg[\exp\bigg\{\int_{0}^{T}|\nabla\log h^{*}(s,\mX_{s}^{\leftarrow})|_{2}^{2}\d s\bigg\}\bigg]<\infty\,.
\end{equation}
Define a process $(\widetilde{\mB}_{t})_{0\leq t<T}$ by
\begin{equation}\label{eq:brownian:P:Q}
\d\widetilde{\mB}_{t}=\d\mB_{t}-\sqrt{2}\nabla\log h^{*}(t,\mX_{t}^{\leftarrow})\dt\,,\qquad \widetilde{\mB}_{0}=\bzero\,,
\end{equation}
where $\mB_t$ is a Brownian motion under $\bbP$ and $\widetilde{\mB}_{T}$ is defined by continuous extension. Then $\widetilde{\mB}$ is a standard Brownian motion under $\bbQ$ on $[0,T]$. Further, under the path measure $\bbQ$, the reference process $\mX_t^{\leftarrow}$ in~\eqref{eq:base:reversal} evolves, for $0\le t<T$, according to:
\begin{equation}\label{eq:Xt:Q:brownian}
\d\mX_{t}^{\leftarrow}=\big\{\mX_{t}^{\leftarrow}+2\nabla\log p_{T-t}(\mX_{t}^{\leftarrow})+2\nabla\log h^{*}(t,\mX_{t}^{\leftarrow})\big\}\dt+\sqrt{2}\d\widetilde{\mB}_{t}\,.
\end{equation}
\end{proposition}

\par The proof is deferred to Appendix~\ref{appendix:section:method}. By the construction in~\eqref{eq:path:Q}, the law of $\mX_{T}^{\leftarrow}$ under $\bbQ$ coincides with the target tilted distribution $q_{0}$.

\begin{remark}[Novikov condition]
Condition~\eqref{eq:Novikov} is sufficient for extending $\widetilde{\mB}$ to a $\bbQ$-Brownian motion on $[0,T]$. On every $[0,T-\delta]$, $0<\delta<T$, the Brownian and SDE conclusions of Proposition~\ref{proposition:controllable:sde:Q} hold without~\eqref{eq:Novikov}, since $L_{t}$ is a positive density martingale there. These conclusions suffice for the early-stopped sampler.
\end{remark}

\paragraph{Controllable diffusion process}

Following Proposition~\ref{proposition:controllable:sde:Q}, we consider the controllable diffusion process, for $0\leq t<T$,
\begin{equation}\label{eq:Zt:P:brownian}
\d\mZ_t^{\leftarrow}=\big\{\mZ_t^{\leftarrow}+2\underbrace{\nabla\log p_{T-t}(\mZ_t^{\leftarrow})}_{\text{reference score}}+2\underbrace{\nabla\log h^*(t,\mZ_t^{\leftarrow})}_{\text{Doob's guidance}}\big\}\dt+\sqrt{2}\d\mB_t\,,
\end{equation}
where $\mB$ is a standard Brownian motion and $\mZ_{0}^{\leftarrow}\sim q_{T}$ is independent of $\mB$.

Under Assumptions~\ref{assumption:bounded:support} and~\ref{assumption:bounded:weight}, the drift in~\eqref{eq:Zt:P:brownian} is globally Lipschitz in space on every \mbox{$[0,T-\delta]$}, $0<\delta<T$; see Appendix~\ref{section:proof:error:decomposition} for details. Thus~\eqref{eq:Zt:P:brownian} has a unique strong solution on each such interval, whose path law agrees with that of $\mX^{\leftarrow}$ under $\bbQ$ by weak uniqueness~\citep[Theorems 8.1.2 and 8.3.3]{Chafai2025Stochastic}. Consequently, the solution of~\eqref{eq:Zt:P:brownian} extends continuously to time $T$, and its terminal state $\mZ_T^{\leftarrow}$ has exactly the target tilted distribution $q_0$.

\section{Variationally Stable Doob's Matching}
\label{section:doob:matching}

\par We have thus far established a controllable diffusion process~\eqref{eq:Zt:P:brownian} capable of generating samples from the target tilted distribution $q_{0}$. However, the Doob's $h$-guidance, required by~\eqref{eq:Zt:P:brownian}, remains intractable. This subsection proposes a variationally stable Doob's matching method to address the estimation of the Doob's guidance.

\subsection{Vanilla least-squares regression for Doob's matching}
\label{section:method:h:matching}

\par In this section assume $w(\mX_{0})\in L^{2}(\mathbb{P})$, as ensured by Assumption~\ref{assumption:bounded:weight}. For any $t\in[0,T)$, the Doob's $h$-function $h_t^*\coloneq h^*(t,\cdot)$ defined as~\eqref{eq:h:function} is the unique minimizer in $L^{2}(p_{T-t})$, of the following implicit Doob's matching objective:
\begin{align}
\mathcal{J}_{t}(h_{t})
&=\mathbb{E}^{\mathbb{P}}\big\{|h_{t}(\mathbf{X}_{t}^{\leftarrow})-w(\mathbf{X}_{T}^{\leftarrow})|^{2}\big\} \label{eq:h:matching:implicit} \\
&=\mathbb{E}_{\mathbf{X}_{0}\sim p_{0}}\big[\mathbb{E}_{\bm{\varepsilon}\sim\mathcal{N}(\mathbf{0}, \mathbf{I}_{d})}\big\{|h_{t}(\mu_{T-t}\mathbf{X}_{0}+\sigma_{T-t}\bm{\varepsilon})-w(\mathbf{X}_{0})|^{2}\big\}\big]\,,\notag
\end{align}
where $w: \mathbb{R}^{d} \to \mathbb{R}_{\geq 0}$ is the known weight function in~\eqref{eq:weight}. The following proposition justifies the use of $\mathcal{J}_{t}$ as a surrogate for the explicit $L^2$-distance.

\begin{proposition}\label{proposition:regression:constant}
If $w(\mX_{0})\in L^{2}(\mathbb{P})$, then for every $t\in[0,T)$ and $h_{t}\in L^{2}(p_{T-t})$, the Doob's $h$-function $h_{t}^{*}$ in~\eqref{eq:h:function} minimizes the implicit Doob's matching objective~\eqref{eq:h:matching:implicit}. Further, 
\begin{equation*}
\calJ_{t}(h_{t})=\bbE^{\bbP}\big\{|h_{t}(\mX_{t}^{\leftarrow})-h_{t}^{*}(\mX_{t}^{\leftarrow})|^{2}\big\}+V_{t}^{2}\,,
\end{equation*}
where $V_{t}^{2}\coloneq\bbE^{\bbP}[\var\{w(\mX_{T}^{\leftarrow})\given\mX_{t}^{\leftarrow}\}]$ is a constant independent of $h_{t}$. 
\end{proposition}

\par The proof of Proposition~\ref{proposition:regression:constant} is provided in Appendix~\ref{appendix:doob:matching}. 

\subsection{Limitations of vanilla regression}

\begingroup\emergencystretch=2em \par Crucially, computing the Doob's guidance $\nabla\log h_{t}^{*}$ in~\eqref{eq:Zt:P:brownian} requires estimating not only the function $h_{t}^{*}$ itself but also its gradient $\nabla h_{t}^{*}$, as the guidance is given by:\par\endgroup
\begin{equation*}
\nabla\log h_{t}^{*}(\vx)=\frac{\nabla h_{t}^{*}(\vx)}{h_{t}^{*}(\vx)}\,, \quad \vx\in\bbR^{d}\,.
\end{equation*}

\par According to Proposition~\ref{proposition:regression:constant}, the objective $\mathcal{J}_{t}$ is only coercive with respect to the $L^{2}(p_{T-t})$-norm. Specifically, for any $h_{t}\in L^{2}(p_{T-t})$,
\begin{equation}\label{eq:stability:L2}
\calJ_{t}(h_{t})-\calJ_{t}(h_{t}^{*})
= \|h_{t}-h_{t}^{*}\|_{L^{2}(p_{T-t})}^{2}\,.
\end{equation} 
However, even if $h_{t}$ is close to $h_{t}^{*}$ in the $L^{2}(p_{T-t})$ sense, the gradient $\nabla h_{t}$ may remain highly oscillatory. This leads to an unstable plug-in estimator for the Doob's guidance, a difficulty noted in prior works such as~\citet[Section 3.2.1]{Tang2024stochastic} and~\citet[Section 3.2.2]{mou2025rlfinetuning}. Similar issues arise in related contexts, including classifier guidance~\citep{Dhariwal2021Diffusion}, Monte Carlo regression~\citep[Section 2.2]{uehara2025inference}, and~\citet{Ouyang2024Transfer}.

\par To illustrate this fundamental limitation, consider the sequence of functions $f_{n}:\calI\to\bbR$ defined by $x\mapsto n^{-1}\sin(nx)$, alongside the zero function $f_{0}\equiv 0$, where $\calI:=[0,2\pi]$. While $\|f_{n}-f_{0}\|_{L^{2}(\calI)}\to 0$ as $n\to\infty$, the distance between their derivatives does not vanish, i.e., $\lim_{n\to\infty}\|f_{n}^{\prime}-f_{0}^{\prime}\|_{L^{2}(\calI)}\neq 0$. As a result, the convergence of the plug-in Doob's guidance estimator derived from vanilla regression~\eqref{eq:h:matching:implicit}, as utilized in~\citet{uehara2025inference,Ouyang2024Transfer}, cannot be guaranteed in general.  

\par To mitigate this,~\citet{Tang2024stochastic} estimate Doob's $h$-function and its gradient separately via a martingale approach. In contrast, in the remainder of this work, we propose an approach to simultaneously estimate both the function and its gradient.

\subsection{Variationally stable Doob's matching} 

\par To simultaneously estimate the Doob's $h$-function and its gradient, we adopt a gradient-regularized regression~\citep{Drucker1991Double,Drucker1992Improving,ding2025Semi}. For $0\leq t<T$, we use the weighted Sobolev space
\[
H^{1}(p_{T-t})=\left\{h\in H^{1}_{\mathrm{loc}}(\bbR^{d}):
 \int(h^{2}+|\nabla h|_{2}^{2})p_{T-t}<\infty\right\}\,;
\]
$H^{2}(p_{T-t})$ is defined similarly using weak derivatives through order two. The following proposition defines the gradient-regularized population objective and establishes the existence and uniqueness of its minimizer.

\begin{proposition}[Existence and uniqueness of the regularized population minimizer]\label{proposition:population:existence}
Let $0\leq t<T$, $\lambda>0$, and $w(\mX_{0})\in L^{2}(\mathbb{P})$. For $h\in H^{1}(p_{T-t})$, define
\begin{equation}\label{eq:h:matching:GP}
\calJ_{t}^{\lambda}(h)\coloneq\calJ_{t}(h)+\lambda\|\nabla h\|_{L^{2}(p_{T-t})}^{2}=\calJ_t(h)+\lambda\bbE\bigl\{|\nabla h(\mu_{T-t}\mX_0+\sigma_{T-t}\vepsilon)|_2^2\bigr\}\,,
\end{equation}
where $\mX_{0}\sim p_0$ and $\vepsilon\sim\calN(\bzero,\mI_{d})$ are independent. Then $\calJ_{t}^{\lambda}$ admits a unique global minimizer in $H^{1}(p_{T-t})$, denoted by
\[
h_{t}^{\lambda}=\argmin_{h\in H^{1}(p_{T-t})}\calJ_{t}^{\lambda}(h)\,,
\]
where uniqueness is understood up to $p_{T-t}$-almost-everywhere equality.
\end{proposition}

The proof is provided in Appendix~\ref{appendix:doob:matching}. The next two propositions quantify the regularization gap and establish variational stability around $h_{t}^{\lambda}$.

\begin{proposition}[Regularization gap]\label{proposition:regularization:gap}
Suppose Assumptions~\ref{assumption:bounded:support} and~\ref{assumption:bounded:weight} hold, and let $0\leq t<T$. Let $\lambda>0$, $h_{t}^{*}$ be the Doob's $h$-function defined as~\eqref{eq:h:function}, and $h_{t}^{\lambda}$ be the minimizer of $\calJ_{t}^{\lambda}$ defined as~\eqref{eq:h:matching:GP}. Then $h_{t}^{*}\in H^{2}(p_{T-t})$, and 
\begin{align*}
\|h_{t}^{\lambda}-h_{t}^{*}\|_{L^{2}(p_{T-t})}^{2} 
&\leq \lambda^{2}\| \Delta h_{t}^{*}+\nabla h_{t}^{*}\cdot\nabla\log p_{T-t}\|_{L^{2}(p_{T-t})}^{2}\,, \\
\|\nabla h_{t}^{\lambda}-\nabla h_{t}^{*}\|_{L^{2}(p_{T-t})}^{2}&\leq \lambda\|\Delta h_{t}^{*}+\nabla h_{t}^{*}\cdot\nabla\log p_{T-t}\|_{L^{2}(p_{T-t})}^{2}\,.
\end{align*}
\end{proposition}

\par The proof of Proposition~\ref{proposition:regularization:gap} is provided in Appendix~\ref{appendix:doob:matching}. This proposition demonstrates that as $\lambda\to 0$, the minimizer $h_{t}^{\lambda}$ of the regularized objective~\eqref{eq:h:matching:GP} converges to $h_{t}^{*}$ in $H^{1}$-norm. More importantly, the objective is variationally stable in the $H^{1}$ sense:

\begin{proposition}[Variational stability]\label{proposition:stability}
Let $0\leq t<T$, $w(\mX_{0})\in L^{2}(\mathbb{P})$ and $\lambda>0$, and let $h_{t}^{\lambda}$ be the minimizer of $\calJ_{t}^{\lambda}$ defined as~\eqref{eq:h:matching:GP}. Then for any $h_{t}\in H^{1}(p_{T-t})$, we have
\begin{equation*}
\frac{1}{\max(\lambda,1)}\{\calJ_{t}^{\lambda}(h_{t})-\calJ_{t}^{\lambda}(h_{t}^{\lambda})\}
\leq \|h_{t}-h_{t}^{\lambda}\|_{H^{1}(p_{T-t})}^{2}
\leq \frac{1}{\min(\lambda,1)}\{\calJ_{t}^{\lambda}(h_{t})-\calJ_{t}^{\lambda}(h_{t}^{\lambda})\}\,.
\end{equation*}
\end{proposition}

\par The proof of Proposition~\ref{proposition:stability} is provided in Appendix~\ref{appendix:doob:matching}. We refer to the regularized objective $\calJ_{t}^{\lambda}$ as variationally stable because the convergence in this objective functional necessitates simultaneous convergence in both the function values and their gradients, i.e., stability in the $H^{1}$ sense. For fixed $\lambda>0$, small excess objective guarantees closeness to $h_t^\lambda$ in both value and gradient. Closeness to the ground truth $h_{t}^{*}$ additionally requires the regularization-gap bounds above. Such variational stability has been investigated in the context of the deep Ritz method~\citep{E2025deep}; see, for example,~\citet{duan2022convergence,duan2022deep,lu2022machine,jiao2024error,jiao2025DRM}.

\paragraph{Comparison between vanilla and gradient-regularized Doob's matching}

The distinction between vanilla Doob's matching~\eqref{eq:h:matching:implicit} and the proposed gradient-regularized Doob's matching~\eqref{eq:h:matching:GP} is fundamental for stable diffusion guidance estimation. Vanilla population excess risk controls the $L^2$ error by~\eqref{eq:stability:L2}, but gives no corresponding derivative control. In contrast, the proposed objective $\calJ_{t}^{\lambda}$ is coercive with respect to the $H^{1}$ norm; that is, for a fixed $\lambda>0$,
\begin{equation*}
\calJ_{t}^{\lambda}(h_{t})-\calJ_{t}^{\lambda}(h_{t}^{\lambda})
\eqsim \|h_{t}-h_{t}^{\lambda}\|_{H^{1}(p_{T-t})}^{2}\,.
\end{equation*}
This property supports the simultaneous estimation of Doob's $h$-function and its gradient, providing a basis for the plug-in estimator of Doob's guidance analyzed below.

\subsection{Doob's guidance estimation}

Since the expectation in the population risk~\eqref{eq:h:matching:GP} is computationally intractable, we approximate it by empirical risk using independent and identically distributed samples:
\begin{equation}\label{eq:h:matching:GP:empirical}
\what{\calJ}_{t}^{\lambda}(h_{t})\coloneq\what{\calJ}_{t}(h_{t})+\frac{\lambda}{n}\sum_{i=1}^{n}|\nabla h_{t}(\mu_{T-t}\mX_{0}^{i}+\sigma_{T-t}\vepsilon^{i})|_{2}^{2}\,,
\end{equation}
where the empirical least-squares risk is defined as
\begin{equation}\label{eq:erm}
\what{\calJ}_{t}(h_{t})\coloneq\frac{1}{n}\sum_{i=1}^{n}|h_{t}(\mu_{T-t}\mX_{0}^{i}+\sigma_{T-t}\vepsilon^{i})-w(\mX_{0}^{i})|^{2}\,.
\end{equation}
Here $\mX_{0}^{1},\ldots,\mX_{0}^{n}$ are independent and identically distributed random variables drawn from the reference distribution $p_{0}$, and $\vepsilon^{1},\ldots,\vepsilon^{n}$ are independent standard Gaussian random variables, independent of the reference samples. Then one has a gradient-regularized empirical risk minimizer:
\begin{equation}\label{eq:erm:GP}
\hat{h}_{t}^{\lambda}\in\argmin_{h_{t}\in\scrH_{t}}\what{\calJ}_{t}^{\lambda}(h_{t})\,,
\end{equation}
where $\scrH_{t}$ is a hypothesis class, which is chosen as a neural network class in this work.

We define $\hat {h}_{kh}^{\lambda}$ by~\eqref{eq:erm:GP} at each grid time $t=kh$, $k=0,\ldots,K-1$, and write $\hat {h}^{\lambda}(kh,\vx)=\hat {h}_{kh}^{\lambda}(\vx)$. The same $n$ sample pairs $(\mX_{0}^{i},\vepsilon^{i})$ may be reused across all grid times. A shared time-dependent network can be used provided that each time slice satisfies~\eqref{eq:erm:GP}.

\par Doob's matching with gradient regularization~\eqref{eq:erm:GP} yields the following plug-in estimator of Doob's guidance:
\begin{equation}\label{eq:regressor:guidance}
\hat{\vg}_{t}^{\lambda}(\vz)\coloneq\nabla\log\hat{h}_{t}^{\lambda}(\vz)=\frac{\nabla\hat{h}_{t}^{\lambda}(\vz)}{\hat{h}_{t}^{\lambda}(\vz)}\approx\frac{\nabla h_{t}^{*}(\vz)}{h_{t}^{*}(\vz)}=\nabla\log h_{t}^{*}(\vz)\,, \quad \vz\in\bbR^{d}\,.
\end{equation}

\subsection{A summary of computing procedure}

\par By a similar argument as Section~\ref{section:preliminaries:inference}, we have the exponential integrator for the controllable diffusion model:
\begin{align}
&\d\what{\mZ}_{t}^{\leftarrow}=\{\what{\mZ}_{t}^{\leftarrow}+2\hat{\vs}(kh,\what{\mZ}_{kh}^{\leftarrow})+2\hat{\vg}^{\lambda}(kh,\what{\mZ}_{kh}^{\leftarrow})\}\dt+\sqrt{2}\d\mB_{t}\,, \notag \\
&~~~~~~~~~~~~~~~~~~~~~~~~~~~~~~~~~~~~~~~~~~~~~~~~~~~~~~~~~~~~~~ t\in(kh,(k+1)h)\,, \quad
\what{\mZ}_{0}^{\leftarrow}\sim\calN(\bzero,\mI_{d})\,, \label{eq:Zt:P:brownian:estimation}
\end{align}
where $0\leq k\leq K-1$, the pre-trained reference score estimator $\hat{\vs}$ is defined as~\eqref{eq:base:score:estimator}, and the Doob's guidance estimator $\hat{\vg}^{\lambda}$ is defined as~\eqref{eq:regressor:guidance}.

\par We apply post-processing to the generated particle $\what{\mZ}_{T-T_{0}}^{\leftarrow}$ to ensure numerical stability and facilitate the theoretical analysis presented in Theorem~\ref{theorem:rate:controllable:diffusion}~\citep{Lee2023Convergence,Chen2023Improved}. First, we assume the target distribution $q_{0}$ is concentrated on a domain centered at the origin, such as a distribution with compact support (Assumption~\ref{assumption:bounded:support}) or with light tails. Consequently, we introduce a truncation operator to the particles obtained from the controllable diffusion model~\eqref{eq:Zt:P:brownian:estimation}. Second, because the controllable diffusion process is terminated at an early-stopping time $T_{0}$, there exists a mean shift between the target distribution $q_{0}$ and the early-stopping distribution $q_{T_{0}}\approx\what{q}_{T-T_{0}}$, as indicated by~\eqref{eq:forward:solution}. To mitigate this drift, we employ a scaling operator. Specifically, for $R>0$, we define a truncation operator $\calT_{R}:\vz\mapsto\vz\bbone_{B(\bzero,R)}(\vz)$ and a scaling operator $\calM:\vz\mapsto\mu_{T_{0}}^{-1}\vz$. The final processed particle is defined as
\begin{equation}\label{eq:trancation:scaling}
\calM\circ\calT_{R}(\what{\mZ}_{T-T_{0}}^{\leftarrow})=\mu_{T_{0}}^{-1}\what{\mZ}_{T-T_{0}}^{\leftarrow}\bbone_{B(\bzero,R)}(\what{\mZ}_{T-T_{0}}^{\leftarrow})\,,
\end{equation}
and we denote its distribution by $(\calM\circ\calT_{R})_{\sharp}\what{q}_{T-T_{0}}$, which may have an atom at the origin.

\par A complete procedure is summarized in Algorithm~\ref{alg:controllable:diffusion}, with $T,T_{0},R$ and the training data as specified above, and $h=(T-T_{0})/K$.

\begin{algorithm}[ht]
\caption{Inference-time alignment via variationally stable Doob's matching}
\label{alg:controllable:diffusion}
\KwIn{Reference score estimator $\hat{\vs}$, the weight function $w$, the regularization parameter $\lambda$, the step size $h$, and the number of steps $K$.}
\KwOut{Particle $\calM\circ\calT_{R}(\what{\mZ}_{T-T_{0}}^{\leftarrow})$ follows the tilted distribution $q_{0}$ approximately.}
\texttt{\# Doob's matching} \\
Estimate Doob's $h$-function by $\hat{h}_{t}^{\lambda}$ via gradient regularized Doob's matching~\eqref{eq:erm:GP}. \\
\texttt{\# Controllable generation} \\
Generate the initial particle $\what{\mZ}_{0}^{\leftarrow}\sim\calN(\bzero,\mI_{d})$. \\
\For{$k=0,\ldots,K-1$}{
Evaluate the reference score: $\hat{\vs}_{k}\leftarrow \hat{\vs}(kh,\what{\mZ}_{kh}^{\leftarrow})$. \\ 
Evaluate the Doob's guidance: $\hat{\vg}_{k}\leftarrow\nabla\log\hat{h}^{\lambda}(kh,\what{\mZ}_{kh}^{\leftarrow})$. \\    
Exponential integrator: $\what{\mZ}_{(k+1)h}^{\leftarrow}\sim\calN\{\exp(h)\what{\mZ}_{kh}^{\leftarrow}+2\phi^{2}(h)(\hat{\vs}_{k}+\hat{\vg}_{k}),\phi^{2}(2h)\mI_{d}\}$.
}
\texttt{\# Truncation and scaling} \\
$\calM\circ\calT_{R}(\what{\mZ}_{T-T_{0}}^{\leftarrow})\leftarrow\mu_{T_{0}}^{-1}\what{\mZ}_{T-T_{0}}^{\leftarrow}\bbone_{B(\bzero,R)}(\what{\mZ}_{T-T_{0}}^{\leftarrow})$. \\
\Return{$\calM\circ\calT_{R}(\what{\mZ}_{T-T_{0}}^{\leftarrow})$}
\end{algorithm}

\section{Convergence Analysis}
\label{section:convergence}

\par In this section, we derive a non-asymptotic convergence rate for the variationally stable Doob's matching~\eqref{eq:erm:GP} and the induced controllable diffusion model~\eqref{eq:Zt:P:brownian:estimation}. Furthermore, we demonstrate that this convergence rate mitigates the curse of dimensionality under mild assumptions.

\subsection{Assumptions}\label{section:convergence:assumption}

\par We begin by outlining the essential technical assumptions required for our theoretical results.

\begin{assumption}[Bounded support]
\label{assumption:bounded:support}
The support of the target distribution $q_{0}$ is a compact set contained within the hypercube $\{\vx_{0}\in\bbR^{d}:|\vx_{0}|_{\infty}\leq 1\}$.
\end{assumption}

\par Assumption~\ref{assumption:bounded:support} is a standard condition imposed on the data distribution~\citep{Lee2023Convergence,Oko2023Diffusion,chang2025deep,beyler2025convergence}. This constraint is well-motivated by practical applications; for instance, image and video data consist of bounded pixel values, thereby satisfying this requirement.

\begin{assumption}[Bounded weight function]\label{assumption:bounded:weight}
The weight function $w$ defined in~\eqref{eq:weight} is bounded from above and bounded away from zero. Specifically, there exist constants $0<\underline{B}<1<\overline{B}<\infty$ such that 
\begin{equation*}
\underline{B}\leq w(\vx)\leq\bar{B}\,, \quad \text{for all}~\vx\in\supp(p_{0})\,.
\end{equation*}
\end{assumption}

\par Assumption~\ref{assumption:bounded:weight} implies that the reference distribution $p_{0}$ and the tilted distribution $q_{0}$ satisfy mutual absolute continuity, ensuring that their supports coincide (i.e., $\supp(p_{0})=\supp(q_{0})$). This condition is crucial for establishing the regularity of Doob's $h$-function. Furthermore, the ratio $\kappa\coloneq\bar{B}/\underline{B}$ serves as a condition number that characterizes the difficulty of the controllable diffusion task, as discussed in the context of posterior sampling by~\citet{purohit2024posterior,ding2025nonlinear,chang2025provable}.

\par Under Assumptions~\ref{assumption:bounded:support} and~\ref{assumption:bounded:weight}, we establish the regularity properties of Doob's $h$-function defined in~\eqref{eq:h:function}.

\begin{proposition}
\label{proposition:regularity:h:function}
Suppose Assumptions~\ref{assumption:bounded:support} and~\ref{assumption:bounded:weight} hold. Then for all $t\in[0,T)$ and $\vx\in\bbR^{d}$, the following bounds hold:
\[
\underline B\leq h_t^*(\vx)\leq\bar B\,,\quad
\max_{1\leq k\leq d}|D_kh_t^*(\vx)|\leq2\sigma_{T-t}^{-2}\bar B\,,\quad
\max_{1\leq k,\ell\leq d}|D_{k\ell}^2h_t^*(\vx)|\leq6\sigma_{T-t}^{-4}\bar B\,,
\]
where $D_{k}$ and $D_{k\ell}^{2}$ denote the first-order and second-order partial derivatives with respect to the input coordinates, respectively.
\end{proposition}

\par The proof of Proposition~\ref{proposition:regularity:h:function} is deferred to Appendix~\ref{appendix:convergence:assumption}. It is worth noting that Proposition~\ref{proposition:regularity:h:function} relies solely on the boundedness of the weight function $w$, without requiring the existence or smoothness of its gradients. Nevertheless, we establish that Doob's $h$-function admits bounded derivatives. This result stems from the definition of Doob's $h$-function as a posterior expectation under a Gaussian likelihood; the inherent smoothness of the Gaussian kernel endows the posterior expectation with strong regularity properties.

\begin{assumption}[Reference score estimation error]
\label{assump:base:score:error}
Treat the pre-trained reference score estimator $\hat{\vs}$ defined in~\eqref{eq:base:score:estimator} as fixed. It satisfies the following error bound:
\begin{equation*}
\frac{1}{T}\sum_{k=0}^{K-1}h\bbE^{\bbP}\big\{|\hat{\vs}(kh,\mX_{kh}^{\leftarrow})-\nabla\log p_{T-kh}(\mX_{kh}^{\leftarrow})|_{2}^{2}\big\}\leq\varepsilon_{\refer}^{2}\,, \qquad \varepsilon_{\refer}\in(0,1).
\end{equation*}
\end{assumption}

\par Assumption~\ref{assump:base:score:error} requires the $L^{2}$-error of the reference score estimator $\hat{\vs}$ to be bounded with respect to the reference path measure $\bbP$. In our setting, where numerous samples from the reference distribution $p_{0}$ are available, estimators satisfying this bound can be obtained via implicit score matching~\citep{hyvarinen2005Estimation}, sliced score matching~\citep{song2020Sliced}, or denoising score matching~\citep{Vincent2011connection}. While one can derive explicit bounds of reference score matching as~\citet{Tang2024Adaptivity,Oko2023Diffusion,fu2024unveil,ding2025characteristic,Yakovlev2025Generalization,yakovlev2025implicit} using non-parametric regression theory for deep neural networks~\citep{bauer2019deep,schmidt2020nonparametric,kohler2021rate,Jiao2023deep}, we adopt this condition to maintain clarity of presentation, following the convention of~\citet{Lee2023Convergence,Chen2023Improved,beyler2025convergence,kremling2025nonasymptotic}.

\subsection{Error bounds for the Doob's guidance estimator}
\label{section:rate:guidance}

\par We measure the complexity of real-valued function classes by their pseudo-dimension, equivalently the VC dimension of their subgraphs~\citep{Vapnik1971Uniform,Anthony1999neural,Bartlett2019nearly}.

\begin{definition}[Pseudo-dimension]
Let $\scrH$ be a class of functions from $\calX$ to $\bbR$. Its subgraph growth function is
\[
\Pi_{\scrH}(m)=\max_{(x_{i},a_{i})\in\calX\times\bbR,\,1\leq i\leq m}
\left|\left\{\big(\bbone\{h(x_{i})>a_{i}\}\big)_{i=1}^m:h\in\scrH\right\}\right|\,.
\]
The pseudo-dimension $\mathrm{Pdim}(\scrH)$ is the largest $m$ for which $\Pi_{\scrH}(m)=2^m$, or infinity if no finite bound exists.
\end{definition}

\par For a differentiable hypothesis class $\scrH$ on $\bbR^d$, write
\[
\mathrm{Pdim}(\nabla\scrH)=\max_{1\leq k\leq d}\mathrm{Pdim}(D_k\scrH)\,,
\qquad D_k\scrH=\{D_kh:h\in\scrH\}\,.
\]
Here $D_k$ denotes differentiation in the $k$-th input coordinate.

\par The following lemma provides an oracle inequality for the variationally stable Doob's matching~\eqref{eq:erm:GP}.

\begin{lemma}[Oracle inequality]
\label{lemma:oracle:inequality:guidance}
Suppose Assumptions~\ref{assumption:bounded:support} and~\ref{assumption:bounded:weight} hold. Fix $t\in[0,T)$, let $n$ be sufficiently large, and let $\lambda>0$. Let $\scrH_t\subset C^{1}(\mathbb {R}^{d})$ be a nonempty class satisfying $\underline {B}\leq h_{t}\leq\bar {B}$ and $\max_{1\leq k\leq d}\|D_{k}h_{t}\|_{\infty}\leq2\bar {B}\sigma_{T-t}^{-2}$ for every $h_{t}\in\scrH_{t}$. Let $\hat h_t^\lambda$ be the gradient-regularized empirical risk minimizer defined as~\eqref{eq:erm:GP}, and let $h_t^*$ be the Doob's $h$-function defined as~\eqref{eq:h:function}. Then the following inequalities hold:
\begin{align*}
&\bbE\big\{\|\hat{h}_{t}^{\lambda}-h_{t}^{*}\|_{L^{2}(p_{T-t})}^{2}\big\}
\lesssim \underbrace{\inf_{h_{t}\in\scrH_{t}}\big\{\|h_{t}-h_{t}^{*}\|_{L^{2}(p_{T-t})}^{2}+\lambda\|\nabla h_{t}-\nabla h_{t}^{*}\|_{L^{2}(p_{T-t})}^{2}\big\}}_{\text{(I)}} \\
&~~~~~~~~~~\quad +\underbrace{\bar{B}^{2}\bigg\{\frac{\mathrm{Pdim}(\scrH_{t})}{n\log^{-1}n}\bigg\}^{1/2}+\frac{\lambda d\bar{B}^{2}}{\sigma_{T-t}^{4}}\bigg\{\frac{\mathrm{Pdim}(\nabla\scrH_{t})}{n\log^{-1}n}\bigg\}^{1/2}}_{\text{(II)}}+\underbrace{\frac{\lambda^{2}d^{2}\bar{B}^{2}}{\sigma_{T-t}^{8}}}_{\text{(III)}}\,, \\
&\bbE\big\{\|\nabla\hat{h}_{t}^{\lambda}-\nabla h_{t}^{*}\|_{L^{2}(p_{T-t})}^{2}\big\}
\lesssim \underbrace{\inf_{h_{t}\in\scrH_{t}}\bigg\{\frac{1}{\lambda}\|h_{t}-h_{t}^{*}\|_{L^{2}(p_{T-t})}^{2}+\|\nabla h_{t}-\nabla h_{t}^{*}\|_{L^{2}(p_{T-t})}^{2}\bigg\}}_{\text{(I)}} \\
&~~~~~~~~~~\quad +\underbrace{\frac{\bar{B}^{2}}{\lambda}\bigg\{\frac{\mathrm{Pdim}(\scrH_{t})}{n\log^{-1}n}\bigg\}^{1/2}+\frac{d\bar{B}^{2}}{\sigma_{T-t}^{4}}\bigg\{\frac{\mathrm{Pdim}(\nabla\scrH_{t})}{n\log^{-1}n}\bigg\}^{1/2}}_{\text{(II)}}+\underbrace{\frac{\lambda d^{2}\bar{B}^{2}}{\sigma_{T-t}^{8}}}_{\text{(III)}}\,,
\end{align*} 
where the notation $\lesssim$ hides absolute constants.
\end{lemma}

\par The proof of Lemma~\ref{lemma:oracle:inequality:guidance} is deferred to Appendix~\ref{appendix:rate:guidance}. Both oracle inequalities for Doob's $h$-function and its gradient decompose the error into three components: approximation error, generalization error, and regularization gap.
\begin{enumerate}[label=(\Roman*)]
\item The \textbf{approximation error} is defined as the minimal $H^{1}$-distance between functions in the hypothesis class $\scrH_{t}$ and the ground-truth Doob's $h$-function $h_{t}^{*}$, measuring the approximation capability of $\scrH_{t}$.
\item The \textbf{generalization error} captures the error arising from finite-sample approximation, which vanishes as the number of samples approaches infinity.
\item The \textbf{regularization gap} is introduced by the gradient regularization in the objective functional, which causes the minimizer of the variationally stable objective~\eqref{eq:h:matching:GP} to deviate from the ground-truth Doob's $h$-function $h_{t}^{*}$~\eqref{eq:h:function}. This gap has been analyzed in Proposition~\ref{proposition:regularization:gap}.
\end{enumerate}

\paragraph{Comparison with oracle inequality of vanilla regression}

\par Lemma~\ref{lemma:oracle:inequality:guidance} is analogous to the oracle inequality found in regression problems. Let $\hat{h}_{t}$ be the vanilla estimator estimated by minimizing~\eqref{eq:erm} over the hypothesis class $\scrH_{t}$. Informally, the following oracle inequality holds:
\begin{equation}\label{eq:h:oracle}
\bbE\big\{\|\hat{h}_{t}-h_{t}^{*}\|_{L^{2}(p_{T-t})}^{2}\big\}\lesssim \underbrace{\inf_{h_{t}\in\scrH_{t}}\|h_{t}-h_{t}^{*}\|_{L^{2}(p_{T-t})}^{2}}_{\text{approximation}}+\underbrace{\bar{B}^{2}\bigg\{\frac{\mathrm{Pdim}(\scrH_{t})}{n\log^{-1}n}\bigg\}^{1/2}}_{\text{generalization}}\,.
\end{equation} 
Comparing~\eqref{eq:h:oracle} with Lemma~\ref{lemma:oracle:inequality:guidance} reveals several crucial differences:
\begin{enumerate}[label=(\roman*)]
\item The approximation error in Lemma~\ref{lemma:oracle:inequality:guidance} is measured in the $H^{1}$-norm, whereas in~\eqref{eq:h:oracle}, it is measured in the $L^{2}$-norm. This distinction is natural because we require the estimator to converge in the $H^{1}$-norm; thus, simultaneous approximation of the function and its derivatives is essential. Simultaneous approximation using neural networks has been investigated in various contexts~\citep{Li2019Better,Guhring2020Error,Guhring2021Approximation,duan2022convergence,duan2022deep,lu2022machine,Shen2022Approximation,shen2024differentiable,Belomestny2023Simultaneous,yakovlev2025simultaneous}.
\item In vanilla regression, the generalization error in~\eqref{eq:h:oracle} depends only on the complexity of the hypothesis class. In contrast, the generalization error bounds in Lemma~\ref{lemma:oracle:inequality:guidance} also depend on the complexity of the derivative classes $\nabla\scrH_{t}$. This occurs because the objective functional of the gradient-regularized Doob's matching~\eqref{eq:h:matching:GP} includes the gradient norm term. Consequently, the error from finite-sample approximation is influenced not only by the complexity of the hypothesis class but also by that of the derivative class.
\item The most significant difference lies in the regularization error. If we focus solely on the oracle inequality for $\hat{h}_{t}^{\lambda}$, letting $\lambda$ go to zero reduces the expression to the vanilla regression oracle inequality~\eqref{eq:h:oracle}. However, the bound for the gradient $\nabla\hat{h}_{t}^{\lambda}$ diverges as the regularization parameter $\lambda$ approaches zero. This highlights the key advantage of our gradient-regularized method: the gradient-regularized is essential for guaranteeing simultaneous convergence of both the estimator value and its gradient. Additionally, there exists a trade-off with respect to $\lambda$ in the oracle inequality for $\hat{h}_{t}^{\lambda}$: a larger $\lambda$ leads to larger regularization error, while reduces the approximation and generalization errors. 
\end{enumerate}

\par Given the oracle inequality for a general hypothesis class $\scrH_{t}$, we consider the specifical case that $\scrH_{t}$ is chosen as a neural network class, with the aim of deriving non-asymptotic convergence rates. We begin by formally defining the neural network class.

\begin{definition}[Neural network class]
A function implemented by a neural network $h:\bbR^{N_{0}}\to\bbR^{N_{L+1}}$ is defined by
\begin{equation*}
h=T_{L}\circ\varrho\circ T_{L-1}\circ\cdots\circ\varrho\circ T_{0}\,,
\end{equation*}
where the activation function $\varrho$ is applied component-wise and $T_{\ell}(\vx):=\mA_{\ell}\vx+\vb_{\ell}$ is an affine transformation with $\mA_{\ell}\in\bbR^{N_{\ell+1}\times N_{\ell}}$ and $\vb_{\ell}\in\bbR^{N_{\ell+1}}$ for $\ell=0,\ldots,L$. We use the ReQU activation $\varrho(u)=[\max\{u,0\}]^{2}$. For the fixed-time classes in this paper, we consider the case where $N_{0}=d$ and $N_{L+1}=1$. The number $L$ is called the depth of neural networks. Additionally, $S:=\sum_{\ell=0}^{L}(|\mA_{\ell}|_{0}+|\vb_{\ell}|_{0})$ represents the total number of non-zero weights within the neural network. We denote by $N(L,S)$ the set of neural networks with depth at most $L$ and the number of non-zero weights at most $S$.
\end{definition}

\par The following theorem establishes the convergence rate of the estimated Doob's guidance given in~\eqref{eq:regressor:guidance}.

\begin{theorem}[Convergence rate of Doob's guidance]\label{theorem:rate:guidance}
Suppose Assumptions~\ref{assumption:bounded:support} and~\ref{assumption:bounded:weight} hold. Fix $t\in[0,T)$ and let $n$ be sufficiently large. Set
\begin{equation}\label{eq:theorem:hypothesis}
\scrH_t=\left\{h_t\in N(L,S):
\begin{aligned}
&\underline {B}/2\leq h_t(\vx)\leq3\bar B/2\quad\text{for all }\vx\in\bbR^d\,,\\
&\max_{1\leq k\leq d}\|D_kh_t\|_{\infty}\leq4\bar {B}\sigma_{T-t}^{-2}
\end{aligned}\right\}\,,
\end{equation}
where $L=\lceil C_{d}\log n\rceil$ and $S=\lceil C_{d}n^{d/(d+8)}\rceil$, with $C_{d}$ depending only on $d$ and sufficiently large. Let $\hat h_t^\lambda$ be the gradient-regularized empirical risk minimizer in~\eqref{eq:erm:GP}, with $\lambda=n^{-2/(d+8)}$, and let $h_t^*$ be defined by~\eqref{eq:h:function}. Then
\[
\bbE\big\{\|\nabla\log\hat h_t^\lambda-\nabla\log h_t^*\|_{L^2(p_{T-t})}^2\big\}
\leq C\sigma_{T-t}^{-8}n^{-2/(d+8)}\log^4 n\,,
\]
where $C$ depends only on $d$, $\bar B$, and $\underline B$, and the expectation is over the training sample.
\end{theorem}

\par The proof of Theorem~\ref{theorem:rate:guidance} is deferred to Appendix~\ref{appendix:rate:guidance}. This theorem demonstrates that the $L^{2}$-error of the Doob's guidance estimator~\eqref{eq:regressor:guidance} converges to the exact Doob's guidance in~\eqref{eq:Zt:P:brownian} as the sample size increases, provided that the size of the neural network is appropriately chosen. However, since the prefactor $\sigma_{T-t}^{-8}$ diverges as $t$ approaches the terminal time $T$, early stopping in controllable diffusion models~\eqref{eq:Zt:P:brownian:estimation} is required to ensure the validity of the Doob's guidance estimator.

\begin{remark}[Comparisons with previous work]
Simultaneous estimation of a function and its gradient using deep neural network has been investigated by~\citet{ding2025Semi}. The most important distinction in our work lies in the elimination of the convexity assumption on the hypothesis class. Specifically,~\citet[Lemma 7]{ding2025Semi} propose an oracle inequality under the assumption that the hypothesis class is convex. Furthermore,~\citet[Theorem 3]{ding2025Semi} requires the estimator to be a minimizer of the gradient-regularized empirical risk over the convex hull of a neural network class, which is intractable in practice. In contrast, Lemma~\ref{lemma:oracle:inequality:guidance} eliminates the requirement of convexity for the hypothesis class, and Theorem~\ref{theorem:rate:guidance} removes the need for the convex hull of the neural network class. This aligns the theoretical analysis more closely with practical computing.
\end{remark}
\subsection{Error bounds for the controllable diffusion models}\label{section:rate:finetuning}

\par In this subsection, we establish a non-asymptotic convergence rate for the controllable diffusion models~\eqref{eq:Zt:P:brownian:estimation}. We begin by presenting an error decomposition for the KL-divergence between the early-stopping distribution $q_{T_{0}}$ and the distribution of $\what{\mZ}_{T-T_{0}}^{\leftarrow}$.

\begin{lemma}[Error decomposition]
\label{lemma:error:decomposition}
Suppose Assumptions~\ref{assumption:bounded:support},~\ref{assumption:bounded:weight}, and~\ref{assump:base:score:error} hold. Use the time grid defined in Section~\ref{section:preliminaries:inference}, with $T\geq1$ and $h\leq1/2$. Conditional on the training data, let $\what q_{T-T_0}$ be the density of the sampler~\eqref{eq:Zt:P:brownian:estimation}, whose Gaussian randomness is independent of these data. If the estimators $\hat {h}_{kh}^{\lambda}$ are positive $C^{1}$ functions and the risks below are finite, then
\begin{align*}
\kl(q_{T_{0}}\,\|\,\what{q}_{T-T_{0}})
&\lesssim\underbrace{\frac{\bar{B}}{\underline{B}}\sum_{k=0}^{K-1}h\bbE^{\bbP}\big\{|\nabla\log\hat{h}_{kh}^{\lambda}(\mX_{kh}^{\leftarrow})-\nabla\log h_{kh}^{*}(\mX_{kh}^{\leftarrow})|_{2}^{2}\big\}}_{\text{(I)}} \\
&\quad+\underbrace{\frac{\bar{B}}{\underline{B}}T\varepsilon_{\refer}^{2}}_{\text{(II)}}+\underbrace{d\exp(-T)}_{\text{(III)}}+\underbrace{\frac{d^{3}T^{2}}{\sigma_{T_{0}}^{8}K}}_{\text{(IV)}}\,,
\end{align*}
where the notation $\lesssim$ hides absolute constants.
\end{lemma}

\par The proof of Lemma~\ref{lemma:error:decomposition} is deferred to Appendix~\ref{appendix:rate:finetuning}. Lemma~\ref{lemma:error:decomposition} decomposes the KL-divergence between the early-stopping distribution $q_{T_{0}}$ and the distribution of $\what{\mZ}_{T-T_{0}}^{\leftarrow}$ into four components: (I) Doob's guidance error, (II) the reference score error, (III) the initialization error, and (IV) the discretization error. Specifically, Doob's guidance error represents the average error of Doob's $h$-guidance estimator at each time point, which has been investigated in Theorem~\ref{theorem:rate:guidance}; the reference score error is the average error of the reference score estimator at each time point, which is discussed in Assumption~\ref{assump:base:score:error}; the initialization error arises from replacing the initial distribution $q_T$ with a Gaussian distribution in~\eqref{eq:Zt:P:brownian:estimation}; and the discretization error is induced by the exponential integrator.

\par While Lemma~\ref{lemma:error:decomposition} characterizes the error between the estimated distribution $\what{q}_{T-T_{0}}$ and the early-stopping distribution $q_{T_{0}}$, our primary interest lies in the discrepancy between $\what{q}_{T-T_{0}}$ and the target tilted distribution $q_{0}$ defined in~\eqref{eq:weight}. Since the KL-divergence does not satisfy the triangular inequality, we instead propose an error bound in Wasserstein-2 distance. The following theorem establishes the Wasserstein-2 distance between the scaled and truncated distribution $(\calM\circ\calT_{R})_{\sharp}\what{q}_{T-T_{0}}$ defined in~\eqref{eq:trancation:scaling} and the target tilted distribution $q_{0}$.

\begin{theorem}[Convergence rate of controllable diffusion models]
\label{theorem:rate:controllable:diffusion}
Suppose Assumptions~\ref{assumption:bounded:support},~\ref{assumption:bounded:weight}, and~\ref{assump:base:score:error} hold. Let $0<\varepsilon\leq e^{-1}$. Set the hypothesis classes $\{\scrH_{kh}\}_{k=0}^{K-1}$ as~\eqref{eq:theorem:hypothesis} with the same depth $L$ and number of non-zero parameters $S$ as Theorem~\ref{theorem:rate:guidance}, with $\lambda=n^{-2/(d+8)}$ and the per-grid minimizers~\eqref{eq:erm:GP}. Let $\what{q}_{T-T_{0}}$ be the marginal density of $\what{\mZ}_{T-T_{0}}^{\leftarrow}$ defined in~\eqref{eq:Zt:P:brownian:estimation}, and let $(\calM\circ\calT_{R})_{\sharp}\what{q}_{T-T_{0}}$ defined as~\eqref{eq:trancation:scaling}. Then it follows that 
\begin{equation}
\bbE\big[ W_{2}^{2}\{q_{0},(\calM\circ\calT_{R})_{\sharp}\what{q}_{T-T_{0}}\}\big]\leq C\varepsilon\log^{3}\bigg(\frac{1}{\varepsilon}\bigg)\,,
\end{equation}
provided that the truncation radius $R$, the terminal time $T$, the step size $h$, the number of steps $K$, the error of reference score $\varepsilon_{\refer}$, the number of samples $n$ for Doob's matching, and the early-stopping time $T_{0}$ are set, respectively, as 
\begin{align*}
&T=\log(\varepsilon^{-2})\,,\qquad T_{0}=-\tfrac{1}{2}\log(1-\varepsilon)\,,\qquad
R^{2}=2(4d\mu_{T_{0}}^{2}+8\sigma_{T_{0}}^{2})\log(2^{d+1}/\varepsilon)\,,\\
&K\gtrsim\varepsilon^{-6}\log^2(\varepsilon^{-1})\,,\qquad h=(T-T_{0})/K\,,\\
&\varepsilon_{\refer}^2\lesssim\varepsilon^2/\log(\varepsilon^{-1})\,,\qquad
n\gtrsim\varepsilon^{-3(d+8)}\log^{(d+8)/2}(\varepsilon^{-1})\,.
\end{align*}
Here $C$ and the constants in the sufficient lower bounds depend only on $d$, $\bar B$ and $\underline B$. The lower bound for $K$ is chosen large enough to ensure $h\leq1/2$. The expectation is over the training sample, and the sampler law is conditional on that sample.
\end{theorem}

\par The proof of Theorem~\ref{theorem:rate:controllable:diffusion} is deferred to Appendix~\ref{appendix:rate:finetuning}. This theorem establishes a non-asymptotic convergence rate for the controllable diffusion model~\eqref{eq:Zt:P:brownian:estimation} using the variationally stable Doob's matching~\eqref{eq:erm:GP}. Crucially, it provides theoretical guidance for selecting hyper-parameters, including the truncation radius $R$ in~\eqref{eq:trancation:scaling}, the step size $h$, the number of steps $K$ in~\eqref{eq:Zt:P:brownian:estimation}, the early stopping time $T_{0}$, the terminal time $T$, the reference score error $\varepsilon_{\refer}$ (Assumption~\ref{assump:base:score:error}), and the sample size $n$ for Doob's matching in~\eqref{eq:h:matching:GP:empirical}.

\par However, this rate suffers from the curse of dimensionality (CoD), since the power of $\varepsilon^{-1}$ in the required sample size grows linearly with the ambient dimension $d$. We address this challenge in the remainder of this section under a low-dimensional subspace assumption.

\subsection{Adaptivity to low-dimensionality}\label{section:rate:low}

\par In this subsection, we demonstrate that the convergence rate mitigates the curse of dimensionality under a low-dimensional subspace assumption, a setting previously explored in~\citet[Section 3]{chen2023Score} and~\citet[Section 6]{Oko2023Diffusion}.

\begin{assumption}[Low-dimensional subspace]\label{assumption:intrinsic}
Let $1\leq d^*\ll d$ be an integer, and $\mP\in\bbR^{d\times d^{*}}$ be a column orthogonal matrix, i.e., $\mP^{\top}\mP=\mI_{ d^{*}}$. Let $\bar{p}_{0}$ be a probability density with a compact support contained within a hypercube $\{\bar{\vx}_{0}\in\bbR^{ d^{*}}:|\bar{\vx}_{0}|_{\infty}\leq 1\}$. The reference distribution $p_{0}$ is a push-forward of $\bar{p}_{0}$ by the linear map $\mP$, i.e., $p_{0}\coloneq \mP_{\sharp}\bar{p}_{0}$.
\end{assumption}

\par Consequently, the reference distribution $p_{0}$ is supported on a linear subspace $\{\mP\bar{\vx}_{0}\in\bbR^{d}:\bar{\vx}_{0}\in\bbR^{d^{*}}\}$ with an ambient dimension $d$, and a much smaller intrinsic dimension $d^{*}\ll d$. See~\cite{ding2025characteristic} for more geometry explanations for this assumption.

\par Before proceeding, we define the forward and time-reversal process in the low-dimensional latent space. Analogously to~\eqref{eq:forward:base}, the forward process reads
\begin{equation}\label{eq:forward:base:low}
\d\bar{\mX}_{t}=-\bar{\mX}_{t}\dt+\sqrt{2}\d\bar{\mB}_{t}\,, \quad t\in(0,T)\,,~\bar{\mX}_{0}\sim \bar{p}_{0}\,,
\end{equation}
where $\bar{\mB}_{t}$ is a $d^{*}$-dimensional standard Brownian motion, and $T>0$ is the terminal time. The transition distribution of this forward process is given by:
\begin{equation}\label{eq:forward:solution:low}
\operatorname{Law}(\bar{\mX}_{t}\given\bar{\mX}_{0}=\bar{\vx}_{0})=\mathcal{N}(\mu_{t}\bar{\vx}_{0},\sigma_{t}^{2}\mI_{d^{*}})\,.
\end{equation}
Let $\bar{p}_{t}$ denote the marginal density of $\bar{\mX}_{t}$ for $t\in(0,T]$. The corresponding time-reversal process~\citep{Anderson1982Reverse} is defined as 
\noeqref{eq:base:reversal:low}
\begin{align}\label{eq:base:reversal:low}
\d\bar{\mX}_{t}^{\leftarrow}=\{\bar{\mX}_{t}^{\leftarrow}+2\nabla\log\bar{p}_{T-t}(\bar{\mX}_{t}^{\leftarrow})\}\dt+\sqrt{2}\d\bar{\mB}_{t}\,,\quad t\in(0,T)\,,\quad\bar{\mX}_{0}^{\leftarrow}\sim\bar{p}_{T}\,.
\end{align}
As shown by~\citet{Anderson1982Reverse}, the path measure of the time-reversal process $(\bar{\mX}_{t}^{\leftarrow})_{0\leq t\leq T}$ corresponds exactly to the reverse of the forward process $(\bar{\mX}_{t})_{0\leq t\leq T}$.

\par The following result establishes a relationship between the ground-truth Doob's $h$-function~\eqref{eq:h:function} and its analogue $\bar{h}_{t}^{*}:\bbR^{d^{*}}\to\bbR$ in the low-dimensional latent space. In other words, it provides a low-dimensional representation of the ground-truth Doob's $h$-function.

\begin{proposition}[Low-dimensional representation]\label{proposition:h:function:low}
Suppose Assumptions~\ref{assumption:bounded:weight} and~\ref{assumption:intrinsic} hold. Then for any $t\in[0,T)$ and $\vx\in\bbR^d$, we have
\begin{equation}\label{eq:doob:h:function:low}
h_t^*(\vx)=\bar h_t^*(\mP^\top\vx)\coloneq
\bbE\big\{w(\mP\bar{\mX}_T^\leftarrow)\given\bar{\mX}_t^\leftarrow=\mP^\top\vx\big\}\,.
\end{equation}
\end{proposition}

\par The proof of Proposition~\ref{proposition:h:function:low} is provided in Appendix~\ref{appendix:rate:low}. Proposition~\ref{proposition:h:function:low} implies that estimating the ground-truth Doob's $h$-function reduces to estimating its low-dimensional counterpart $\bar{h}_{t}^{*}:\bbR^{d^{*}}\rightarrow\bbR$, thereby enabling the Doob's guidance estimator to adapt to low-dimensional structures.

\par Analogously to Proposition~\ref{proposition:regularity:h:function}, we can establish the regularity properties of the low-dimensional representation of Doob's $h$-function $\bar{h}_{t}^{*}:\bbR^{d^{*}}\to\bbR$ in~\eqref{eq:doob:h:function:low}.

\begin{proposition}\label{proposition:regularity:h:function:low}
Suppose Assumptions~\ref{assumption:bounded:weight} and~\ref{assumption:intrinsic} hold. Then for all $t\in[0,T)$ and $\bar{\vx}\in\bbR^{d^{*}}$, the following bounds hold:
\[
\underline B\leq\bar h_t^*(\bar\vx)\leq\bar B\,,\quad
\max_{1\leq k\leq d^{*}}|D_k\bar h_t^*(\bar\vx)|\leq2\sigma_{T-t}^{-2}\bar B\,,\quad
\max_{1\leq k,\ell\leq d^{*}}|D_{k\ell}^2\bar h_t^*(\bar\vx)|\leq6\sigma_{T-t}^{-4}\bar B\,,
\]
where $D_{k}$ and $D_{k\ell}^{2}$ denote the first-order and second-order partial derivatives with respect to the input coordinates, respectively.    
\end{proposition}

\par The proof of Proposition~\ref{proposition:regularity:h:function:low} is provided in Appendix~\ref{appendix:rate:low}. Based on these results, we derive the convergence rates for the variationally stable Doob's matching under the assumption of low-dimensional subspace.

\begin{theorem}[Adaptivity to intrinsic dimension]\label{theorem:rate:guidance:low}
Suppose Assumptions~\ref{assumption:bounded:weight} and~\ref{assumption:intrinsic} hold. Fix $t\in[0,T)$ and let $n$ be sufficiently large. Set
\begin{equation}\label{eq:theorem:hypothesis:low}
\scrH_t=\left\{h_t\in N(L,S):\begin{aligned}
&\underline {B}/2\leq h_t(\vx)\leq3\bar B/2\quad\text{for all }\vx\in\bbR^d\,,\\
&\max_{1\leq k\leq d}\|D_kh_t\|_{\infty}\leq4\sqrt{d^{*}}\,\bar {B}\sigma_{T-t}^{-2}
\end{aligned}\right\}\,,
\end{equation}
where $L=\lceil C_{d^{*}}\log n\rceil$ and $S=\lceil C_{d^{*}}\{n^{d^*/(d^*+8)}+dd^{*}\}\rceil$, with $C_{d^{*}}$ depending only on $d^{*}$ and sufficiently large. Let $\hat h_t^\lambda$ be the gradient-regularized empirical risk minimizer in~\eqref{eq:erm:GP}, with $\lambda=n^{-2/(d^{*}+8)}$, and let $h_t^*$ be defined by~\eqref{eq:h:function}. Then
\[
\bbE\big\{\|\nabla\log\hat h_t^\lambda-\nabla\log h_t^*\|_{L^2(p_{T-t})}^2\big\}
\leq Cd^{2}\sigma_{T-t}^{-8}n^{-2/(d^*+8)}\log^{4} n\,,
\]
where $C$ depends only on $d^*$, $\bar B$, and $\underline B$, and the expectation is over the training sample.
\end{theorem}

\par The proof of Theorem~\ref{theorem:rate:guidance:low} is provided in Appendix~\ref{appendix:rate:low}. This result confirms that the convergence rate eliminates the exponential dependence on the ambient dimension $d$, scaling exponentially solely with the intrinsic dimension $d^{*}\ll d$. This effectively mitigates the curse of dimensionality in Theorem~\ref{theorem:rate:guidance}.

\par The following corollary is a direct consequence of Theorem~\ref{theorem:rate:guidance:low}, derived using arguments similar to those in Theorem~\ref{theorem:rate:controllable:diffusion}.

\begin{corollary}\label{corollary:rate:controllable:diffusion:low}
Suppose Assumptions~\ref{assumption:bounded:weight}--\ref{assumption:intrinsic} hold. Let $0<\varepsilon\leq e^{-1}$. Set the hypothesis classes $\{\scrH_{kh}\}_{k=0}^{K-1}$ as~\eqref{eq:theorem:hypothesis:low}, with $L$, $S$, and $\lambda$ chosen as in Theorem~\ref{theorem:rate:guidance:low}, and use the estimators~\eqref{eq:erm:GP}. Let $\what q_{T-T_0}$ be the marginal density of $\what{\mZ}_{T-T_0}^{\leftarrow}$ defined in~\eqref{eq:Zt:P:brownian:estimation}, and let $(\calM\circ\calT_R)_\sharp\what q_{T-T_0}$ be its post-processed distribution defined in~\eqref{eq:trancation:scaling}. Then it follows that
\[
\bbE\big[\calW_2^2\big\{q_0,(\calM\circ\calT_R)_\sharp\what q_{T-T_0}\big\}\big]\leq Cd^{3}\varepsilon\log^3(\varepsilon^{-1})\,,
\]
provided that the parameters are chosen as
\begin{align*}
&T=\log(\varepsilon^{-2})\,,\qquad T_{0}=-\tfrac{1}{2}\log(1-\varepsilon)\,,\qquad R^{2}=2(4d\mu_{T_{0}}^{2}+8\sigma_{T_{0}}^{2})\log(2^{d+1}/\varepsilon)\,,\\
&K\gtrsim d^2\varepsilon^{-6}\log^2(\varepsilon^{-1})\,,\qquad h=(T-T_{0})/K\,,\\
&\varepsilon_{\refer}^2\lesssim\varepsilon^2/\log(\varepsilon^{-1})\,,\qquad n\gtrsim\varepsilon^{-3(d^*+8)}\log^{(d^{*}+8)/2}(\varepsilon^{-1})\,.
\end{align*}
Here $C$ and the constants in the sufficient lower bounds depend only on $d^*$, $\bar B$, and $\underline B$. The lower bound for $K$ is chosen large enough to ensure $h\leq1/2$. The expectation is over the training sample, and the sampler law is conditional on that sample.
\end{corollary}

\par Crucially, the convergence rates in Corollary~\ref{corollary:rate:controllable:diffusion:low} depend only polynomially on the ambient dimension $d$, while the sample complexity depends exponentially solely on the intrinsic dimension $d^{*}\ll d$. This result significantly mitigates the curse of dimensionality.

\begin{remark}
Adaptivity to low dimensionality plays a pivotal role in the analysis of diffusion and flow-based models. One line of work studies the adaptivity of score or velocity estimator to low dimensionality under Assumption~\ref{assumption:intrinsic} or its variants; see e.g.,~\citet{chen2023Score,Oko2023Diffusion,Yakovlev2025Generalization,ding2025characteristic}. A second line of work focuses on the adaptivity of the sampling procedure; see e.g.,~\citet{li2024adapting,huang2024denoising,potaptchik2025linear}. These works provide valuable insights for extending Corollary~\ref{corollary:rate:controllable:diffusion:low} to achieve provable adaptivity across reference score estimation, sampling, and guidance estimation, thereby completely eliminating dependence on the ambient dimension. We leave this unified analysis, which is outside the scope of the current work, for future research. Importantly, even within the current framework, the dependence on the ambient dimension $d$ remains only polynomial.
\end{remark}

\section{Numerical Experiments}
\label{section:experiments}

\par In this section, we apply our method to reward-guided generation in Section~\ref{section:experiments:reward}, transfer learning for diffusion models in Section~\ref{section:experiments:transfer}, and Bayesian inverse problems in Section~\ref{section:experiments:inverse}.

\subsection{Reward-guided generation}\label{section:experiments:reward}

\par We use the two-dimensional ``grid'' and ``ring'' Gaussian mixtures as reference distributions. The grid example has nine component means on $\{-4,0,4\}^{2}$ with component covariance $0.3\mI_2$; while the ring example has eight equally spaced means on a circle of radius~$4$ with component covariance $0.3\mI_2$ as well. We set weight function $w(\vx)=\exp\{r(\vx)\}$, with
\begin{equation*}
r_{\mathrm{grid}}(\vx)=-x_1^2/100-x_2^2,
\qquad
r_{\mathrm{ring}}(\vx)=-x_1^2-(x_2-1)^2/10.
\end{equation*}
We aim to generate samples from the tilted distribution $q_0(\d\vx)\propto w(\vx)p_0(\d\vx)$. 

\par We compare our proposed method with four well-known baselines, namely reinforcement learning (RL)~\citep{Fan2023DPOK,black2024training}, direct backpropagation~\citep{clark2024directly}, approximate guidance~\citep[DPS]{chung2023diffusion}, and Diffusion Alignment as Sampling~\citep[DAS]{kim2025testtime}. As discussed in Section~\ref{section:introduction}, RL and direct backpropagation are fine-tuning methods that update the reference model using a reverse-KL objective. Approximate guidance and DAS are inference-time alignment methods that keep the reference model fixed. Approximate guidance computes reward gradients through a denoised estimate, while DAS uses tempered sequential Monte Carlo with importance weighting and resampling.

\par For the fairness of comparison, we train the reference diffusion models under the same settings as~\citep[DAS]{kim2025testtime}. Holding these models fixed, we fit a positive scalar network $f_\theta(t,\vx)$ to $n=80{,}000$ reference samples using the normalized form of~\eqref{eq:h:matching:GP:empirical}:
\begin{equation*}
\frac{1}{n}\sum_{i=1}^n\left[
\left(f_\theta(t,\mX_{T-t}^{i})-\frac{w(\mX_0^i)}{\what Z}\right)^2
+\lambda\left\|\nabla_{\vx}f_\theta(t,\mX_{T-t}^{i})\right\|_2^2
\right],\qquad \lambda=10^{-3},
\end{equation*}
where $\mX_{T-t}^{i}=\mu_{T-t}\mX_0^i+\sigma_{T-t}\vepsilon^i$, with $\vepsilon^i\sim\calN(\bzero,\mI_2)$, and the coefficients follow the benchmark noise schedule. The loss fits the normalized weights and penalizes the network's spatial gradient.
We compute $\what Z=n^{-1}\sum_{i=1}^n w(\mX_0^i)$ once from the training samples and hold it fixed, so the regression labels have mean one. With fitted parameters $\hat\theta$, the estimator is $\hat h_t^\lambda(\vx)=\what Z f_{\hat\theta}(t,\vx)$. This constant rescaling leaves the guidance unchanged: $\hat{\vg}_t^\lambda(\vx)=\nabla\log\hat h_t^\lambda(\vx)=\nabla_{\vx}\log f_{\hat\theta}(t,\vx)$.

\par The network $f_\theta$ has three hidden layers, each with $128$ units and ReQU activations, and a Softplus output. We average the loss over diffusion time steps and select checkpoints using the same loss on held-out reference samples. We choose our configurations by comparing generated samples with held-out reference samples weighted by $w$. Target samples are used only for final evaluation. All methods start from the same pre-trained reference model and sample with a $100$-step stochastic DDIM sampler ($\eta=1$). DAS retains $10{,}000$ samples from $100{,}000$ particles, while the other samplers generate $10{,}000$ samples each. We estimate the Wasserstein-1 distance ($W_1$) on $2{,}000$-point subsamples and average results over three generation seeds, keeping all trained networks fixed.

\begin{figure}[ht]
\centering
\includegraphics[width=.66\textwidth]{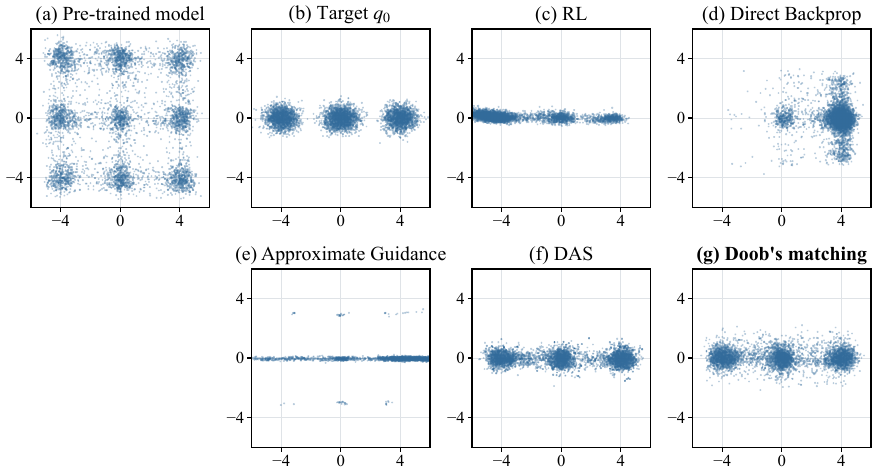}
\caption{Grid example in Section~\ref{section:experiments:reward}: $5{,}000$ displayed samples per distribution. Evaluation includes samples outside the plotting window.}
\label{fig:experiments:das2d}
\end{figure}

\begin{figure}[ht]
\centering
\includegraphics[width=.66\textwidth]{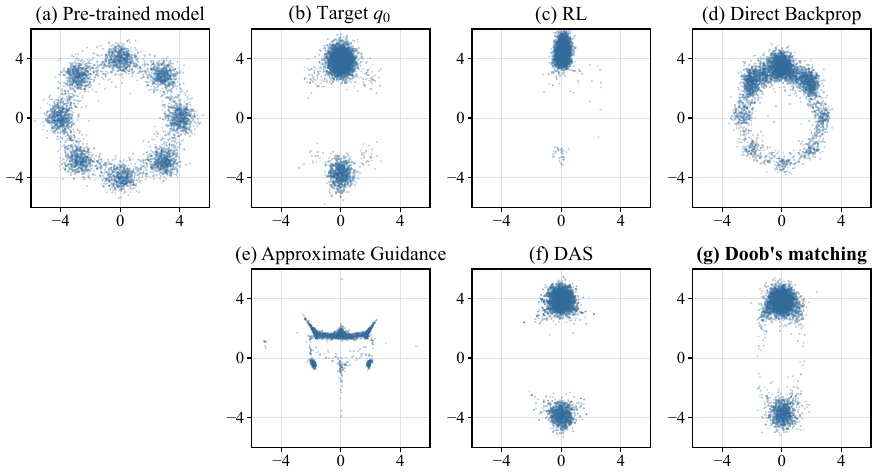}
\caption{Ring example in Section~\ref{section:experiments:reward}: $5{,}000$ displayed samples per distribution. Evaluation includes samples outside the plotting window.}
\label{fig:experiments:das2d:ring}
\end{figure}

\par Figures~\ref{fig:experiments:das2d} and~\ref{fig:experiments:das2d:ring} show that our proposed method recovers the main modes and overall shape of both target distributions. 

\begin{table}[ht]
\centering
\small
\caption{Quanlities are means over three generation seeds with fixed networks. Best values are bold.}
\label{tab:experiments:metrics}
\begin{tabular}{@{}lrrrr@{}}
\toprule
& \multicolumn{2}{c}{Grid} & \multicolumn{2}{c}{Ring} \\
\cmidrule(lr){2-3}\cmidrule(l){4-5}
Method & $W_1\!\downarrow$ & Mean reward $\uparrow$ & $W_1\!\downarrow$ & Mean reward $\uparrow$ \\
\midrule
Pre-trained model & $2.417$ & $-10.463$ & $3.436$ & $-9.120$ \\
\midrule
RL~\citep{Fan2023DPOK,black2024training} & $2.893$ & $\mathbf{-0.209}$ & $1.899$ & $\mathbf{-1.284}$ \\
Direct backpropagation~\citep{clark2024directly} & $3.636$ & $-0.967$ & $1.414$ & $-2.714$ \\
Approximate guidance~\citep{chung2023diffusion} & $18.678$ & $-430.808$ & $10.320$ & $-492.069$ \\
DAS~\citep{kim2025testtime} & $0.303$ & $-0.249$ & $0.602$ & $-1.433$ \\
\textbf{Doob's matching} & $\mathbf{0.247}$ & $-0.339$ & $\mathbf{0.222}$ & $-1.316$ \\
\bottomrule
\end{tabular}
\end{table}

\par As shown by Table~\ref{tab:experiments:metrics}, RL achieves higher reward than our method, but its samples do not match the target distributions well. RL, direct backpropagation, and approximate guidance have much higher $W_1$ values. Using the objective based on Doob's $h$-transform, our method achieves the lowest empirical $W_1$ among the methods compared (Table~\ref{tab:experiments:metrics}).

\subsection{Diffusion transfer learning}\label{section:experiments:transfer}
We employ our method to transfer a pre-trained diffusion model to a target sub-population without retraining its score network. Here the reference distribution $p_{0}$ is the distribution of handwritten digits and the target $q_{0}$ is the distribution of handwritten digit ``9''. We first leverage telescoping ratio estimation~\citep{rhodes2020telescopingdensityratioestimation} to estimate the weight function using samples from the reference distribution $p_{0}$, i.e., images from MNIST, and those from the target distribution $q_{0}$, i.e, images corresponding to ``9'' in MNIST. Then an estimator of Doob's guidance can be achieved through our method.

\begin{figure}[ht]
\centering
\includegraphics[width=0.25\linewidth]{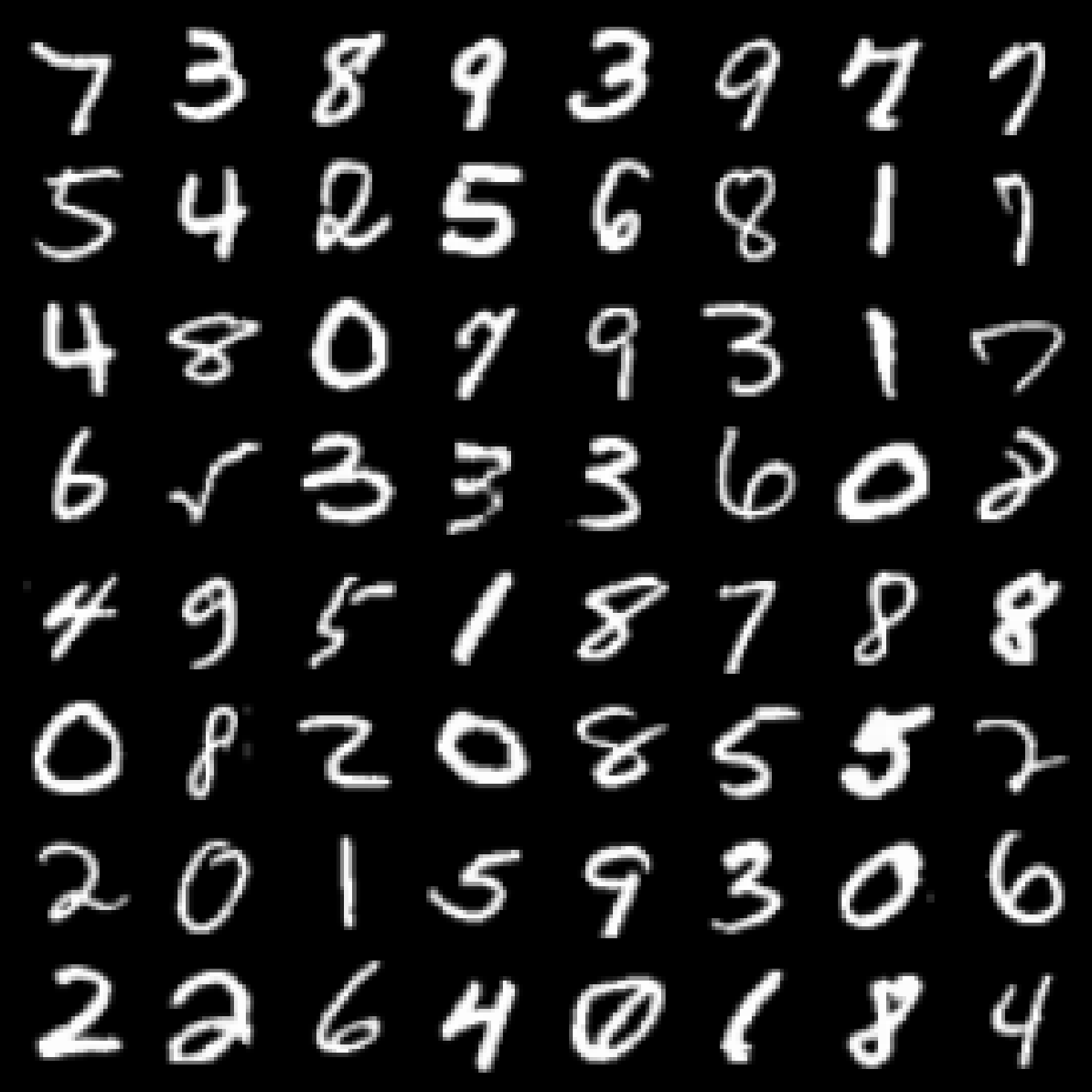}
\includegraphics[width=0.25\linewidth]{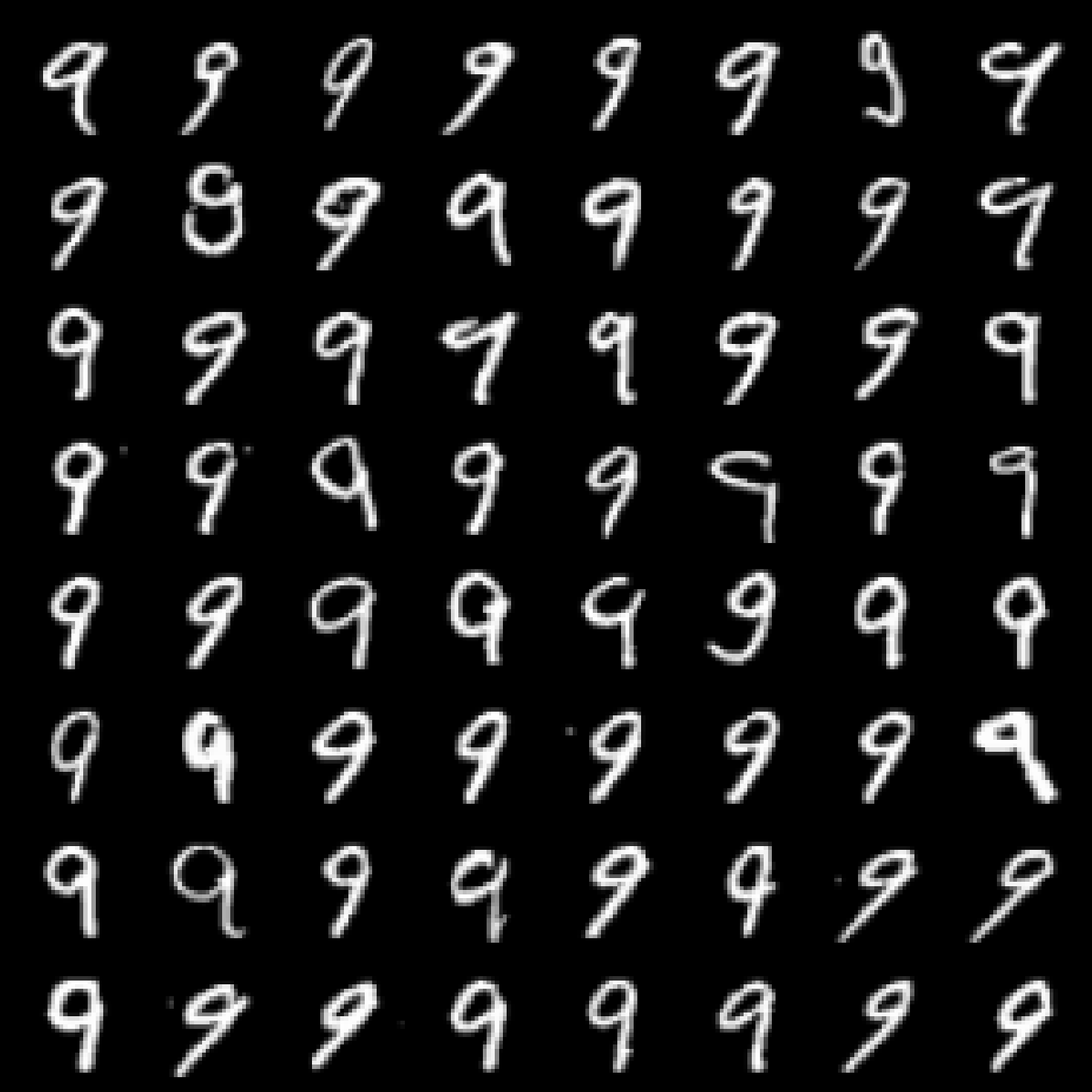}
\includegraphics[width=0.25\linewidth]{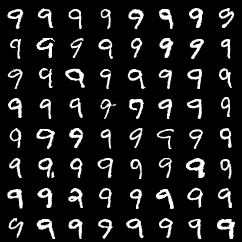}
\caption{Experimental results of diffusion transfer learning in Section~\ref{section:experiments:transfer}. \textbf{(Left)} Samples from the reference distribution. \textbf{(Middle)} Samples from the target distribution. \textbf{(Right)} Samples generated from the diffusion model with Doob's guidance.}
\label{fig:transfer:mnist}
\end{figure}

\par As shown in Figure~\ref{fig:transfer:mnist}, sampling from the trained reference model reproduces the full range of digits (left), whereas adding Doob's guidance yields almost exclusively well-formed nines (right) matching the target distribution (middle). Transfer learning thus reduces to learning a guidance network on top of a once-trained backbone.

\subsection{Diffusion inverse problem}\label{section:experiments:inverse}
We apply our guidance to Bayesian inverse problems. As in Example~\ref{example:bip}, an unknown image $\mX_{0}$ is observed through a noisy measurement $\mY=\mathcal{A}(\mX_{0})+\vn$, where the forward operator $\mathcal{A}$ and the density of the noise $\vn$ are known. The prior $p_{0}$ is the distribution of flower images, represented by samples from the Oxford Flowers-102 dataset\footnote{\url{https://www.kaggle.com/datasets/waseemalastal/the-oxford-flowers-102-dataset}} at resolution $128\times128$. We consider three forward operators: Gaussian deblurring, motion deblurring, and $48\times48$ center inpainting. Each is corrupted by additive Gaussian noise with standard deviation $\sigma=0.05$. Following Algorithm~\ref{alg:controllable:diffusion}, we estimate Doob's $h$-function via variationally stable Doob's matching. We then add the resulting plug-in guidance to the reference score in the controllable reverse process~\eqref{eq:Zt:P:brownian}, which we integrate with a $200$-step EDM sampler. We compare against diffusion posterior sampling (DPS)~\citep{chung2023diffusion}, which approximates Doob's guidance with a biased estimator, as discussed in Section~\ref{section:introduction}.

\begin{figure}[ht]
\centering
\includegraphics[width=0.60\linewidth]{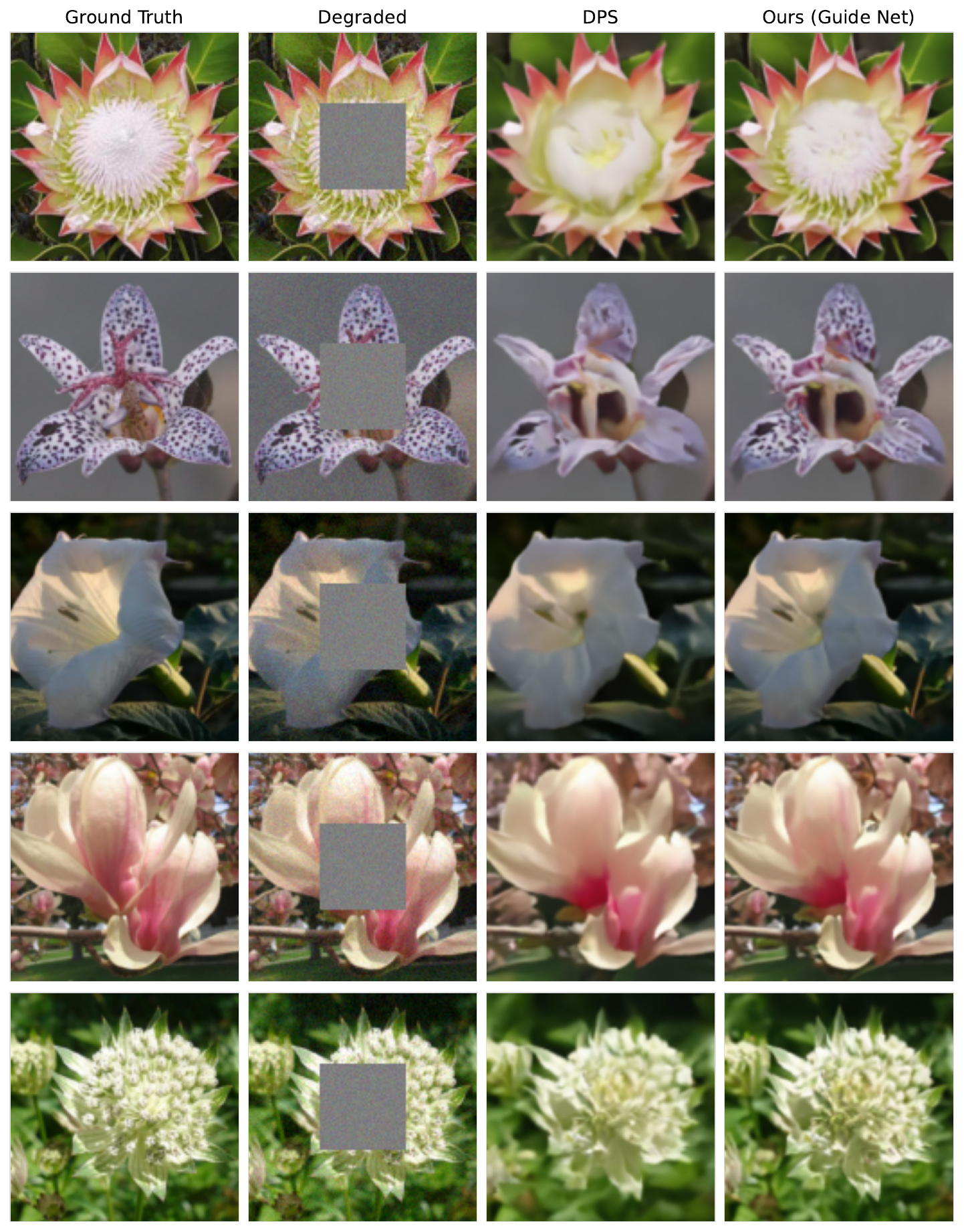}
\caption{Experimental results of inpainting in Section~\ref{section:experiments:inverse}. Columns:
ground truth, measurements, DPS recovery, and ours (Doob's matching).}
\label{fig:grid_inpaint.pdf}
\end{figure}

\par Figure~\ref{fig:grid_inpaint.pdf} shows recoveries for five inpainting examples. Because the masked region is not observed, it must be filled in from the prior, conditioned on the surrounding pixels. DPS either fails to recover fine details in this region or fills it with details that are inconsistent with the ground truth. In contrast, our method produces completions that are coherent with the observed context and closely match the ground truth.

\begin{table}[ht]
\centering
\small
\caption{Quantitative comparison on Oxford Flowers $128\times128$ inverse problems. Mean over 40 test images, 200-step EDM sampler.
$\uparrow$/$\downarrow$: higher/lower is better; \textbf{bold} marks the better of the two methods.}
\label{tab:flowers-inverse}
\begin{tabular}{ll ccc}
\toprule
Task & Method & PSNR$\uparrow$ & SSIM$\uparrow$ & LPIPS$\downarrow$ \\
\midrule
\multirow{2}{*}{Gaussian Deblur}
  & DPS  & 25.54 & 0.757 & 0.155 \\
  & Ours & \textbf{27.35} & \textbf{0.833} & \textbf{0.112} \\
\midrule
\multirow{2}{*}{Motion Deblur}
  & DPS  & 19.85 & 0.495 & 0.239 \\
  & Ours & \textbf{20.60} & \textbf{0.540} & \textbf{0.213} \\
\midrule
\multirow{2}{*}{Inpainting ($48\times48$)}
  & DPS  & 21.90 & 0.761 & 0.161 \\
  & Ours & \textbf{22.50} & \textbf{0.835} & \textbf{0.104} \\
\bottomrule
\end{tabular}
\end{table}

\par Table~\ref{tab:flowers-inverse} reports metrics averaged over $40$ test images. Our method outperforms DPS on all three tasks under all three metrics. The largest PSNR gain is on Gaussian deblurring ($+1.81$\,dB). The largest LPIPS reduction is on inpainting ($0.161\to0.104$), which is consistent with the qualitative differences in Figure~\ref{fig:grid_inpaint.pdf}.

\section{Conclusions}
\label{section:conclusion}

\par In this work, we proposed variationally stable Doob's matching, a principled inference-time alignment framework for diffusion models grounded in the theory of Doob's $h$-transform. Our approach reformulates guidance as the gradient of the logarithm of an underlying Doob's $h$-function, providing a mathematically consistent mechanism for tilting a pre-trained diffusion model toward a target distribution without retraining the reference score network. By leveraging gradient-regularized regression, Doob's matching simultaneously estimates both the $h$-function and its gradient, thereby providing a consistent estimator for Doob's guidance.

\par From a theoretical perspective, we established non-asymptotic convergence rates for the proposed guidance estimator, showing that the estimated Doob's guidance converges to the true guidance under suitable choices of the hypothesis class and regularization parameter. Building on this result, we further derived non-asymptotic convergence guarantees for the induced controllable diffusion process, demonstrating that the generated distribution converges to the target distribution in the Wasserstein-2 distance. These results provide an end-to-end theoretical guarantees for inference-time aligned diffusion models that explicitly account for guidance estimation error, reference score estimation error, initialization bias, and discretization error. Furthermore, we show that our convergence rates depend solely on the intrinsic dimension of the linear subspace rather than the ambient dimension. This highlights the estimator's adaptivity to low-dimensional structures, effectively mitigating the curse of dimensionality. Numerical experiments demonstrate the effectiveness of our proposed method, with lower empirical $W_1$ than DAS in both examples.


\bibliographystyle{plainnat}
\nocite{*}
\bibliography{reference}

@article{Anderson1982Reverse,
  author = {Anderson, Brian D.~O.},
  year = {1982},
  title = {Reverse-time diffusion equation models},
  journal = {Stochastic Processes and their Applications},
  volume = {12},
  pages = {313--326},
}

@book{Anthony1999neural,
  author = {Anthony, Martin and Bartlett, Peter~L},
  year = {1999},
  title = {{Neural network learning: Theoretical foundations}},
  volume = {9},
  publisher = {Cambridge University Press},
}

@misc{bahram2026dogfit,
  author = {Bahram, Yara and Shateri, Mohammadhadi and Granger, Eric},
  year = {2026},
  title = {DogFit: Domain-guided Fine-tuning for Efficient Transfer Learning of Diffusion Models},
  note = {arXiv:2508.05685},
}

@book{Bakry2014Analysis,
  author = {Bakry, Dominique and Gentil, Ivan and Ledoux, Michel},
  year = {2014},
  title = {Analysis and Geometry of Markov Diffusion Operators},
  volume = {348},
  publisher = {Springer Cham},
  series = {Grundlehren der mathematischen Wissenschaften (GL)},
  edition = {first},
}

@article{Bao2024Score,
  author = {Bao, Feng and Zhang, Zezhong and Zhang, Guannan},
  year = {2024},
  title = {A score-based filter for nonlinear data assimilation},
  journal = {Journal of Computational Physics},
  volume = {514},
  pages = {113207},
}

@article{Bartlett2019nearly,
  author = {Bartlett, Peter~L. and Harvey, Nick and Liaw, Christopher and Mehrabian, Abbas},
  year = {2019},
  title = {{Nearly-tight VC-dimension and pseudodimension bounds for piecewise linear neural networks}},
  journal = {Journal of Machine Learning Research},
  volume = {20},
  pages = {1-17},
}

@article{bauer2019deep,
  author = {Bauer, Benedikt and Kohler, Michael},
  year = {2019},
  title = {On deep learning as a remedy for the curse of dimensionality in nonparametric regression},
  journal = {The Annals of Statistics},
  volume = {47},
  pages = {2261--2285},
}

@article{Belomestny2023Simultaneous,
  author = {Belomestny, Denis and Naumov, Alexey and Puchkin, Nikita and Samsonov, Sergey},
  year = {2023},
  title = {Simultaneous approximation of a smooth function and its derivatives by deep neural networks with piecewise-polynomial activations},
  journal = {Neural Networks},
  volume = {161},
  pages = {242--253},
}

@misc{beyler2025convergence,
  author = {Beyler, Eliot and Bach, Francis},
  year = {2025},
  title = {Convergence of Deterministic and Stochastic Diffusion-Model Samplers: A Simple Analysis in Wasserstein Distance},
  note = {arXiv:2508.03210},
}

@inproceedings{black2024training,
  author = {Black, Kevin and Janner, Michael and Du, Yilun and Kostrikov, Ilya and Levine, Sergey},
  year = {2024},
  title = {Training diffusion models with reinforcement learning},
  booktitle = {The Twelfth International Conference on Learning Representations},
}

@misc{Chafai2025Stochastic,
  author = {Chafa{\"i}, Djalil},
  year = {2025},
  title = {Introduction to stochastic calculus},
  howpublished = {Lecture notes, Universit\'e Paris-Dauphine--PSL, 2020--2021 course; corrected version compiled December 7, 2025},
}

@misc{chang2025provable,
  author = {Chang, Jinyuan and Duan, Chenguang and Jiao, Yuling and Li, Ruoxuan and Yang, Jerry~Zhijian and Yuan, Cheng},
  year = {2025},
  title = {Provable Diffusion Posterior Sampling for {Bayesian} Inversion},
  note = {arXiv:2512.08022},
}

@article{chang2025deep,
  author = {Chang, Jinyuan and Ding, Zhao and Jiao, Yuling and Li, Ruoxuan and Yang, Jerry~Zhijian},
  year = {2025},
  title = {Deep conditional distribution learning via conditional F\"ollmer flow},
  journal = {Journal of the American Statistical Association},
  note = {In press},
}

@inproceedings{Chen2023Improved,
  author = {Chen, Hongrui and Lee, Holden and Lu, Jianfeng},
  year = {2023},
  title = {Improved Analysis of Score-based Generative Modeling: User-Friendly Bounds under Minimal Smoothness Assumptions},
  pages = {4735--4763},
  booktitle = {Proceedings of the 40th International Conference on Machine Learning},
}

@inproceedings{chen2023Score,
  author = {Chen, Minshuo and Huang, Kaixuan and Zhao, Tuo and Wang, Mengdi},
  year = {2023},
  title = {Score Approximation, Estimation and Distribution Recovery of Diffusion Models on Low-Dimensional Data},
  volume = {202},
  pages = {4672--4712},
  booktitle = {Proceedings of the 40th International Conference on Machine Learning},
  publisher = {PMLR},
  series = {Proceedings of Machine Learning Research},
}

@misc{chen2025solving,
  author = {Chen, Haoxuan and Ren, Yinuo and Min, Martin~Renqiang and Ying, Lexing and Izzo, Zachary},
  year = {2025},
  title = {Solving Inverse Problems via Diffusion-Based Priors: An Approximation-Free Ensemble Sampling Approach},
  note = {arXiv:2506.03979},
}

@misc{Chewi2025log,
  author = {Chewi, Sinho},
  year = {2025},
  title = {Log-Concave Sampling},
  note = {unfinished draft},
}

@inproceedings{chung2023diffusion,
  author = {Chung, Hyungjin and Kim, Jeongsol and Mccann, Michael~Thompson and Klasky, Marc~Louis and Ye, Jong~Chul},
  year = {2023},
  title = {Diffusion posterior sampling for general noisy inverse problems},
  booktitle = {The Eleventh International Conference on Learning Representations},
}

@inproceedings{clark2024directly,
  author = {Clark, Kevin and Vicol, Paul and Swersky, Kevin and Fleet, David~J.},
  year = {2024},
  title = {Directly fine-tuning diffusion models on differentiable rewards},
  booktitle = {The Twelfth International Conference on Learning Representations},
}

@inproceedings{Denker2024DEFT,
  author = {Denker, Alexander and Vargas, Francisco and Padhy, Shreyas and Didi, Kieran and Mathis, Simon and Dutordoir, Vincent and Barbano, Riccardo and Mathieu, Emile and Komorowska, Urszula~Julia and Lio, Pietro},
  year = {2024},
  title = {DEFT: Efficient Fine-tuning of Diffusion Models by Learning the Generalised h-transform},
  volume = {37},
  pages = {19636--19682},
  booktitle = {Advances in Neural Information Processing Systems},
  publisher = {Curran Associates, Inc.},
}

@misc{denker2025iterative,
  author = {Denker, Alexander and Padhy, Shreyas and Vargas, Francisco and Hertrich, Johannes},
  year = {2025},
  title = {Iterative Importance Fine-tuning of Diffusion Models},
  note = {arXiv:2502.04468},
}

@inproceedings{Dhariwal2021Diffusion,
  author = {Dhariwal, Prafulla and Nichol, Alexander},
  year = {2021},
  title = {Diffusion models beat {GANs} on image synthesis},
  volume = {34},
  pages = {8780--8794},
  booktitle = {Advances in Neural Information Processing Systems},
  publisher = {Curran Associates, Inc.},
}

@misc{ding2025nonlinear,
  author = {Ding, Zhao and Duan, Chenguang and Jiao, Yuling and Yang, Jerry~Zhijian and Yuan, Cheng and Zhang, Pingwen},
  year = {2024},
  title = {Nonlinear assimilation via score-based sequential {L}angevin sampling},
  note = {arXiv:2411.13443},
}

@misc{ding2025characteristic,
  author = {Ding, Zhao and Duan, Chenguang and Jiao, Yuling and Li, Ruoxuan and Yang, Jerry~Zhijian and Zhang, Pingwen},
  year = {2025},
  title = {Characteristic Learning for Provable One Step Generation},
  note = {arXiv:2405.05512},
}

@article{ding2025Semi,
  author = {Ding, Zhao and Duan, Chenguang and Jiao, Yuling and Yang, Jerry~Zhijian},
  year = {2025},
  title = {Semi-supervised deep {S}obolev regression: {E}stimation and variable selection by {ReQU} neural network},
  journal = {IEEE Transactions on Information Theory},
}

@inproceedings{Domingo2024Stochastic,
  author = {Domingo-Enrich, Carles and Han, Jiequn and Amos, Brandon and Bruna, Joan and Chen, Ricky T.~Q.},
  year = {2024},
  title = {Stochastic Optimal Control Matching},
  volume = {37},
  pages = {112459--112504},
  booktitle = {Advances in Neural Information Processing Systems},
  publisher = {Curran Associates, Inc.},
}

@inproceedings{domingo2025adjoint,
  author = {Domingo-Enrich, Carles and Drozdzal, Michal and Karrer, Brian and Chen, Ricky T.~Q.},
  year = {2025},
  title = {Adjoint Matching: Fine-tuning Flow and Diffusion Generative Models with Memoryless Stochastic Optimal Control},
  booktitle = {The Thirteenth International Conference on Learning Representations},
}

@inproceedings{Drucker1991Double,
  author = {Drucker, H. and Le~Cun, Y.},
  year = {1991},
  title = {Double backpropagation increasing generalization performance},
  volume = {2},
  pages = {145--150},
  booktitle = {IJCNN-91-Seattle International Joint Conference on Neural Networks},
}

@article{Drucker1992Improving,
  author = {Drucker, H. and Le~Cun, Y.},
  year = {1992},
  title = {Improving generalization performance using double backpropagation},
  journal = {IEEE Transactions on Neural Networks},
  volume = {3},
  pages = {991--997},
}

@article{duan2022convergence,
  author = {Duan, Chenguang and Jiao, Yuling and Lai, Yanming and Li, Dingwei and Lu, Xiliang and Yang, Jerry~Zhijian},
  year = {2022},
  title = {Convergence Rate Analysis for Deep {Ritz} Method},
  journal = {Communications in Computational Physics},
  volume = {31},
  pages = {1020--1048},
}

@article{duan2022deep,
  author = {Duan, Chenguang and Jiao, Yuling and Lai, Yanming and Lu, Xiliang and Quan, Qimeng and Yang, Jerry~Zhijian},
  year = {2022},
  title = {Deep {Ritz} Methods for {Laplace} Equations with {Dirichlet} Boundary Condition},
  journal = {CSIAM Transactions on Applied Mathematics},
  volume = {3},
  pages = {761--791},
}

@article{E2025deep,
  author = {E, Weinan and Yu, Bing},
  year = {2018},
  title = {The Deep Ritz Method: A Deep Learning-Based Numerical Algorithm for Solving Variational Problems},
  journal = {Communications in Mathematics and Statistics},
  volume = {6},
  pages = {1--12},
}

@inproceedings{Fan2023DPOK,
  author = {Fan, Ying and Watkins, Olivia and Du, Yuqing and Liu, Hao and Ryu, Moonkyung and Boutilier, Craig and Abbeel, Pieter and Ghavamzadeh, Mohammad and Lee, Kangwook and Lee, Kimin},
  year = {2023},
  title = {{DPOK}: {R}einforcement learning for fine-tuning text-to-image diffusion models},
  volume = {36},
  pages = {79858--79885},
  booktitle = {Advances in Neural Information Processing Systems},
  publisher = {Curran Associates, Inc.},
}

@misc{fu2024unveil,
  author = {Fu, Hengyu and Yang, Zhuoran and Wang, Mengdi and Chen, Minshuo},
  year = {2024},
  title = {Unveil Conditional Diffusion Models with Classifier-free Guidance: A Sharp Statistical Theory},
  note = {arXiv:2403.11968},
}

@article{Guhring2020Error,
  author = {G\"{u}hring, Ingo and Kutyniok, Gitta and Petersen, Philipp},
  year = {2020},
  title = {Error bounds for approximations with deep ReLU neural networks in Ws,p norms},
  journal = {Analysis and Applications},
  volume = {18},
  pages = {803--859},
}

@inproceedings{Gao2023Scaling,
  author = {Gao, Leo and Schulman, John and Hilton, Jacob},
  year = {2023},
  title = {Scaling laws for reward model overoptimization},
  volume = {202},
  pages = {10835--10866},
  booktitle = {Proceedings of the 40th International Conference on Machine Learning},
  publisher = {PMLR},
  series = {Proceedings of Machine Learning Research},
}

@article{Guhring2021Approximation,
  author = {G{\"u}hring, Ingo and Raslan, Mones},
  year = {2021},
  title = {Approximation rates for neural networks with encodable weights in smoothness spaces},
  journal = {Neural Networks},
  volume = {134},
  pages = {107--130},
}

@misc{han2024stochastic,
  author = {Han, Yinbin and Razaviyayn, Meisam and Xu, Renyuan},
  year = {2024},
  title = {Stochastic control for fine-tuning diffusion models: {O}ptimality, regularity, and convergence},
  note = {arXiv:2412.18164},
}

@article{HaussmannPardoux1986Time,
  author = {Haussmann, Ulrich G. and Pardoux, \'{E}tienne},
  year = {1986},
  title = {Time reversal of diffusions},
  journal = {The Annals of Probability},
  volume = {14},
  pages = {1188--1205},
}

@article{Heng2024Diffusion,
  author = {Heng, Jeremy and Bortoli, Valentin~De and Doucet, Arnaud},
  year = {2024},
  title = {Diffusion Schr{\"o}dinger Bridges for Bayesian Computation},
  journal = {Statistical Science},
  volume = {39},
  pages = {90 -- 99},
}

@inproceedings{Ho2020Denoising,
  author = {Ho, Jonathan and Jain, Ajay and Abbeel, Pieter},
  year = {2020},
  title = {Denoising diffusion probabilistic models},
  volume = {33},
  pages = {6840--6851},
  booktitle = {Advances in Neural Information Processing Systems},
  publisher = {Curran Associates, Inc.},
}

@inproceedings{ho2021classifierfree,
  author = {Ho, Jonathan and Salimans, Tim},
  year = {2021},
  title = {Classifier-free diffusion guidance},
  booktitle = {NeurIPS 2021 Workshop on Deep Generative Models and Downstream Applications},
}

@article{Hochbruck2005Explicit,
  author = {Hochbruck, Marlis and Ostermann, Alexander},
  year = {2005},
  title = {Explicit exponential Runge-Kutta methods for semilinear parabolic problems},
  journal = {SIAM Journal on Numerical Analysis},
  volume = {43},
  pages = {1069--1090},
}

@article{Hochbruck2010Exponential,
  author = {Hochbruck, Marlis and Ostermann, Alexander},
  year = {2010},
  title = {Exponential integrators},
  journal = {Acta Numerica},
  volume = {19},
  pages = {209--286},
}

@misc{huang2024denoising,
  author = {Huang, Zhihan and Wei, Yuting and Chen, Yuxin},
  year = {2024},
  title = {Denoising diffusion probabilistic models are optimally adaptive to unknown low dimensionality},
  note = {arXiv:2410.18784},
}

@article{hyvarinen2005Estimation,
  author = {Hyv{\"a}rinen, Aapo},
  year = {2005},
  title = {Estimation of non-normalized statistical models by score matching},
  journal = {Journal of Machine Learning Research},
  volume = {6},
  pages = {695--709},
}

@misc{jiao2025unified,
  author = {Jiao, Yuchen and Chen, Yuxin and Li, Gen},
  year = {2025},
  title = {Towards a unified framework for guided diffusion models},
  note = {arXiv:2512.04985},
}

@article{Jiao2023deep,
  author = {Jiao, Yuling and Shen, Guohao and Lin, Yuanyuan and Huang, Jian},
  year = {2023},
  title = {Deep nonparametric regression on approximate manifolds: {N}onasymptotic error bounds with polynomial prefactors},
  journal = {The Annals of Statistics},
  volume = {51},
  pages = {691 -- 716},
}

@article{jiao2024error,
  author = {Jiao, Yuling and Lai, Yanming and Lo, Yisu and Wang, Yang and Yang, Yunfei},
  year = {2024},
  title = {Error analysis of deep Ritz methods for elliptic equations},
  journal = {Analysis and Applications},
  volume = {22},
  pages = {57-87},
}

@article{jiao2025DRM,
  author = {Jiao, Yuling and Li, Ruoxuan and Wu, Peiying and Yang, Jerry~Zhijian and Zhang, Pingwen},
  year = {2025},
  title = {DRM Revisited: A Complete Error Analysis},
  journal = {Journal of Machine Learning Research},
  volume = {26},
  pages = {1--76},
}

@article{Kantas2014Sequential,
  author = {Kantas, Nikolas and Beskos, Alexandros and Jasra, Ajay},
  year = {2014},
  title = {Sequential {M}onte {C}arlo methods for high-dimensional inverse problems: {A} case study for the {N}avier--{S}tokes equations},
  journal = {SIAM/ASA Journal on Uncertainty Quantification},
  volume = {2},
  pages = {464--489},
}

@book{Karatzas1998Brownian,
  author = {Karatzas, Ioannis and Shreve, Steven~E.},
  year = {1998},
  title = {Brownian Motion and Stochastic Calculus},
  volume = {113},
  publisher = {Springer New York, NY},
  series = {Graduate Texts in Mathematics (GTM)},
  edition = {second},
}

@inproceedings{kim2025testtime,
  author = {Kim, Sunwoo and Kim, Minkyu and Park, Dongmin},
  year = {2025},
  title = {Test-time alignment of diffusion models without reward over-optimization},
  booktitle = {The Thirteenth International Conference on Learning Representations},
}

@article{kohler2021rate,
  author = {Kohler, Michael and Langer, Sophie},
  year = {2021},
  title = {On the rate of convergence of fully connected deep neural network regression estimates},
  journal = {The Annals of Statistics},
  volume = {49},
  pages = {2231--2249},
}

@misc{kremling2025nonasymptotic,
  author = {Kremling, Gitte and Iafrate, Francesco and Taheri, Mahsa and Lederer, Johannes},
  year = {2025},
  title = {Non-asymptotic error bounds for probability flow ODEs under weak log-concavity},
  note = {arXiv:2510.17608},
}

@inproceedings{Lee2023Convergence,
  author = {Lee, Holden and Lu, Jianfeng and Tan, Yixin},
  year = {2023},
  title = {Convergence of score-based generative modeling for general data distributions},
  volume = {201},
  pages = {946--985},
  booktitle = {Proceedings of The 34th International Conference on Algorithmic Learning Theory},
  publisher = {PMLR},
  series = {Proceedings of Machine Learning Research},
}

@misc{lee2023aligning,
  author = {Lee, Kimin and Liu, Hao and Ryu, Moonkyung and Watkins, Olivia and Du, Yuqing and Boutilier, Craig and Abbeel, Pieter and Ghavamzadeh, Mohammad and Gu, Shixiang~Shane},
  year = {2023},
  title = {Aligning text-to-image models using human feedback},
  note = {arXiv:2302.12192},
}

@article{Li2025State,
  author = {Li, Zhuoyuan and Dong, Bin and Zhang, Pingwen},
  year = {2025},
  title = {State-observation augmented diffusion model for nonlinear assimilation with unknown dynamics},
  journal = {Journal of Computational Physics},
  volume = {539},
  pages = {114240},
}

@article{Li2019Better,
  author = {Li, Bo and Tang, Shanshan and Yu, Haijun},
  year = {2019},
  title = {Better approximations of high dimensional smooth functions by deep neural networks with rectified power units},
  journal = {Communications in Computational Physics},
  volume = {27},
  pages = {379--411},
}

@inproceedings{li2024adapting,
  author = {Li, Gen and Yan, Yuling},
  year = {2024},
  title = {Adapting to Unknown Low-Dimensional Structures in Score-Based Diffusion Models},
  booktitle = {The Thirty-eighth Annual Conference on Neural Information Processing Systems},
}

@inproceedings{Lu2022DPMv1,
  author = {Lu, Cheng and Zhou, Yuhao and Bao, Fan and Chen, Jianfei and LI, Chongxuan and Zhu, Jun},
  year = {2022},
  title = {DPM-Solver: A fast ODE solver for diffusion probabilistic model sampling in around 10 steps},
  volume = {35},
  pages = {5775--5787},
  booktitle = {Advances in Neural Information Processing Systems},
  publisher = {Curran Associates, Inc.},
}

@inproceedings{lu2022machine,
  author = {Lu, Yiping and Chen, Haoxuan and Lu, Jianfeng and Ying, Lexing and Blanchet, Jose},
  year = {2022},
  title = {Machine Learning For Elliptic {PDE}s: Fast Rate Generalization Bound, Neural Scaling Law and Minimax Optimality},
  booktitle = {International Conference on Learning Representations},
}

@inproceedings{mardani2024a,
  author = {Mardani, Morteza and Song, Jiaming and Kautz, Jan and Vahdat, Arash},
  year = {2024},
  title = {A variational perspective on solving inverse problems with diffusion models},
  booktitle = {The Twelfth International Conference on Learning Representations},
}

@inproceedings{martin2025pnpflow,
  author = {Martin, S{\'e}gol{\`e}ne~Tiffany and Gagneux, Anne and Hagemann, Paul and Steidl, Gabriele},
  year = {2025},
  title = {PnP-Flow: Plug-and-Play Image Restoration with Flow Matching},
  booktitle = {The Thirteenth International Conference on Learning Representations},
}

@misc{mou2025rlfinetuning,
  author = {Mou, Wenlong},
  year = {2025},
  title = {Is RL fine-tuning harder than regression? A PDE learning approach for diffusion models},
  note = {arXiv:2509.02528},
}

@inproceedings{Oko2023Diffusion,
  author = {Oko, Kazusato and Akiyama, Shunta and Suzuki, Taiji},
  year = {2023},
  title = {Diffusion models are minimax optimal distribution estimators},
  volume = {202},
  pages = {26517--26582},
  booktitle = {Proceedings of the 40th International Conference on Machine Learning},
  publisher = {PMLR},
  series = {Proceedings of Machine Learning Research},
}

@inproceedings{ouyang2022training,
  author = {Ouyang, Long and Wu, Jeffrey and Jiang, Xu and Almeida, Diogo and Wainwright, Carroll and Mishkin, Pamela and Zhang, Chong and Agarwal, Sandhini and Slama, Katarina and Ray, Alex and Schulman, John and Hilton, Jacob and Kelton, Fraser and Miller, Luke and Simens, Maddie and Askell, Amanda and Welinder, Peter and Christiano, Paul~F and Leike, Jan and Lowe, Ryan},
  year = {2022},
  title = {Training language models to follow instructions with human feedback},
  volume = {35},
  pages = {27730--27744},
  booktitle = {Advances in Neural Information Processing Systems},
  publisher = {Curran Associates, Inc.},
}

@inproceedings{Ouyang2024Transfer,
  author = {Ouyang, Yidong and Xie, Liyan and Zha, Hongyuan and Cheng, Guang},
  year = {2024},
  title = {Transfer learning for diffusion models},
  volume = {37},
  pages = {136962--136989},
  booktitle = {Advances in Neural Information Processing Systems},
  publisher = {Curran Associates, Inc.},
}

@misc{pachebat2025iterative,
  author = {Pachebat, Jean and Conforti, Giovanni and Durmus, Alain and Janati, Yazid},
  year = {2025},
  title = {Iterative Tilting for Diffusion Fine-Tuning},
  note = {arXiv:2512.03234},
}

@misc{potaptchik2025linear,
  author = {Potaptchik, Peter and Azangulov, Iskander and Deligiannidis, George},
  year = {2025},
  title = {Linear Convergence of Diffusion Models Under the Manifold Hypothesis},
  note = {arXiv:2410.09046},
}

@inproceedings{purohit2024posterior,
  author = {Purohit, Vishal and Repasky, Matthew and Lu, Jianfeng and Qiu, Qiang and Xie, Yao and Cheng, Xiuyuan},
  year = {2025},
  title = {Consistency Posterior Sampling for Diverse Image Synthesis},
  booktitle = {The IEEE/CVF Conference on Computer Vision and Pattern Recognition (CVPR)},
}

@inproceedings{Rafailov2023Direct,
  author = {Rafailov, Rafael and Sharma, Archit and Mitchell, Eric and Manning, Christopher~D and Ermon, Stefano and Finn, Chelsea},
  year = {2023},
  title = {Direct preference optimization: {Y}our language model is secretly a reward model},
  volume = {36},
  pages = {53728--53741},
  booktitle = {Advances in Neural Information Processing Systems},
  publisher = {Curran Associates, Inc.},
}

@inproceedings{Rafailov2024Scaling,
  author = {Rafailov, Rafael and Chittepu, Yaswanth and Park, Ryan and Sikchi, Harshit~Sushil and Hejna, Joey and Knox, Brad and Finn, Chelsea and Niekum, Scott},
  year = {2024},
  title = {Scaling laws for reward model overoptimization in direct alignment algorithms},
  volume = {37},
  pages = {126207--126242},
  booktitle = {Advances in Neural Information Processing Systems},
  publisher = {Curran Associates, Inc.},
}

@inproceedings{Ramesh2021Zero,
  author = {Ramesh, Aditya and Pavlov, Mikhail and Goh, Gabriel and Gray, Scott and Voss, Chelsea and Radford, Alec and Chen, Mark and Sutskever, Ilya},
  year = {2021},
  title = {Zero-shot text-to-image generation},
  volume = {139},
  pages = {8821--8831},
  booktitle = {Proceedings of the 38th International Conference on Machine Learning},
  publisher = {PMLR},
  series = {Proceedings of Machine Learning Research},
}

@misc{ren2025drift,
  author = {Ren, Yinuo and Gao, Wenhao and Ying, Lexing and Rotskoff, Grant~M. and Han, Jiequn},
  year = {2025},
  title = {DriftLite: Lightweight Drift Control for Inference-Time Scaling of Diffusion Models},
  note = {arXiv:2509.21655},
}

@misc{rhodes2020telescopingdensityratioestimation,
  author = {Rhodes, Benjamin and Xu, Kai and Gutmann, Michael~U.},
  year = {2020},
  title = {Telescoping Density-Ratio Estimation},
}

@book{Rogers2000Diffusions,
  author = {Rogers, L.~C.~G. and Williams, David},
  year = {2000},
  title = {Diffusions, Markov Processes and Martingales},
  volume = {2: It{\^o} Calculus},
  publisher = {Cambridge University Press},
  series = {Cambridge Mathematical Library},
  edition = {2},
}

@misc{sabour2025test,
  author = {Sabour, Amirmojtaba and Albergo, Michael~S. and Domingo-Enrich, Carles and Boffi, Nicholas~M. and Fidler, Sanja and Kreis, Karsten and Vanden-Eijnden, Eric},
  year = {2025},
  title = {Test-time scaling of diffusions with flow maps},
  note = {arXiv:2511.22688},
}

@article{schmidt2020nonparametric,
  author = {Schmidt-Hieber, Johannes},
  year = {2020},
  title = {Nonparametric regression using deep neural networks with {ReLU} activation function},
  journal = {The Annals of Statistics},
  volume = {48},
  pages = {1875--1897},
}

@inproceedings{Shen2022Approximation,
  author = {Shen, Guohao and Jiao, Yuling and Lin, Yuanyuan and Huang, Jian},
  year = {2022},
  title = {Approximation with CNNs in Sobolev Space: with Applications to Classification},
  volume = {35},
  pages = {2876--2888},
  booktitle = {Advances in Neural Information Processing Systems},
  publisher = {Curran Associates, Inc.},
}

@misc{shen2024differentiable,
  author = {Shen, Guohao and Jiao, Yuling and Lin, Yuanyuan and Huang, Jian},
  year = {2024},
  title = {Differentiable Neural Networks with RePU Activation: with Applications to Score Estimation and Isotonic Regression},
  note = {arXiv:2305.00608},
}

@inproceedings{si2025latentensf,
  author = {Si, Phillip and Chen, Peng},
  year = {2025},
  title = {Latent-En{SF}: A Latent Ensemble Score Filter for High-Dimensional Data Assimilation with Sparse Observation Data},
  booktitle = {The Thirteenth International Conference on Learning Representations},
}

@inproceedings{Sohl2015Deep,
  author = {Sohl-Dickstein, Jascha and Weiss, Eric and Maheswaranathan, Niru and Ganguli, Surya},
  year = {2015},
  title = {Deep Unsupervised Learning using Nonequilibrium Thermodynamics},
  volume = {37},
  pages = {2256--2265},
  booktitle = {Proceedings of the 32nd International Conference on Machine Learning},
  publisher = {PMLR},
  series = {Proceedings of Machine Learning Research},
  address = {Lille, France},
}

@inproceedings{song2019Generative,
  author = {Song, Yang and Ermon, Stefano},
  year = {2019},
  title = {Generative Modeling by Estimating Gradients of the Data Distribution},
  volume = {32},
  booktitle = {Advances in Neural Information Processing Systems},
  publisher = {Curran Associates, Inc.},
}

@inproceedings{song2020Sliced,
  author = {Song, Yang and Garg, Sahaj and Shi, Jiaxin and Ermon, Stefano},
  year = {2020},
  title = {Sliced score matching: {A} scalable approach to density and score estimation},
  volume = {115},
  pages = {574--584},
  booktitle = {Proceedings of The 35th Uncertainty in Artificial Intelligence Conference},
  publisher = {PMLR},
  series = {Proceedings of Machine Learning Research},
}

@inproceedings{song2021scorebased,
  author = {Song, Yang and Sohl-Dickstein, Jascha and Kingma, Diederik~P and Kumar, Abhishek and Ermon, Stefano and Poole, Ben},
  year = {2021},
  title = {Score-based generative modeling through stochastic differential equations},
  booktitle = {International Conference on Learning Representations},
}

@inproceedings{song2022solving,
  author = {Song, Yang and Shen, Liyue and Xing, Lei and Ermon, Stefano},
  year = {2022},
  title = {Solving inverse problems in medical imaging with score-based generative models},
  booktitle = {International Conference on Learning Representations},
}

@inproceedings{song2023Loss,
  author = {Song, Jiaming and Zhang, Qinsheng and Yin, Hongxu and Mardani, Morteza and Liu, Ming-Yu and Kautz, Jan and Chen, Yongxin and Vahdat, Arash},
  year = {2023},
  title = {Loss-guided diffusion models for plug-and-play controllable generation},
  volume = {202},
  pages = {32483--32498},
  booktitle = {Proceedings of the 40th International Conference on Machine Learning},
  publisher = {PMLR},
  series = {Proceedings of Machine Learning Research},
}

@inproceedings{Stiennon2020Learning,
  author = {Stiennon, Nisan and Ouyang, Long and Wu, Jeffrey and Ziegler, Daniel and Lowe, Ryan and Voss, Chelsea and Radford, Alec and Amodei, Dario and Christiano, Paul~F},
  year = {2020},
  title = {Learning to summarize with human feedback},
  volume = {33},
  pages = {3008--3021},
  booktitle = {Advances in Neural Information Processing Systems},
  publisher = {Curran Associates, Inc.},
}

@article{Stuart2010Inverse,
  author = {Stuart, Andrew~M},
  year = {2010},
  title = {Inverse problems: A {Bayesian} perspective},
  journal = {Acta Numerica},
  volume = {19},
  pages = {451--559},
}

@book{Sarkka2019applied,
  author = {S{\"a}rkk{\"a}, Simo and Solin, Arno},
  year = {2019},
  title = {Applied Stochastic Differential Equations},
  publisher = {Cambridge University Press},
  series = {Institute of Mathematical Statistics Textbooks},
}

@misc{Tang2024stochastic,
  author = {Tang, Wenping and Xu, Renyuan},
  year = {2024},
  title = {A stochastic analysis approach to conditional diffusion guidance},
  note = {Columbia University Preprint},
}

@inproceedings{Tang2024Adaptivity,
  author = {Tang, Rong and Yang, Yun},
  year = {2024},
  title = {Adaptivity of Diffusion Models to Manifold Structures},
  volume = {238},
  pages = {1648--1656},
  booktitle = {Proceedings of The 27th International Conference on Artificial Intelligence and Statistics},
  publisher = {PMLR},
  series = {Proceedings of Machine Learning Research},
}

@misc{tang2025finetuning,
  author = {Tang, Wenpin and Zhou, Fuzhong},
  year = {2025},
  title = {Fine-tuning of diffusion models via stochastic control: entropy regularization and beyond},
  note = {arXiv:2403.06279},
}

@misc{uehara2024finetuning,
  author = {Uehara, Masatoshi and Zhao, Yulai and Black, Kevin and Hajiramezanali, Ehsan and Scalia, Gabriele and Diamant, Nathaniel~Lee and Tseng, Alex~M and Biancalani, Tommaso and Levine, Sergey},
  year = {2024},
  title = {Fine-Tuning of Continuous-Time Diffusion Models as Entropy-Regularized Control},
  note = {arXiv:2402.15194},
}

@inproceedings{Uehara2024Feedback,
  author = {Uehara, Masatoshi and Zhao, Yulai and Black, Kevin and Hajiramezanali, Ehsan and Scalia, Gabriele and Diamant, Nathaniel~Lee and Tseng, Alex~M and Levine, Sergey and Biancalani, Tommaso},
  year = {2024},
  title = {Feedback efficient online fine-tuning of diffusion models},
  volume = {235},
  pages = {48892--48918},
  booktitle = {Proceedings of the 41st International Conference on Machine Learning},
  publisher = {PMLR},
  series = {Proceedings of Machine Learning Research},
}

@misc{uehara2025reward,
  author = {Uehara, Masatoshi and Su, Xingyu and Zhao, Yulai and Li, Xiner and Regev, Aviv and Ji, Shuiwang and Levine, Sergey and Biancalani, Tommaso},
  year = {2025},
  title = {Reward-guided iterative refinement in diffusion models at test-time with applications to protein and {DNA} design},
  note = {arXiv:2502.14944},
}

@misc{uehara2025inference,
  author = {Uehara, Masatoshi and Zhao, Yulai and Wang, Chenyu and Li, Xiner and Regev, Aviv and Levine, Sergey and Biancalani, Tommaso},
  year = {2025},
  title = {Inference-time alignment in diffusion models with reward-guided generation: {T}utorial and review},
  note = {arXiv:2501.09685},
}

@article{Vapnik1971Uniform,
  author = {Vapnik, V.~N. and Chervonenkis, A.~Ya.},
  year = {1971},
  title = {On the Uniform Convergence of Relative Frequencies of Events to Their Probabilities},
  journal = {Theory of Probability \& Its Applications},
  volume = {16},
  pages = {264--280},
}

@inproceedings{Vempala2019Rapid,
  author = {Vempala, Santosh and Wibisono, Andre},
  year = {2019},
  title = {Rapid convergence of the unadjusted {L}angevin algorithm: isoperimetry suffices},
  volume = {32},
  booktitle = {Advances in Neural Information Processing Systems},
  publisher = {Curran Associates, Inc.},
}

@article{Vincent2011connection,
  author = {Vincent, Pascal},
  year = {2011},
  title = {A connection between score matching and denoising autoencoders},
  journal = {Neural Computation},
  volume = {23},
  pages = {1661--1674},
}

@inproceedings{wang2024Bridging,
  author = {Wang, Xiyu and Lin, Baijiong and Liu, Daochang and Chen, Ying-Cong and Xu, Chang},
  year = {2024},
  title = {Bridging Data Gaps in Diffusion Models with Adversarial Noise-Based Transfer Learning},
  volume = {235},
  pages = {50944--50959},
  booktitle = {Proceedings of the 41st International Conference on Machine Learning},
  publisher = {PMLR},
  series = {Proceedings of Machine Learning Research},
}

@inproceedings{Xu2023ImageReward,
  author = {Xu, Jiazheng and Liu, Xiao and Wu, Yuchen and Tong, Yuxuan and Li, Qinkai and Ding, Ming and Tang, Jie and Dong, Yuxiao},
  year = {2023},
  title = {{ImageReward}: {L}earning and evaluating human preferences for text-to-image generation},
  volume = {36},
  pages = {15903--15935},
  booktitle = {Advances in Neural Information Processing Systems},
  publisher = {Curran Associates, Inc.},
}

@misc{yakovlev2025implicit,
  author = {Yakovlev, Konstantin and Markovich, Anna and Puchkin, Nikita},
  year = {2025},
  title = {Implicit score matching meets denoising score matching: improved rates of convergence and log-density Hessian estimation},
  note = {arXiv:2512.24378},
}

@inproceedings{Yakovlev2025Generalization,
  author = {Yakovlev, Konstantin and Puchkin, Nikita},
  year = {2025},
  title = {Generalization error bound for denoising score matching under relaxed manifold assumption},
  volume = {291},
  pages = {5824--5891},
  booktitle = {Proceedings of Thirty Eighth Conference on Learning Theory},
  publisher = {PMLR},
  series = {Proceedings of Machine Learning Research},
}

@misc{yakovlev2025simultaneous,
  author = {Yakovlev, Konstantin and Puchkin, Nikita},
  year = {2025},
  title = {Simultaneous Approximation of the Score Function and Its Derivatives by Deep Neural Networks},
  note = {arXiv:2512.23643},
}

@inproceedings{yuan2023rewarddirected,
  author = {Yuan, Hui and Huang, Kaixuan and Ni, Chengzhuo and Chen, Minshuo and Wang, Mengdi},
  year = {2023},
  title = {Reward-directed conditional diffusion: {P}rovable distribution estimation and reward improvement},
  booktitle = {Thirty-seventh Conference on Neural Information Processing Systems},
}

@inproceedings{zhang2023fast,
  author = {Zhang, Qinsheng and Chen, Yongxin},
  year = {2023},
  title = {Fast sampling of diffusion models with exponential integrator},
  booktitle = {The Eleventh International Conference on Learning Representations},
}

@inproceedings{Zhang2023Adding,
  author = {Zhang, Lvmin and Rao, Anyi and Agrawala, Maneesh},
  year = {2023},
  title = {Adding conditional control to text-to-image diffusion models},
  pages = {3836-3847},
  booktitle = {Proceedings of the IEEE/CVF International Conference on Computer Vision (ICCV)},
}

@misc{zhong2025domain,
  author = {Zhong, Jincheng and Zhang, Xiangcheng and Wang, Jianmin and Long, Mingsheng},
  year = {2025},
  title = {Domain Guidance: A Simple Transfer Approach for a Pre-trained Diffusion Model},
  note = {arXiv:2504.01521},
}

@book{Oksendal2003Stochastic,
  author = {{\O}ksendal, Bernt},
  year = {2003},
  title = {Stochastic Differential Equations: An Introduction with Applications},
  publisher = {Springer Berlin, Heidelberg},
  series = {Universitext (UTX)},
  edition = {sixth},
}

@article{Bartlett2002Rademacher,
  author = {Bartlett, Peter~L. and Mendelson, Shahar},
  year = {2002},
  title = {Rademacher and Gaussian Complexities: Risk Bounds and Structural Results},
  journal = {Journal of Machine Learning Research},
  volume = {3},
  pages = {463--482},
}

@misc{Baudoin2012Girsanov,
  author = {Baudoin, Fabrice},
  year = {2012},
  title = {Lecture 25. Girsanov theorem},
  howpublished = {Research and Lecture Notes, October 2, 2012},
}

@book{Brezis2011Functional,
  author = {Brezis, Haim},
  year = {2011},
  title = {Functional Analysis, Sobolev Spaces and Partial Differential Equations},
  publisher = {Springer New York, NY},
  series = {Universitext (UTX)},
  edition = {first},
}

@book{Durrett2019Probability,
  author = {Durrett, Rick},
  year = {2019},
  title = {Probability: Theory and Examples},
  publisher = {Cambridge University Press},
  edition = {5th},
}

@book{mohri2018foundations,
  author = {Mohri, Mehryar and Rostamizadeh, Afshin and Talwalkar, Ameet},
  year = {2018},
  title = {{Foundations of Machine Learning}},
  publisher = {MIT press},
  edition = {{Second}},
}

@book{Villani2009Optimal,
  author = {Villani, C{\'e}dric},
  year = {2009},
  title = {Optimal Transport: Old and New},
  volume = {338},
  publisher = {Springer Berlin, Heidelberg},
  series = {Grundlehren der mathematischen Wissenschaften (GL)},
  edition = {first},
}

\clearpage
\appendix
\section*{Outline of Appendices}
The appendices are organized as follows.
\begin{enumerate}[label=(\roman*)]
\item \textbf{Appendix~\ref{appendix:section:method}} collects the derivations for Section~\ref{section:method} (controllable diffusion models and Doob's $h$-transform).
\item \textbf{Appendix~\ref{appendix:doob:matching}} collects the derivations for Section~\ref{section:doob:matching} (variationally stable Doob's matching).
\item \textbf{Appendix~\ref{appendix:convergence:assumption}} discusses the assumptions in Section~\ref{section:convergence:assumption}.
\item \textbf{Appendix~\ref{appendix:rate:guidance}} proves the error bounds for the Doob's guidance estimator in Section~\ref{section:rate:guidance}.
\item \textbf{Appendix~\ref{appendix:rate:finetuning}} proves the error bounds for the controllable diffusion models in Section~\ref{section:rate:finetuning}.
\item \textbf{Appendix~\ref{appendix:rate:low}} proves the adaptivity to low-dimensionality in Section~\ref{section:rate:low}.
\item \textbf{Appendix~\ref{appendix:auxiliary:lemmas}} collects auxiliary lemmas used throughout the proofs.
\end{enumerate}

\section{Derivations in Section~\ref{section:method}}
\label{appendix:section:method}

\begin{customproposition}{\ref{proposition:doob:transform}}
The Radon--Nikodym derivative process $L_t$, as defined in~\eqref{eq:time:dependent:likelihood}, admits the representation
\[
L_t=\bbE^{\bbP}\bigg(\frac{\d\bbQ}{\d\bbP}\given[\Big]\calF_t\bigg)=\frac{h^*(t,\mX_t^{\leftarrow})}{Z}\,,\qquad 0\le t\le T\,,
\]
where $Z=\bbE^{\bbP}\big\{w(\mX_T^{\leftarrow})\big\}$ and Doob's $h$-function is defined by
\begin{equation*}
h^*(t,\vx)\coloneq\bbE^{\bbP}\big\{w(\mX_T^{\leftarrow})\given\mX_t^{\leftarrow}=\vx\big\}\,.
\end{equation*}
The processes $h^{*}(t,\mX_{t}^{\leftarrow})$ and $L_t$ are uniformly integrable $\bbP$-martingales on $[0,T]$. For $t<T$, $h^{*}$ admits a positive $C^{1,2}$ version, and
\[
\d h^{*}(t,\mX_{t}^{\leftarrow})=\sqrt{2}\nabla h^{*}(t,\mX_{t}^{\leftarrow})^{\top}\d\mB_{t}\,,\qquad
\d L_{t}=\sqrt{2}L_{t}\nabla\log h^*(t,\mX_t^{\leftarrow})^{\top}\d\mB_t\,.
\]
\end{customproposition}

\begin{proof}[Proof of Proposition~\ref{proposition:doob:transform}]
Fix $t\in[0,T]$. By~\eqref{eq:time:dependent:likelihood},
\[
L_t=\left.\frac{\d\bbQ}{\d\bbP}\right|_{\calF_t}=\frac{\d(\bbQ|_{\calF_{t}})}{\d(\bbP|_{\calF_{t}})}\,.
\]
For every $A\in\calF_t$, the defining property of conditional expectation gives
\[
\int_{A} L_{t}\d\bbP=\bbQ(A)=\int_A\frac{\d\bbQ}{\d\bbP}\d\bbP=\int_A\bbE^{\bbP}\bigg(\frac{\d\bbQ}{\d\bbP}\given[\Big]\calF_t\bigg)\d\bbP\,.
\]
The conditional expectation is $\calF_t$-measurable, so uniqueness of the Radon--Nikodym derivative identifies it with $L_t$. By the Markov property of the forward OU process~\citep[Theorem 7.1.2]{Oksendal2003Stochastic}, conditioning $w(\mX_{0})$ on $(\mX_{s})_{T-t\le s\le T}$ is equivalent to conditioning on $\mX_{T-t}$. Using~\eqref{eq:path:Q}, we obtain
\begin{align*}
L_t&=\bbE^{\bbP}\bigg(\frac{\d\bbQ}{\d\bbP}\given[\Big]\calF_t\bigg)=\bbE^{\bbP}(L_{T}\given\calF_{t})\\
&=\frac{1}{Z}\bbE^{\bbP}\big\{w(\mX_T^{\leftarrow})\given\calF_t\big\}=\frac{1}{Z}\bbE^{\bbP}\big\{w(\mX_T^{\leftarrow})\given\mX_t^{\leftarrow}\big\}=\frac{h^{*}(t,\mX_{t}^{\leftarrow})}{Z}\,.
\end{align*}
The Markov identity extends to the usual augmentation by reverse martingale convergence~\citep[Theorem 4.7.3]{Durrett2019Probability} and the continuity established below for $t<T$; at $t=T$ it is immediate. Since $\bbE^{\bbP}\big\{w(\mX_T^{\leftarrow})\big\}=Z<\infty$, these conditional-expectation representations make both processes uniformly integrable martingales.

By Bayes' rule and~\eqref{eq:forward:solution}, for $t<T$,
\[
h^{*}(t,\vx)=\frac{\displaystyle\int w(\vu)\varphi_{d}(\vx;\mu_{T-t}\vu,\sigma_{T-t}^{2}\mI_{d})\,p_{0}(\d\vu)}{p_{T-t}(\vx)}\,,
\]
with $h^{*}(T,\vx)=w(\vx)$. Since $0<Z<\infty$, both the numerator and denominator are strictly positive. Gaussian kernel derivatives are uniformly bounded in $\vu$ on compact $(t,\vx)$-sets with $t<T$; differentiation under the integrals against the finite measures $p_{0}$ and $wp_{0}$ therefore gives $h^{*}\in C^{1,2}$. Applying It\^o's formula yields
\begin{align*}
\d h^*(t,\mX_t^{\leftarrow})
&=\partial_{t} h^{*}(t,\mX_{t}^{\leftarrow})\dt+\nabla h^{*}(t,\mX_{t}^{\leftarrow})^{\top}\d\mX_{t}^{\leftarrow}+\frac{1}{2}(\d\mX_{t}^{\leftarrow})^{\top}\nabla^{2} h^{*}(t,\mX_{t}^{\leftarrow})\d\mX_{t}^{\leftarrow}\\
&=\sqrt{2}\nabla h^{*}(t,\mX_{t}^{\leftarrow})^{\top}\d\mB_{t}\,,
\end{align*}
where the last equality follows from the martingale property. Finally, $L_{t}=h^{*}(t,\mX_t^{\leftarrow})/Z$ gives
\begin{equation}\label{eq:proposition:doob:transform:2}
\d L_t=\frac{1}{Z}\d h^*(t,\mX_t^{\leftarrow})=\frac{\sqrt{2}}Z\nabla h^*(t,\mX_t^{\leftarrow})^{\top}\d\mB_t=\sqrt{2}L_t\nabla\log h^*(t,\mX_t^{\leftarrow})^{\top}\d\mB_t\,.
\end{equation}
This completes the proof.
\end{proof}

\begin{customproposition}{\ref{proposition:controllable:sde:Q}}
Let the reference process $\mX_t^{\leftarrow}$ satisfy the SDE~\eqref{eq:base:reversal} under the path measure $\bbP$, and let $\bbQ$ be the target path measure defined by~\eqref{eq:path:Q}. Assume that the Novikov condition holds:
\begin{equation*}  
\bbE^{\bbP}\bigg[\exp\bigg\{\int_{0}^{T}|\nabla\log h^{*}(s,\mX_{s}^{\leftarrow})|_{2}^{2}\d s\bigg\}\bigg]<\infty\,.
\end{equation*}
Define a process $(\widetilde{\mB}_{t})_{0\leq t<T}$ by
\begin{equation*}
\d\widetilde{\mB}_{t}=\d\mB_{t}-\sqrt{2}\nabla\log h^{*}(t,\mX_{t}^{\leftarrow})\dt\,,\qquad \widetilde{\mB}_{0}=\bzero\,,
\end{equation*}
where $\mB_t$ is a Brownian motion under $\bbP$ and $\widetilde{\mB}_{T}$ is defined by continuous extension. Then $\widetilde{\mB}$ is a standard Brownian motion under $\bbQ$ on $[0,T]$. Further, under the path measure $\bbQ$, the reference process $\mX_t^{\leftarrow}$ in~\eqref{eq:base:reversal} evolves, for $0\le t<T$, according to:
\begin{equation*}
\d\mX_{t}^{\leftarrow}=\big\{\mX_{t}^{\leftarrow}+2\nabla\log p_{T-t}(\mX_{t}^{\leftarrow})+2\nabla\log h^{*}(t,\mX_{t}^{\leftarrow})\big\}\dt+\sqrt{2}\d\widetilde{\mB}_{t}\,.
\end{equation*}
\end{customproposition}

\begin{proof}[Proof of Proposition~\ref{proposition:controllable:sde:Q}]
Fix $0<\delta<T$. By Proposition~\ref{proposition:doob:transform}, $L$ is a positive density martingale on $[0,T-\delta]$ and
\[
\d L_{t}=\sqrt{2}L_{t}\sum_{j=1}^{d} D_{j}\log h^{*}(t,\mX_{t}^{\leftarrow})\d B_{t,j}\,.
\]
For $i=1,\ldots,d$,~\eqref{eq:brownian:P:Q} gives
\[
\widetilde {B}_{t,i}-B_{t,i}=-\sqrt{2}\int_0^t D_{i}\log h^{*}(s,\mX_{s}^{\leftarrow})\ds\,.
\]
This time integral has finite variation and hence zero quadratic covariation with the continuous process $L$. Thus $[L,\widetilde {B}_{i}-B_{i}]=0$. By bilinearity and $\d[B_{j},B_{i}]_{t}=\mathbf{1}_{\{i=j\}}\dt$,

\[
\dd[L,\widetilde {B}_{i}]_{t}=\d[L,B_{i}]_{t}=\sqrt{2}L_{t}\sum_{j=1}^{d} D_{j}\log h^{*}(t,\mX_{t}^{\leftarrow})\d[B_{j},B_{i}]_{t}=\sqrt{2}L_{t}D_{i}\log h^{*}(t,\mX_{t}^{\leftarrow})\dt\,.
\]
The product rule cancels the drift in \eqref{eq:brownian:P:Q}, giving
\[
\d(L_{t}\widetilde {B}_{t,i})=\widetilde {B}_{t,i}\d L_{t}+L_{t}\d\widetilde {B}_{t,i}+\d[L,\widetilde {B}_{i}]_{t}=\widetilde {B}_{t,i}\d L_{t}+L_{t}\d B_{t,i}\,.
\]
Substituting $\d L_{t}$ and using $\widetilde {B}_{0,i}=0$, we obtain
\[
L_{t}\widetilde {B}_{t,i}=\sqrt{2}\int_{0}^{t} L_{s}\widetilde {B}_{s,i}\nabla\log h^*(s,\mX_s^{\leftarrow})^{\top}\d\mB_s+\int_{0}^{t} L_{s}\d B_{s,i}\,.
\]
The integrands are adapted and continuous on $[0,T-\delta]$, hence locally square-integrable. Both It\^o integrals are therefore local $\bbP$-martingales~\citep[Theorem 6.2.1]{Chafai2025Stochastic}. The density-process change-of-measure rule~\citep[Lecture 25, first exercise]{Baudoin2012Girsanov} then makes $\widetilde {B}_{i}$ a local $\bbQ$-martingale. The same finite-variation argument gives

\[
[\widetilde {B}_{i},\widetilde {B}_{j}]_{t}=[B_{i},B_{j}]_{t}=\mathbf{1}_{\{i=j\}}t\,,\qquad \widetilde{\mB}_{0}=\bzero\,,
\]
L\'evy's characterization~\citep[Theorem 7.2.1]{Chafai2025Stochastic} shows that $\widetilde{\mB}$ is a $\bbQ$-Brownian motion on $[0,T-\delta]$.

Under~\eqref{eq:Novikov}, the squared-gradient integral is finite $\bbP$-almost surely and hence $\bbQ$-almost surely. Consequently,
\[
\int_{0}^{T}|\nabla\log h^{*}(s,\mX_{s}^{\leftarrow})|_{2}\ds
\le\sqrt{T}\bigg\{\int_{0}^{T}|\nabla\log h^*(s,\mX_s^{\leftarrow})|_2^2\ds\bigg\}^{1/2}<\infty\,,\qquad\bbQ\text{-a.s.}
\]
The drift integral in~\eqref{eq:brownian:P:Q} thus extends continuously to $T$. Hence $\widetilde{\mB}$ extends to $T$, with its Brownian increment laws preserved by continuity.

Since $\bbQ\ll\bbP$, the identity~\eqref{eq:base:reversal} also holds $\bbQ$-almost surely. Substituting~\eqref{eq:brownian:P:Q}, for $t<T$ we obtain
\begin{align*}
\d\mX_t^{\leftarrow}
&=\{\mX_t^{\leftarrow}+2\nabla\log p_{T-t}(\mX_t^{\leftarrow})\}\dt+\sqrt{2}\{\d\widetilde{\mB}_t+\sqrt{2}\nabla\log h^*(t,\mX_t^{\leftarrow})\dt\} \\
&=\{\mX_t^{\leftarrow}+2\nabla\log p_{T-t}(\mX_t^{\leftarrow})+2\nabla\log h^*(t,\mX_t^{\leftarrow})\}\dt+\sqrt{2}\d\widetilde{\mB}_t\,.
\end{align*}
This proves the stated dynamics and, without~\eqref{eq:Novikov}, their validity on every $[0,T-\delta]$.
\end{proof}

\section{Derivations in Section~\ref{section:doob:matching}}\label{appendix:doob:matching}

\begin{customproposition}{\ref{proposition:regression:constant}}
If $w(\mX_{0})\in L^{2}(\mathbb{P})$, then for every $t\in[0,T)$ and $h_{t}\in L^{2}(p_{T-t})$, the Doob's $h$-function $h_{t}^{*}$ in~\eqref{eq:h:function} minimizes the implicit Doob's matching loss~\eqref{eq:h:matching:implicit}. Further, 
\begin{equation*}
\calJ_{t}(h_{t})=\bbE^{\bbP}\big\{|h_{t}(\mX_{t}^{\leftarrow})-h_{t}^{*}(\mX_{t}^{\leftarrow})|^{2}\big\}+V_{t}^{2}\,,
\end{equation*}
where $V_{t}^{2}\coloneq\bbE^{\bbP}[\var\{w(\mX_{T}^{\leftarrow})\given\mX_{t}^{\leftarrow}\}]$ is a constant independent of $h_{t}$.
\end{customproposition}

\begin{proof}[Proof of Proposition~\ref{proposition:regression:constant}]
By~\eqref{eq:h:function},
\[
\bbE^{\bbP}\big\{h_{t}^{*}(\mX_{t}^{\leftarrow})-w(\mX_{T}^{\leftarrow})\given\mX_{t}^{\leftarrow}\big\}=0\,.
\]
Thus the cross term in the squared-loss expansion has zero expectation, and
\begin{align*}
\calJ_{t}(h_{t})
&=\bbE^{\bbP}\big\{|h_{t}(\mX_{t}^{\leftarrow})-h_{t}^{*}(\mX_{t}^{\leftarrow})|^{2}\big\}+\bbE^{\bbP}\big\{|h_{t}^{*}(\mX_{t}^{\leftarrow})-w(\mX_{T}^{\leftarrow})|^{2}\big\} \\
&=\bbE^{\bbP}\big\{|h_{t}(\mX_{t}^{\leftarrow})-h_{t}^{*}(\mX_{t}^{\leftarrow})|^{2}\big\}+V_{t}^{2}\,.
\end{align*}
The last equality follows from the conditional variance formula and the law of total expectation. Since $V_{t}^{2}$ is independent of $h_{t}$, the minimum is attained precisely when $h_{t}=h_{t}^{*}$ $p_{T-t}$-almost everywhere.
\end{proof}

\begin{customproposition}{\ref{proposition:population:existence}}
Let $0\leq t<T$, $\lambda>0$, and $w(\mX_{0})\in L^{2}(\mathbb{P})$. For $h\in H^{1}(p_{T-t})$, define
\begin{equation*}
\calJ_{t}^{\lambda}(h)\coloneq\calJ_{t}(h)+\lambda\|\nabla h\|_{L^{2}(p_{T-t})}^{2}=\calJ_{t}(h)+\lambda\bbE\bigl\{|\nabla h(\mu_{T-t}\mX_{0}+\sigma_{T-t}\vepsilon)|_{2}^{2}\bigr\}\,,
\end{equation*}
where $\mX_{0}\sim p_{0}$ and $\vepsilon\sim\calN(\bzero,\mI_{d})$ are independent. Then $\calJ_{t}^{\lambda}$ admits a unique global minimizer in $H^{1}(p_{T-t})$, denoted by
\[
h_{t}^{\lambda}=\argmin_{h\in H^{1}(p_{T-t})}\calJ_{t}^{\lambda}(h)\,,
\]
where uniqueness is understood up to $p_{T-t}$-almost-everywhere equality.
\end{customproposition}

\begin{proof}[Proof of Proposition~\ref{proposition:population:existence}]
By Proposition~\ref{proposition:regression:constant},
\begin{equation}\label{eq:population:decomposition}
\calJ_{t}^{\lambda}(h)=V_{t}^{2}+\|h-h_{t}^{*}\|_{L^{2}(p_{T-t})}^{2}+\lambda\|\nabla h\|_{L^{2}(p_{T-t})}^{2}\,.
\end{equation}
Here $V_{t}^{2}$ is independent of $h$, and conditional Jensen's inequality gives $\|h_{t}^{*}\|_{L^{2}(p_{T-t})}^{2}\leq\bbE\{w(\mX_{0})^{2}\}<\infty$.

Because $T-t>0$, Gaussian convolution makes $p_{T-t}$ smooth and strictly positive. On each bounded ball, $p_{T-t}$ is bounded above and away from zero. A Cauchy sequence in $H^{1}(p_{T-t})$ therefore has weighted $L^{2}$ limits for its functions and gradients, and converges in ordinary $H^{1}$ on each such ball. Local completeness identifies the gradient limit as the weak gradient of the function limit. Thus $H^{1}(p_{T-t})$ is complete. The unweighted result used on each ball is~\citet[Proposition 9.1 and Remark 4(i), p.~264]{Brezis2011Functional}; the local bounds on $p_{T-t}$ provide its application here.

Define
\begin{equation}\label{eq:population:form}
a_{\lambda}(u,v)=\langle u,v\rangle_{L^{2}(p_{T-t})}+\lambda\langle\nabla u,\nabla v\rangle_{L^{2}(p_{T-t})}\,,\qquad \ell(v)=\langle h_{t}^{*},v\rangle_{L^{2}(p_{T-t})}\,.
\end{equation}
The form $a_{\lambda}$ is an inner product, with
\begin{gather}\label{eq:population:bounds}
\min\{1,\lambda\}\|v\|_{H^{1}(p_{T-t})}^{2}\leq a_{\lambda}(v,v)\leq\max\{1,\lambda\}\|v\|_{H^{1}(p_{T-t})}^{2}\,,\notag\\
|\ell(v)|\leq\|h_{t}^{*}\|_{L^{2}(p_{T-t})}\,a_{\lambda}(v,v)^{1/2}\,.
\end{gather}
Hence $H^{1}(p_{T-t})$ is also a Hilbert space under $a_{\lambda}$, and $\ell$ is continuous in this inner product. The Riesz--Fr\'echet representation theorem~\citep[Theorem 5.5, p.~135]{Brezis2011Functional} gives a unique $u\in H^{1}(p_{T-t})$ satisfying $a_{\lambda}(u,v)=\ell(v)$ for every $v\in H^{1}(p_{T-t})$; explicitly,
\begin{equation}\label{eq:population:weak}
\langle u-h_{t}^{*},v\rangle_{L^{2}(p_{T-t})}+\lambda\langle\nabla u,\nabla v\rangle_{L^{2}(p_{T-t})}=0\,,\qquad v\in H^{1}(p_{T-t})\,.
\end{equation}
To relate this equation to the optimization problem, expand the objective along any direction $v\in H^{1}(p_{T-t})$. For $h\in H^{1}(p_{T-t})$ and $\eta\in\bbR$,
\begin{align}
\calJ_{t}^{\lambda}(h+\eta v)-\calJ_{t}^{\lambda}(h)
&=2\eta\{\langle h-h_{t}^{*},v\rangle_{L^{2}(p_{T-t})}+\lambda\langle\nabla h,\nabla v\rangle_{L^{2}(p_{T-t})}\}\notag\\
&\quad+\eta^{2}\{\|v\|_{L^{2}(p_{T-t})}^{2}+\lambda\|\nabla v\|_{L^{2}(p_{T-t})}^{2}\}\,.
\label{eq:population:variation}
\end{align}
Consequently,
\[
\left.\frac{d}{d\eta}\calJ_{t}^{\lambda}(h+\eta v)\right|_{\eta=0}=2\{\langle h-h_{t}^{*},v\rangle_{L^{2}(p_{T-t})}+\lambda\langle\nabla h,\nabla v\rangle_{L^{2}(p_{T-t})}\}\,.
\]
Taking $h=u$ and using~\eqref{eq:population:weak}, the linear term vanishes. For any $h\in H^{1}(p_{T-t})$, set $v=h-u$ and $\eta=1$ in~\eqref{eq:population:variation} to obtain
\begin{equation}\label{eq:population:excess}
\calJ_{t}^{\lambda}(h)-\calJ_{t}^{\lambda}(u)=\|h-u\|_{L^{2}(p_{T-t})}^{2}+\lambda\|\nabla(h-u)\|_{L^{2}(p_{T-t})}^{2}\,.
\end{equation}
This is nonnegative and vanishes only when $h=u$ in $H^{1}(p_{T-t})$. Thus $u$ is the unique global minimizer $h_{t}^{\lambda}$.
\end{proof}

\begin{lemma}\label{lemma:doob:h:function:green}
Suppose Assumptions~\ref{assumption:bounded:support} and~\ref{assumption:bounded:weight} hold. Let $0\leq t<T$ and $v_t\in H^1(p_{T-t})$. Then
\begin{equation*}
-(\nabla h_{t}^{*},\nabla v_{t})_{L^{2}(p_{T-t})}=(\Delta h_{t}^{*}+\nabla h_{t}^{*}\cdot\nabla\log p_{T-t},v_{t})_{L^{2}(p_{T-t})}\,.
\end{equation*}
\end{lemma}

\begin{proof}[Proof of Lemma~\ref{lemma:doob:h:function:green}]
All integrals below are over $\bbR^{d}$. Choose $\psi_k\in C_{c}^\infty(\bbR^d)$ such that
\[
0\le\psi_k\le1\,,\qquad
\psi_k=1\ \text{on }B(0,k)\,,\qquad
\psi_k=0\ \text{outside }B(0,2k)\,,\qquad
\|\nabla\psi_k\|_{\infty}\le C/k\,.
\]
See~\citet[p.~267, footnote 5]{Brezis2011Functional} for this construction.

Since $p_{T-t}$ is smooth and strictly positive, $\psi_kv_t$ is a compactly supported $H^1(\bbR^{d})$ function. Weak integration by parts gives
\begin{align}
-\int \nabla h_{t}^{*}\cdot\nabla(\psi_kv_t)\,p_{T-t}\,dx
&=\int \psi_{k}v_{t}\,\nabla\cdot(p_{T-t}\nabla h_{t}^{*})\,dx\notag\\
&=\int \psi_{k}v_{t}\bigl(\Delta h_{t}^{*}+\nabla\log p_{T-t}\cdot\nabla h_{t}^{*}\bigr)p_{T-t}\,dx\,,
\label{eq:lemma:doob:h:function:green:1}
\end{align}
where we used the product rule and $\nabla p_{T-t}=p_{T-t}\nabla\log p_{T-t}$.

Expanding $\nabla(\psi_{k}v_{t})=\psi_{k}\nabla v_{t}+v_{t}\nabla\psi_{k}$ in \eqref{eq:lemma:doob:h:function:green:1}, we obtain
\begin{align}
-\int \psi_k\nabla h_t^*\cdot\nabla v_t\,p_{T-t}\,dx
&=\int \psi_{k}v_{t}\bigl(\Delta h_{t}^{*}+\nabla\log p_{T-t}\cdot\nabla h_t^*\bigr)p_{T-t}\,dx\notag\\
&\quad+\int v_{t}\nabla h_t^*\cdot\nabla\psi_k\,p_{T-t}\,dx\,.
\label{eq:lemma:doob:h:function:green:2}
\end{align}
The last term vanishes as $k\to\infty$, since Cauchy--Schwarz gives
\[
\left|
\int v_{t}\,\nabla h_t^*\cdot\nabla\psi_k\,p_{T-t}\,dx
\right|
\le\frac{C}{k}
\|v_t\|_{L^{2}(p_{T-t})}
\|\nabla h_t^*\|_{L^2(p_{T-t})}
\longrightarrow0\,.
\]

Lemmas~\ref{lemma:section:error:guidance:regularity:2} and~\ref{lemma:reg:gap:bound} ensure that $\nabla h_t^{*}$ and $\Delta h_{t}^{*}+\nabla\log p_{T-t}\cdot\nabla h_{t}^{*}$ belong to $L^2(p_{T-t})$, respectively. Together with $v_t\in H^1(p_{T-t})$, these bounds provide integrable majorants for the two remaining integrands in \eqref{eq:lemma:doob:h:function:green:2}. Since $0\le\psi_{k}\le1$ and $\psi_k\to1$, dominated convergence yields
\[
-\int \nabla h_{t}^{*}\cdot\nabla v_{t}\,p_{T-t}\,dx
=\int v_{t}\bigl(\Delta h_{t}^{*}+\nabla\log p_{T-t}\cdot\nabla h_{t}^{*}\bigr)p_{T-t}\,dx\,.
\]
This proves the lemma.
\end{proof}

\begin{customproposition}{\ref{proposition:regularization:gap}}
Suppose Assumptions~\ref{assumption:bounded:support} and~\ref{assumption:bounded:weight} hold, and let $0\leq t<T$. Let $\lambda>0$, $h_{t}^{*}$ be the Doob's $h$-function defined as~\eqref{eq:h:function}, and $h_{t}^{\lambda}$ be the minimizer of $\calJ_{t}^{\lambda}$ defined as~\eqref{eq:h:matching:GP}. Then $h_{t}^{*}\in H^{2}(p_{T-t})$, and 
\begin{align*}
\|h_{t}^{\lambda}-h_{t}^{*}\|_{L^{2}(p_{T-t})}^{2} 
&\leq \lambda^{2}\| \Delta h_{t}^{*}+\nabla h_{t}^{*}\cdot\nabla\log p_{T-t}\|_{L^{2}(p_{T-t})}^{2}\,, \\
\|\nabla h_{t}^{\lambda}-\nabla h_{t}^{*}\|_{L^{2}(p_{T-t})}^{2}&\leq \lambda\|\Delta h_{t}^{*}+\nabla h_{t}^{*}\cdot\nabla\log p_{T-t}\|_{L^{2}(p_{T-t})}^{2}\,.
\end{align*}
\end{customproposition}

\begin{proof}[Proof of Proposition~\ref{proposition:regularization:gap}]
First, $h_t^*\in H^2(p_{T-t})$ follows directly from Lemmas~\ref{lemma:section:error:guidance:bound},~\ref{lemma:section:error:guidance:regularity:2}, and~\ref{lemma:section:error:guidance:regularity:3}. By Proposition~\ref{proposition:regression:constant} and \eqref{eq:h:matching:GP},
\[
\calJ_t^\lambda(h_t)=\|h_t-h_t^*\|_{L^2(p_{T-t})}^2+V_t^2+\lambda\|\nabla h_t\|_{L^2(p_{T-t})}^2\,.
\]
Since $h_t^\lambda$ minimizes $\calJ_t^\lambda$ over $H^{1}(p_{T-t})$, for any $v_t\in H^1(p_{T-t})$,
\[
0=\left.\frac{1}{2}\frac{d}{d\eta}\calJ_t^\lambda(h_t^\lambda+\eta v_t)\right|_{\eta=0}=\langle h_t^\lambda-h_t^*,v_t\rangle_{L^2(p_{T-t})}+\lambda\langle\nabla h_t^\lambda,\nabla v_t\rangle_{L^2(p_{T-t})}\,.
\]

Taking $v_{t}=h_t^\lambda-h_t^*$ and writing $\nabla h_t^{\lambda}=(\nabla h_t^\lambda-\nabla h_t^*)+\nabla h_t^*$, we obtain
\begin{align}
&\|h_t^\lambda-h_t^*\|_{L^2(p_{T-t})}^2+\lambda\|\nabla h_t^\lambda-\nabla h_t^*\|_{L^2(p_{T-t})}^2
=-\lambda\langle\nabla h_{t}^{*},\nabla(h_{t}^{\lambda}-h_{t}^{*})\rangle_{L^{2}(p_{T-t})}\notag\\
&\quad=\lambda\left\langle\Delta h_t^*+\nabla h_t^*\cdot\nabla\log p_{T-t},h_t^\lambda-h_t^*\right\rangle_{L^2(p_{T-t})}\notag\\
&\quad\le\lambda\|\Delta h_t^*+\nabla h_t^*\cdot\nabla\log p_{T-t}\|_{L^2(p_{T-t})}\|h_t^\lambda-h_t^*\|_{L^2(p_{T-t})}\,,
\label{eq:proposition:regularization:gap:1}
\end{align}
where the second equality follows from Lemma~\ref{lemma:doob:h:function:green} and the inequality follows from Cauchy--Schwarz. The norm on the right-hand side is finite by Lemma~\ref{lemma:reg:gap:bound}.

By \eqref{eq:proposition:regularization:gap:1} and the nonnegativity of the gradient term,
\[
\|h_t^\lambda-h_{t}^{*}\|_{L^{2}(p_{T-t})}\le\lambda\|\Delta h_{t}^{*}+\nabla h_{t}^{*}\cdot\nabla\log p_{T-t}\|_{L^{2}(p_{T-t})}\,.
\]
Squaring this inequality proves the first bound. Substituting it into \eqref{eq:proposition:regularization:gap:1} yields
\[
\|h_{t}^{\lambda}-h_{t}^{*}\|_{L^{2}(p_{T-t})}^{2}+\lambda\|\nabla h_{t}^{\lambda}-\nabla h_t^*\|_{L^2(p_{T-t})}^2\le\lambda^2\|\Delta h_t^*+\nabla h_t^*\cdot\nabla\log p_{T-t}\|_{L^2(p_{T-t})}^2\,.
\]
Dropping the first term and dividing by $\lambda>0$ proves the second bound. This completes the proof.
\end{proof}

\begin{customproposition}{\ref{proposition:stability}}
Let $0\leq t<T$, $w(\mX_{0})\in L^{2}(\mathbb{P})$ and $\lambda>0$, and let $h_{t}^{\lambda}$ be the minimizer of $\calJ_{t}^{\lambda}$ defined as~\eqref{eq:h:matching:GP}. Then for any $h_{t}\in H^{1}(p_{T-t})$, we have
\begin{equation*}
\frac{1}{\max(\lambda,1)}\{\calJ_{t}^{\lambda}(h_{t})-\calJ_{t}^{\lambda}(h_{t}^{\lambda})\}
\leq \|h_{t}-h_{t}^{\lambda}\|_{H^{1}(p_{T-t})}^{2}
\leq \frac{1}{\min(\lambda,1)}\{\calJ_{t}^{\lambda}(h_{t})-\calJ_{t}^{\lambda}(h_{t}^{\lambda})\}\,.
\end{equation*}
\end{customproposition}

\begin{proof}[Proof of Proposition~\ref{proposition:stability}]
By Proposition~\ref{proposition:regression:constant} and \eqref{eq:h:matching:GP},
\[
\calJ_t^\lambda(h_t)=\|h_t-h_t^*\|_{L^2(p_{T-t})}^2+V_t^2+\lambda\|\nabla h_t\|_{L^2(p_{T-t})}^2\,.
\]
Since $h_t^\lambda$ minimizes $\calJ_t^\lambda$ over $H^{1}(p_{T-t})$, for any $v_{t}\in H^{1}(p_{T-t})$,
\begin{equation}\label{eq:proposition:stability:1}
0=\left.\frac{1}{2}\frac{d}{d\eta}\calJ_t^\lambda(h_t^\lambda+\eta v_t)\right|_{\eta=0}=\langle h_t^\lambda-h_t^*,v_t\rangle_{L^2(p_{T-t})}+\lambda\langle\nabla h_t^\lambda,\nabla v_t\rangle_{L^2(p_{T-t})}\,.
\end{equation}

For any $h_t\in H^1(p_{T-t})$, taking $v_t=h_t-h_t^\lambda$ in \eqref{eq:proposition:stability:1} gives
\[
\langle h_t^\lambda-h_t^*,h_{t}-h_{t}^{\lambda}\rangle_{L^2(p_{T-t})}+\lambda\langle\nabla h_t^\lambda,\nabla h_t-\nabla h_t^\lambda\rangle_{L^2(p_{T-t})}=0\,.
\]
Expanding the squared norms in $\calJ_{t}^{\lambda}(h_{t})-\calJ_t^\lambda(h_t^\lambda)$, the two cross terms cancel by this identity. Hence
\[
\calJ_{t}^{\lambda}(h_{t})-\calJ_{t}^{\lambda}(h_{t}^{\lambda})=\|h_t-h_t^\lambda\|_{L^2(p_{T-t})}^2+\lambda\|\nabla h_t-\nabla h_t^\lambda\|_{L^2(p_{T-t})}^2\,.
\]
Comparing the coefficients $1$ and $\lambda$, we obtain
\[
\min\{1,\lambda\}\|h_t-h_t^\lambda\|_{H^{1}(p_{T-t})}^2\le \calJ_t^\lambda(h_t)-\calJ_t^\lambda(h_t^\lambda)\le \max\{1,\lambda\}\|h_t-h_t^\lambda\|_{H^{1}(p_{T-t})}^2\,.
\]
Rearranging these inequalities proves the proposition.
\end{proof}

\section{Derivations in Section~\ref{section:convergence:assumption}}\label{appendix:convergence:assumption}

\begin{lemma}\label{lemma:section:error:guidance:bound}
Suppose Assumption~\ref{assumption:bounded:weight} holds. Then for all $t\in[0,T)$,
\begin{equation*}
\underline{B}\leq h_{t}^{*}(\vx)\leq\bar{B}\,, \quad \vx\in\bbR^{d}\,.
\end{equation*}
\end{lemma}

\begin{proof}[Proof of Lemma~\ref{lemma:section:error:guidance:bound}]
A direct conclusion of Assumption~\ref{assumption:bounded:weight} and the definition of Doob's $h$-function $h_{t}^{*}$~\eqref{eq:h:function}.
\end{proof}

\begin{lemma}[Tweedie's formula]
\label{lemma:tweedie}
Let $t\in(0,T]$, and let $\mX_{t}$ be defined as~\eqref{eq:forward:base}. Then 
\begin{equation*}
\nabla\log p_{t}(\vx)+\frac{\vx}{\sigma_{t}^{2}}=\frac{\mu_{t}}{\sigma_{t}^{2}}\bbE(\mX_{0}\given\mX_{t}=\vx)\,, \quad \vx\in\bbR^{d}\,.
\end{equation*}
\end{lemma}

\begin{proof}[Proof of Lemma~\ref{lemma:tweedie}]
It is straightforward that 
\begin{align*}
&\nabla\log p_{t}(\vx)=\frac{\nabla p_{t}(\vx)}{p_{t}(\vx)}=\frac{1}{p_{t}(\vx)}\int\nabla_{\vx}\varphi_{d}(\vx;\mu_{t}\vx_{0},\sigma_{t}^{2}\mI_{d})p_{0}(\mathrm{d}\vx_{0}) \\
&~~~~~~~=-\frac{1}{p_{t}(\vx)}\int\bigg(\frac{\vx-\mu_{t}\vx_{0}}{\sigma_{t}^{2}}\bigg)\varphi_{d}(\vx;\mu_{t}\vx_{0},\sigma_{t}^{2}\mI_{d})p_{0}(\mathrm{d}\vx_{0}) \\
&~~~~~~~=-\frac{1}{p_{t}(\vx)}\frac{\vx}{\sigma_{t}^{2}}\int\varphi_{d}(\vx;\mu_{t}\vx_{0},\sigma_{t}^{2}\mI_{d})p_{0}(\mathrm{d}\vx_{0}) +\frac{\mu_{t}}{\sigma_{t}^{2}}\int\vx_{0}\varphi_{d}(\vx;\mu_{t}\vx_{0},\sigma_{t}^{2}\mI_{d})\frac{p_{0}(\mathrm{d}\vx_{0})}{p_{t}(\vx)} \\
&~~~~~~~ = -\frac{\vx}{\sigma_{t}^{2}}+\frac{\mu_{t}}{\sigma_{t}^{2}}\bbE(\mX_{0}\given\mX_{t}=\vx)\,,
\end{align*}
where the second equality is due to~\eqref{eq:forward:solution}, and last equality invokes the Bayes' rule. This completes the proof.
\end{proof}

\begin{lemma}\label{lemma:grad:condition:expectation}
Let $g:\bbR^d\to\bbR$ be measurable with $\int|g|\,\mathrm{d}p_{0}<\infty$. Let $t\in(0,T]$, and let $\mX_{t}$ be defined as~\eqref{eq:forward:base}. Then for each $\vx\in\bbR^{d}$,
\begin{equation*}
\nabla_{\vx}\bbE\big\{g(\mX_{0})\given\mX_{t}=\vx\big\}=\frac{\mu_{t}}{\sigma_{t}^{2}}\cov\{\mX_{0},g(\mX_{0})\given\mX_{t}=\vx\}\,,
\end{equation*}
where the $k$-th entry of $\cov\{\mX_{0},g(\mX_{0})\given\mX_{t}=\vx\}$ is defined as $\cov\{X_{0,k},g(\mX_{0})\given\mX_{t}=\vx\}$ with $\mX_{0}=(X_{0,1},\ldots,X_{0,d})$.
\end{lemma}

\begin{proof}[Proof of Lemma~\ref{lemma:grad:condition:expectation}]
According to Bayes' rule, we have
\begin{equation*}
\bbE\big\{g(\mX_{0})\given\mX_{t}=\vx\big\}=\frac{1}{p_{t}(\vx)}\int g(\vx_{0})\varphi_{d}(\vx;\mu_{t}\vx_{0},\sigma_{t}^{2}\mI_{d})p_{0}(\mathrm{d}\vx_{0})\,.
\end{equation*}
On compact sets of $\vx$, the Gaussian kernel and its first derivatives are uniformly bounded as functions of $\vx_{0}$. Dominated convergence with $|g|\,\mathrm{d}p_{0}$ therefore justifies differentiation. Taking the gradient with respect to $\vx$ on both sides yields
\begin{align*}
&\nabla_{\vx}\bbE\big\{g(\mX_{0})\given\mX_{t}=\vx\big\}
=\frac{1}{p_{t}(\vx)}\int g(\vx_{0})\nabla_{\vx}\varphi_{d}(\vx;\mu_{t}\vx_{0},\sigma_{t}^{2}\mI_{d})p_{0}(\mathrm{d}\vx_{0}) \\
&~~~~~~~~~~~~~~~~~~~~~~~~~~~~~~~~~~~~~~~~~~-\frac{\nabla p_{t}(\vx)}{p_{t}^{2}(\vx)}\int g(\vx_{0})\varphi_{d}(\vx;\mu_{t}\vx_{0},\sigma_{t}^{2}\mI_{d})p_{0}(\mathrm{d}\vx_{0})  \\
&~~~~~~~~=-\frac{1}{p_{t}(\vx)}\int g(\vx_{0})\bigg(\frac{\vx-\mu_{t}\vx_{0}}{\sigma_{t}^{2}}\bigg)\varphi_{d}(\vx;\mu_{t}\vx_{0},\sigma_{t}^{2}\mI_{d})p_{0}(\mathrm{d}\vx_{0}) \\
&~~~~~~~~~~~~~~~~~~~~~~~~~~-\frac{\nabla\log p_{t}(\vx)}{p_{t}(\vx)}\int g(\vx_{0})\varphi_{d}(\vx;\mu_{t}\vx_{0},\sigma_{t}^{2}\mI_{d})p_{0}(\mathrm{d}\vx_{0}) \\
&~~~~~~~~=-\bigg\{\nabla\log p_{t}(\vx)+\frac{\vx}{\sigma_{t}^{2}}\bigg\}\int g(\vx_{0})\varphi_{d}(\vx;\mu_{t}\vx_{0},\sigma_{t}^{2}\mI_{d})\frac{p_{0}(\mathrm{d}\vx_{0})}{p_{t}(\vx)}  \\
&~~~~~~~~~~~~~~~~+\frac{\mu_{t}}{\sigma_{t}^{2}}\int \vx_{0}g(\vx_{0})\varphi_{d}(\vx;\mu_{t}\vx_{0},\sigma_{t}^{2}\mI_{d})\frac{p_{0}(\mathrm{d}\vx_{0})}{p_{t}(\vx)}  \\
&~~~~~~~~=-\frac{\mu_t}{\sigma_t^2}\bbE(\mX_0\given\mX_t=\vx)\bbE\{g(\mX_0)\given\mX_t=\vx\}{+\frac{\mu_{t}}{\sigma_{t}^{2}}}\bbE\{\mX_0g(\mX_0)\given\mX_t=\vx\}\\
&~~~~~~~~=\frac{\mu_{t}}{\sigma_{t}^{2}}\cov\{\mX_{0},g(\mX_{0})\given\mX_{t}=\vx\}\,,
\end{align*}
where the fourth equality follows from Lemma~\ref{lemma:tweedie} and the Bayes' rule. Here the $k$-th entry of the conditional covariance $\cov\{\mX_{0},g(\mX_{0})\given\mX_{t}=\vx\}$ is defined as $\cov\{X_{0,k},g(\mX_{0})\given\mX_{t}=\vx\}$, where $\mX_{0}=(X_{0,1},\ldots,X_{0,d})$. This completes the proof.
\end{proof}

\begin{lemma}\label{lemma:section:error:guidance:regularity:2}
Suppose Assumptions~\ref{assumption:bounded:support} and~\ref{assumption:bounded:weight} hold. Then for each $\vx\in\bbR^{d}$ and $0\leq t<T$,
\begin{equation*}
\max_{1\leq k\leq d}|D_{k}h_{t}^{*}(\vx)|\leq\frac{2\bar{B}}{\sigma_{T-t}^{2}}\,,
\end{equation*}
where $D_{k}$ denote the differential operator with respect to the $k$-th entry of $\vx$.
\end{lemma}

\begin{proof}[Proof of Lemma~\ref{lemma:section:error:guidance:regularity:2}]
According to the definition of Doob's $h$-function $h_{t}^{*}$~\eqref{eq:h:function} and the property of the time-reversal process~\eqref{eq:base:reversal}, we have 
\begin{equation*}
h^{*}(t,\vx)\coloneq\bbE^{\bbP}\big\{w(\mX_{T}^{\leftarrow})\given\mX_{t}^{\leftarrow}=\vx\big\}=\bbE\big\{w(\mX_{0})\given\mX_{T-t}=\vx\big\}\,,
\end{equation*}
where the second expectation is with respect to the path measure of the forward process~\eqref{eq:forward:base}. Then it follows from Lemma~\ref{lemma:grad:condition:expectation} that 
\begin{align}\label{eq:lemma:section:error:guidance:regularity:2:2}
\nabla h_{t}^{*}(\vx)=\frac{\mu_{T-t}}{\sigma_{T-t}^{2}}\cov\{\mX_{0},w(\mX_{0})\given\mX_{T-t}=\vx\}\,.
\end{align}
Then it follows from Assumptions~\ref{assumption:bounded:support} and~\ref{assumption:bounded:weight} that for each $\vx\in\bbR^{d}$,
\begin{equation}\label{eq:lemma:section:error:guidance:regularity:2:3}
|\cov\{\mX_{0},w(\mX_{0})\given\mX_{T-t}=\vx\}|_{\infty}=\max_{1\leq k\leq d}|\cov\{X_{0,k},w(\mX_{0})\given\mX_{T-t}=\vx\}|\leq 2\bar{B}\,.
\end{equation}
Substituting~\eqref{eq:lemma:section:error:guidance:regularity:2:3} into~\eqref{eq:lemma:section:error:guidance:regularity:2:2} yields 
\begin{equation*}
|\nabla h_{t}^{*}(\vx)|_{\infty}=\frac{\mu_{T-t}}{\sigma_{T-t}^{2}}|\cov\{\mX_{0},w(\mX_{0})\given\mX_{T-t}=\vx\}|_{\infty}\leq\frac{2\bar{B}}{\sigma_{T-t}^{2}}\,,
\end{equation*}
where we used the fact that $\mu_{T-t}=\exp(t-T)\leq1$. This completes the proof.
\end{proof}

\begin{lemma}\label{lemma:section:error:guidance:regularity:3}
Suppose Assumptions~\ref{assumption:bounded:support} and~\ref{assumption:bounded:weight} hold. Then for each $\vx\in\bbR^{d}$ and $0\leq t<T$,
\begin{equation*}
|D_{k\ell}^{2}h_{t}^{*}(\vx)|\leq\frac{6\bar{B}}{\sigma_{T-t}^{4}}\,,
\end{equation*}
where $D_{k\ell}^{2}$ denote the second-order differential operator with respect to $k$-th and $\ell$-th entry.
\end{lemma}

\begin{proof}[Proof of Lemma~\ref{lemma:section:error:guidance:regularity:3}]
Taking derivative with respect to the $\ell$-th entry of $\vx$ on both sides of~\eqref{eq:lemma:section:error:guidance:regularity:2:2} implies
\begin{equation}\label{eq:lemma:section:error:guidance:regularity:3:1}
D_{k\ell}^{2}h_{t}^{*}(\vx)=\frac{\mu_{T-t}}{\sigma_{T-t}^{2}}D_{\ell}\cov\{X_{0,k},w(\mX_0)\given\mX_{T-t}=\vx\}\,.
\end{equation}
It remains to estimate the derivative of the conditional covariance. Indeed,
\begin{align*}
&D_{\ell}\cov\{X_{0,k},w(\mX_0)\given\mX_{T-t}=\vx\} \\
&~~~~~~=D_{\ell}\bbE\big\{X_{0,k}w(\mX_{0})\given\mX_{T-t}=\vx\big\} -D_{\ell}\bbE(X_{0,k}\given\mX_{T-t}=\vx)\bbE\big\{w(\mX_{0})\given\mX_{T-t}=\vx\big\} \\
&~~~~~~~~~~~~~~~~~~-\bbE(X_{0,k}\given\mX_{T-t}=\vx)D_{\ell}\bbE\big\{w(\mX_{0})\given\mX_{T-t}=\vx\big\} \\
&~~~~~~=\frac{\mu_{T-t}}{\sigma_{T-t}^{2}}\cov\{X_{0,k}w(\mX_0),X_{0,\ell}\given\mX_{T-t}=\vx\} \\
&~~~~~~~~~~~~-\frac{\mu_{T-t}}{\sigma_{T-t}^{2}}\cov(X_{0,\ell},X_{0,k}\given\mX_{T-t}=\vx)\bbE\big\{w(\mX_{0})\given\mX_{T-t}=\vx\big\} \\
&~~~~~~~~~~~~-\bbE(X_{0,k}\given\mX_{T-t}=\vx)\frac{\mu_{T-t}}{\sigma_{T-t}^{2}}\cov\{w(\mX_0),X_{0,\ell}\given\mX_{T-t}=\vx\}\,,
\end{align*}
where the last equality holds from Lemma~\ref{lemma:grad:condition:expectation}. Consequently, for each $\vx\in\bbR^{d}$ and $0\leq t<T$,
\begin{equation}\label{eq:lemma:section:error:guidance:regularity:3:2}
|D_{\ell}\cov\{X_{0,k},w(\mX_0)\given\mX_{T-t}=\vx\}|\leq\frac{6\mu_{T-t}\bar{B}}{\sigma_{T-t}^{2}}\,,
\end{equation}
where we used Assumptions~\ref{assumption:bounded:support} and~\ref{assumption:bounded:weight}. Substituting~\eqref{eq:lemma:section:error:guidance:regularity:3:2} into~\eqref{eq:lemma:section:error:guidance:regularity:3:1} completes the proof.
\end{proof}

\begin{customproposition}{\ref{proposition:regularity:h:function}}
Suppose Assumptions~\ref{assumption:bounded:support} and~\ref{assumption:bounded:weight} hold. Then for all $t\in[0,T)$ and $\vx\in\bbR^{d}$, the following bounds hold:
\[
\underline B\leq h_t^*(\vx)\leq\bar B\,,\quad
\max_{1\leq k\leq d}|D_kh_t^*(\vx)|\leq2\sigma_{T-t}^{-2}\bar B\,,\quad
\max_{1\leq k,\ell\leq d}|D_{k\ell}^2h_t^*(\vx)|\leq6\sigma_{T-t}^{-4}\bar B\,,
\]
where $D_{k}$ and $D_{k\ell}^{2}$ denote the first-order and second-order partial derivatives with respect to the input coordinates, respectively.
\end{customproposition}

\begin{proof}[Proof of Proposition~\ref{proposition:regularity:h:function}]
A direct conclusion of Lemmas~\ref{lemma:section:error:guidance:bound},~\ref{lemma:section:error:guidance:regularity:2}, and~\ref{lemma:section:error:guidance:regularity:3}. 
\end{proof}

\section{Derivations in Section~\ref{section:rate:guidance}}\label{appendix:rate:guidance}

\subsection{Oracle inequality of variationally stable Doob's matching}
\label{section:error:guidance:oracle}

\begin{customlemma}{\ref{lemma:oracle:inequality:guidance}}
Suppose Assumptions~\ref{assumption:bounded:support} and~\ref{assumption:bounded:weight} hold. Fix $t\in[0,T)$, let $n$ be sufficiently large, and let $\lambda>0$. Let $\scrH_t\subset C^{1}(\mathbb {R}^{d})$ be a nonempty class satisfying $\underline {B}\leq h_{t}\leq\bar {B}$ and $\max_{1\leq k\leq d}\|D_{k}h_{t}\|_{\infty}\leq2\bar {B}\sigma_{T-t}^{-2}$ for every $h_{t}\in\scrH_{t}$. Let $\hat h_t^\lambda$ be the gradient-regularized empirical risk minimizer defined as~\eqref{eq:erm:GP}, and let $h_t^*$ be the Doob's $h$-function defined as~\eqref{eq:h:function}. Then the following inequalities hold:
\begin{align*}
&\bbE\big\{\|\hat{h}_{t}^{\lambda}-h_{t}^{*}\|_{L^{2}(p_{T-t})}^{2}\big\}
\lesssim \inf_{h_{t}\in\scrH_{t}}\big\{\|h_{t}-h_{t}^{*}\|_{L^{2}(p_{T-t})}^{2}+\lambda\|\nabla h_{t}-\nabla h_{t}^{*}\|_{L^{2}(p_{T-t})}^{2}\big\} \\
&~~~~~~~~~~\quad +\bar{B}^{2}\bigg\{\frac{\mathrm{Pdim}(\scrH_{t})}{n\log^{-1}n}\bigg\}^{1/2}+\frac{\lambda d\bar{B}^{2}}{\sigma_{T-t}^{4}}\bigg\{\frac{\mathrm{Pdim}(\nabla\scrH_{t})}{n\log^{-1}n}\bigg\}^{1/2}+\frac{\lambda^{2}d^{2}\bar{B}^{2}}{\sigma_{T-t}^{8}}\,, \\
&\bbE\big\{\|\nabla\hat{h}_{t}^{\lambda}-\nabla h_{t}^{*}\|_{L^{2}(p_{T-t})}^{2}\big\}
\lesssim \inf_{h_{t}\in\scrH_{t}}\bigg\{\frac{1}{\lambda}\|h_{t}-h_{t}^{*}\|_{L^{2}(p_{T-t})}^{2}+\|\nabla h_{t}-\nabla h_{t}^{*}\|_{L^{2}(p_{T-t})}^{2}\bigg\} \\
&~~~~~~~~~~\quad +\frac{\bar{B}^{2}}{\lambda}\bigg\{\frac{\mathrm{Pdim}(\scrH_{t})}{n\log^{-1}n}\bigg\}^{1/2}+\frac{d\bar{B}^{2}}{\sigma_{T-t}^{4}}\bigg\{\frac{\mathrm{Pdim}(\nabla\scrH_{t})}{n\log^{-1}n}\bigg\}^{1/2}+\frac{\lambda d^{2}\bar{B}^{2}}{\sigma_{T-t}^{8}}\,,
\end{align*} 
where the notation $\lesssim$ hides absolute constants.
\end{customlemma}

\begin{proof}[Proof of Lemma~\ref{lemma:oracle:inequality:guidance}]
The class $\scrH_{t}$ has a countable dense subset for locally uniform convergence of functions and their first derivatives.\footnote{Apply Le Ma\^itre's \href{https://flemaitre.perso.math.cnrs.fr/enseignements/1718/tde/DSTnotes.pdf}{descriptive set theory notes}, Corollary~1.97 (p.~32), Theorem~4.7 (p.~75), and Proposition~4.8 (p.~76), to the continuous maps $\vx\mapsto(h_{t}(\vx),\nabla h_{t}(\vx))$.} The stated envelopes and dominated convergence allow the empirical-process suprema below to be taken over this subset, so they are measurable.

The identity in the proof of Proposition~\ref{proposition:stability} and Lemma~\ref{lemma:oracle:risk} give 
\begin{align*}
&\bbE\big\{\|\hat{h}_{t}^{\lambda}-h_{t}^{\lambda}\|_{L^{2}(p_{T-t})}^{2}\big\}+\lambda\bbE\big\{\|\nabla\hat{h}_{t}^{\lambda}-\nabla h_{t}^{\lambda}\|_{L^{2}(p_{T-t})}^{2}\big\} \\
&~~~~~~~= \bbE\big\{\calJ_{t}^{\lambda}(\hat{h}_{t}^{\lambda})-\calJ_{t}^{\lambda}(h_{t}^{\lambda})\big\} \\
&~~~~~~~\leq \inf_{h_{t}\in\scrH_{t}}\big\{\|h_{t}-h_{t}^{\lambda}\|_{L^{2}(p_{T-t})}^{2}+\lambda\|\nabla h_{t}-\nabla h_{t}^{\lambda}\|_{L^{2}(p_{T-t})}^{2}\big\} \\
&~~~~~~~\quad +48\bar{B}^{2}\bigg\{\frac{\mathrm{Pdim}(\scrH_{t})}{n\log^{-1}n}\bigg\}^{1/2}+160\frac{\lambda d\bar{B}^{2}}{\sigma_{T-t}^{4}}\bigg\{\frac{\mathrm{Pdim}(\nabla\scrH_{t})}{n\log^{-1}n}\bigg\}^{1/2} \\
&~~~~~~~\leq \inf_{h_{t}\in\scrH_{t}}\big\{2\|h_{t}-h_{t}^{*}\|_{L^{2}(p_{T-t})}^{2}+2\lambda\|\nabla h_{t}-\nabla h_{t}^{*}\|_{L^{2}(p_{T-t})}^{2}\big\} \\
&~~~~~~~\quad +48\bar{B}^{2}\bigg\{\frac{\mathrm{Pdim}(\scrH_{t})}{n\log^{-1}n}\bigg\}^{1/2}+160\frac{\lambda d\bar{B}^{2}}{\sigma_{T-t}^{4}}\bigg\{\frac{\mathrm{Pdim}(\nabla\scrH_{t})}{n\log^{-1}n}\bigg\}^{1/2} \\
&~~~~~~~\quad +2\|h_{t}^{*}-h_{t}^{\lambda}\|_{L^{2}(p_{T-t})}^{2}+2\lambda\|\nabla h_{t}^{*}-\nabla h_{t}^{\lambda}\|_{L^{2}(p_{T-t})}^{2}\,,
\end{align*}
where the last inequality uses $\|u+v\|^{2}\leq2\|u\|^{2}+2\|v\|^{2}$. Applying the same inequality once more gives
\begin{align}
&\bbE\big\{\|\hat{h}_{t}^{\lambda}-h_{t}^{*}\|_{L^{2}(p_{T-t})}^{2}\big\}+\lambda\bbE\big\{\|\nabla\hat{h}_{t}^{\lambda}-\nabla h_{t}^{*}\|_{L^{2}(p_{T-t})}^{2}\big\} \notag \\
&~~~~~~~\leq 2\bbE\big\{\|\hat{h}_{t}^{\lambda}-h_{t}^{\lambda}\|_{L^{2}(p_{T-t})}^{2}\big\}+2\lambda\bbE\big\{\|\nabla\hat{h}_{t}^{\lambda}-\nabla h_{t}^{\lambda}\|_{L^{2}(p_{T-t})}^{2}\big\} \notag \\
&~~~~~~~\quad +2\|h_{t}^{*}-h_{t}^{\lambda}\|_{L^{2}(p_{T-t})}^{2}+2\lambda\|\nabla h_{t}^{*}-\nabla h_{t}^{\lambda}\|_{L^{2}(p_{T-t})}^{2} \label{eq:lemma:oracle:inequality:guidance:1} \\
&~~~~~~~\leq \underbrace{\inf_{h_{t}\in\scrH_{t}}\big\{4\|h_{t}-h_{t}^{*}\|_{L^{2}(p_{T-t})}^{2}+4\lambda\|\nabla h_{t}-\nabla h_{t}^{*}\|_{L^{2}(p_{T-t})}^{2}\big\}}_{\text{approximation error}} \notag \\
&~~~~~~~\quad +\underbrace{96\bar{B}^{2}\bigg\{\frac{\mathrm{Pdim}(\scrH_{t})}{n\log^{-1}n}\bigg\}^{1/2}+320\frac{\lambda d\bar{B}^{2}}{\sigma_{T-t}^{4}}\bigg\{\frac{\mathrm{Pdim}(\nabla\scrH_{t})}{n\log^{-1}n}\bigg\}^{1/2}}_{\text{generalization error}} \notag \\
&~~~~~~~\quad +\underbrace{6\|h_{t}^{*}-h_{t}^{\lambda}\|_{L^{2}(p_{T-t})}^{2}+6\lambda\|\nabla h_{t}^{*}-\nabla h_{t}^{\lambda}\|_{L^{2}(p_{T-t})}^{2}}_{\text{regularization gap}}\,.\notag
\end{align}
Combining Proposition~\ref{proposition:regularization:gap} and Lemma~\ref{lemma:reg:gap:bound} yields
\begin{align}
\|h_{t}^{*}-h_{t}^{\lambda}\|_{L^{2}(p_{T-t})}^{2}
 \leq 144\lambda^{2}\frac{d^{2}\bar{B}^{2}}{\sigma_{T-t}^{8}}\,, \quad \label{eq:lemma:oracle:inequality:guidance:2}
\|\nabla h_{t}^{*}-\nabla h_{t}^{\lambda}\|_{L^{2}(p_{T-t})}^{2}
\leq 144\lambda\frac{d^{2}\bar{B}^{2}}{\sigma_{T-t}^{8}}\,.
\end{align}
Substituting~\eqref{eq:lemma:oracle:inequality:guidance:2} into~\eqref{eq:lemma:oracle:inequality:guidance:1}, the regularization gap is at most $1728\lambda^{2}d^{2}\bar {B}^{2}\sigma_{T-t}^{-8}$. Dropping the nonnegative gradient term gives the first asserted bound; dropping the function term and dividing by $\lambda>0$ gives the second. This completes the proof.
\end{proof}

\begin{lemma}\label{lemma:oracle:risk}
Suppose Assumptions~\ref{assumption:bounded:support} and~\ref{assumption:bounded:weight} hold. Fix $t\in[0,T)$, let $n$ be sufficiently large, and let $\lambda>0$. Let $\scrH_t$ be as in Lemma~\ref{lemma:oracle:inequality:guidance}. Let $\hat h_t^\lambda$ be the gradient-regularized empirical risk minimizer defined as~\eqref{eq:erm:GP}. Then
\begin{align*}
&\bbE\big\{\calJ_t^\lambda(\hat h_t^\lambda)-\calJ_t^\lambda(h_t^\lambda)\big\}\\
&\leq\inf_{h_t\in\scrH_t}\big\{\|h_t-h_t^\lambda\|_{L^2(p_{T-t})}^2+\lambda\|\nabla h_t-\nabla h_t^\lambda\|_{L^2(p_{T-t})}^2\big\}\\
&\quad+48\bar B^2\bigg\{\frac{\mathrm{Pdim}(\scrH_t)}{n\log^{-1}n}\bigg\}^{1/2}
+\frac{160\lambda d\bar B^2}{\sigma_{T-t}^4}\bigg\{\frac{\mathrm{Pdim}(\nabla\scrH_t)}{n\log^{-1}n}\bigg\}^{1/2}\,.
\end{align*}
\end{lemma}

\begin{proof}[Proof of Lemma~\ref{lemma:oracle:risk}]
For any fixed $h_t\in\scrH_t$, we have
\begin{align*}
&\calJ_t^\lambda(\hat h_t^\lambda)-\calJ_t^\lambda(h_t^\lambda)\\
&~~~=\calJ_t^\lambda(\hat h_t^\lambda)-\what{\calJ}_t^\lambda(\hat h_t^\lambda)
+\what{\calJ}_t^\lambda(\hat h_t^\lambda)-\what{\calJ}_t^\lambda(h_t)+\what{\calJ}_t^\lambda(h_t)-\calJ_t^\lambda(h_t)
+\calJ_t^\lambda(h_t)-\calJ_t^\lambda(h_t^\lambda)\\
&~~~\leq\calJ_t^\lambda(\hat h_t^\lambda)-\what{\calJ}_t^\lambda(\hat h_t^\lambda)
+\what{\calJ}_t^\lambda(h_t)-\calJ_t^\lambda(h_t)+\calJ_t^\lambda(h_t)-\calJ_t^\lambda(h_t^\lambda)\,,
\end{align*}
where the inequality follows from the fact that $\hat h_t^\lambda$ minimizes $\what{\calJ}_t^\lambda$ over $\scrH_t$. Taking expectations on both sides yields
\begin{align*}
&\bbE\big\{\calJ_t^\lambda(\hat h_t^\lambda)-\calJ_t^\lambda(h_t^\lambda)\big\} \leq\bbE\big\{\calJ_t^\lambda(\hat h_t^\lambda)-\what{\calJ}_t^\lambda(\hat h_t^\lambda)\big\}
+\calJ_t^\lambda(h_t)-\calJ_t^\lambda(h_t^\lambda)\\
&~~~~~~~~~~~~~~~~~\leq\bbE\bigg[\sup_{h_t\in\scrH_t}\big\{\calJ_t^\lambda(h_t)-\what{\calJ}_t^\lambda(h_t)\big\}\bigg]
+\calJ_t^\lambda(h_t)-\calJ_t^\lambda(h_t^\lambda)\,,
\end{align*}
where the first inequality uses $\bbE\big\{\what{\calJ}_t^{\lambda}(h_t)\big\}=\calJ_t^{\lambda}(h_t)$ for fixed $h_t$, and the second uses $\hat h_t^{\lambda}\in\scrH_t$. Taking the infimum over these fixed comparison functions gives
\begin{align}
&\bbE\big\{\calJ_t^\lambda(\hat h_t^\lambda)-\calJ_t^\lambda(h_t^\lambda)\big\}\notag\\
&~~~~~~~\leq\underbrace{\bbE\bigg[\sup_{h_t\in\scrH_t}\big\{\calJ_t^\lambda(h_t)-\what{\calJ}_t^\lambda(h_t)\big\}\bigg]}_{\text{generalization error}}
+\underbrace{\inf_{h_t\in\scrH_t}\big\{\calJ_t^\lambda(h_t)-\calJ_t^\lambda(h_t^\lambda)\big\}}_{\text{approximation error}}\,.
\label{eq:lemma:oracle:risk:1}
\end{align}
The rest of the proof is divided into three steps.

\par\noindent\emph{Step 1. Generalization error in~\eqref{eq:lemma:oracle:risk:1}.}
Splitting the supremum gives
\begin{align}
&\bbE\bigg[\sup_{h_t\in\scrH_t}\big\{\calJ_t^\lambda(h_t)-\what{\calJ}_t^\lambda(h_t)\big\}\bigg]\notag\\
&~~~~~~~\leq\underbrace{\bbE\bigg[\sup_{h_t\in\scrH_t}\bigg\{\calJ_{t}(h_t)-\frac{1}{n}\sum_{i=1}^n\big|h_t(\mX_{T-t}^i)-w(\mX_0^i)\big|^2\bigg\}\bigg]}_{\text{(G1)}}\notag\\
&~~~~~~~\quad+\lambda\underbrace{\bbE\bigg[\sup_{h_t\in\scrH_t}\bigg\{\|\nabla h_t\|_{L^{2}(p_{T-t})}^2-\frac{1}{n}\sum_{i=1}^n|\nabla h_t(\mX_{T-t}^i)|_2^2\bigg\}\bigg]}_{\text{(G2)}}\,.
\label{eq:lemma:oracle:risk:G1}
\end{align}
For (G1), set $\xi_{i}=w(\mX_{0}^{i})-h_{t}^{*}(\mX_{T-t}^{i})$. Conditional on $\mX_{T-t}^{1:n}$, the residuals are independent, with $\bbE(\xi_{i}\given\mX_{T-t}^{1:n})=0$ and $|\xi_{i}|\leq\bar {B}$. Also, $\bbE(\xi_{i}^{2})=V_{t}^{2}$. By Proposition~\ref{proposition:regression:constant},
\begin{equation}\label{eq:lemma:oracle:risk:G2}
\calJ_{t}(h_t)=\|h_t-h_t^*\|_{L^{2}(p_{T-t})}^2+V_t^{2}\,.
\end{equation}
Expanding the empirical loss gives
\begin{align*}
&\frac{1}{n}\sum_{i=1}^n\big|h_t(\mX_{T-t}^i)-w(\mX_0^i)\big|^2
=\frac{1}{n}\sum_{i=1}^n\big|h_t(\mX_{T-t}^i)-h_t^*(\mX_{T-t}^i)\big|^2+\frac{1}{n}\sum_{i=1}^n\xi_{i}^{2}\\
&~~~~~~~~~~~~~~~~~~~~~~~~~~~~~~~~~~~~~~~~~~~~~~~~~~~~~~~~-\frac{2}{n}\sum_{i=1}^{n} h_t(\mX_{T-t}^i)\xi_{i}
+\frac{2}{n}\sum_{i=1}^{n} h_t^*(\mX_{T-t}^i)\xi_{i}\,.
\end{align*}
The terms $V_{t}^2-n^{-1}\sum_i\xi_i^2$ and $n^{-1}\sum_i h_t^*(\mX_{T-t}^i)\xi_{i}$ do not depend on $h_{t}$ and have zero expectation, since
\begin{align*}
&\bbE\bigg(V_t^{2}-\frac{1}{n}\sum_{i=1}^{n}\xi_{i}^{2}\bigg)=0\,,\\
&\bbE\big\{h_t^*(\mX_{T-t}^i)\xi_i\big\}
=\bbE\big\{ h_{t}^{*}(\mX_{T-t}^{i})\bbE(\xi_{i}\given\mX_{T-t}^{1:n})\big\}=0\,.
\end{align*}
Moving these terms outside the supremum and using~\eqref{eq:lemma:oracle:risk:G2}, we obtain
\begin{align}
&\text{(G1)}=\bbE\bigg[\sup_{h_t\in\scrH_t}\bigg\{\|h_t-h_{t}^{*}\|_{L^{2}(p_{T-t})}^2
-\frac{1}{n}\sum_{i=1}^n\big|h_t(\mX_{T-t}^i)-h_t^*(\mX_{T-t}^i)\big|^{2}\notag\\
&~~~~~~~~~~~~~~~~~~~+\frac{2}{n}\sum_{i=1}^{n} h_{t}(\mX_{T-t}^{i})\xi_{i}\bigg\}\bigg]\notag\\
&~~~~~~~\leq\bbE\bigg[\sup_{h_{t}\in\scrH_{t}}\bigg\{\|h_{t}-h_{t}^{*}\|_{L^{2}(p_{T-t})}^{2}
-\frac{1}{n}\sum_{i=1}^{n}\big|h_{t}(\mX_{T-t}^{i})-h_{t}^{*}(\mX_{T-t}^{i})\big|^2\bigg\}\bigg]\notag\\
&~~~~~~~~~~~~~~~~~~~+2\bbE\bigg[\sup_{h_{t}\in\scrH_{t}}\bigg\{\frac{1}{n}\sum_{i=1}^n h_t(\mX_{T-t}^i)\xi_{i}\bigg\}\bigg]\notag\\
&~~~~~~~\leq(32+2\cdot8)\bar {B}^{2}\bigg\{\frac{\mathrm{Pdim}(\scrH_{t})\log n}{n}\bigg\}^{1/2}
=48\bar {B}^{2}\bigg\{\frac{\mathrm{Pdim}(\scrH_{t})\log n}{n}\bigg\}^{1/2}\,.
\label{eq:lemma:oracle:risk:G3}
\end{align}
The first inequality splits the supremum. The last uses Lemma~\ref{lemma:generalization:weight} and the expectation of the conditional bound in Lemma~\ref{lemma:complexity:weight}.

\par For (G2), Lemma~\ref{lemma:generalization:weight:grad} gives
\begin{equation}\label{eq:lemma:oracle:risk:G4}
\text{(G2)}\leq\frac{160d\bar B^2}{\sigma_{T-t}^4}\bigg\{\frac{\mathrm{Pdim}(\nabla\scrH_t)\log n}{n}\bigg\}^{1/2}\,.
\end{equation}
Substituting~\eqref{eq:lemma:oracle:risk:G3} and~\eqref{eq:lemma:oracle:risk:G4} into~\eqref{eq:lemma:oracle:risk:G1} yields
\begin{align}
&\bbE\bigg[\sup_{h_t\in\scrH_t}\big\{\calJ_t^\lambda(h_t)-\what{\calJ}_t^\lambda(h_t)\big\}\bigg]
\leq48\bar B^2\bigg\{\frac{\mathrm{Pdim}(\scrH_t)\log n}{n}\bigg\}^{1/2}\notag\\
&~~~~~~~~~~~~~~~~~~~~~~~~~~~~~~~~~~~~~~~~~~+\frac{160\lambda d\bar B^2}{\sigma_{T-t}^4}\bigg\{\frac{\mathrm{Pdim}(\nabla\scrH_t)\log n}{n}\bigg\}^{1/2}\,.
\label{eq:lemma:oracle:risk:G5}
\end{align}

\par\noindent\emph{Step 2. Approximation error in~\eqref{eq:lemma:oracle:risk:1}.}
The proof of Proposition~\ref{proposition:stability} gives
\begin{equation}\label{eq:lemma:oracle:risk:A1}
\calJ_t^\lambda(h_t)-\calJ_t^\lambda(h_t^\lambda)
=\|h_t-h_t^\lambda\|_{L^2(p_{T-t})}^2+\lambda\|\nabla h_t-\nabla h_t^\lambda\|_{L^2(p_{T-t})}^2\,.
\end{equation}

\par\noindent\emph{Step 3. Conclusion.}
Substituting~\eqref{eq:lemma:oracle:risk:G5} and~\eqref{eq:lemma:oracle:risk:A1} into~\eqref{eq:lemma:oracle:risk:1} completes the proof.
\end{proof}

\subsection{Auxilary lemmas for the oracle inequality}

\par According to the standard techniques of symmetrization~\citep[Theorem 3.3]{mohri2018foundations}, we have the following generalization bounds. We introduce the concept of Rademacher complexity~\citep{Bartlett2002Rademacher,mohri2018foundations}, which is crucial for analyzing the generalization error.

\begin{definition}[Rademacher complexity]
Let $\scrH$ be a function class, and let $\mX^{1:n}\coloneq(\mX^{1},\ldots,\mX^{n})$ be a set of samples. The empirical Rademacher complexity of $\scrH$ with respect to $\mX^{1:n}$ is defined as
\begin{equation*}
\frakR(\scrH\given\mX^{1:n})\coloneq\bbE\bigg[\sup_{h\in\scrH}\bigg\{\frac{1}{n}\sum_{i=1}^{n}\varepsilon_{i}h(\mX^{i})\bigg\}\given\mX^{1:n}\bigg]\,,
\end{equation*} 
where $\varepsilon_{1},\ldots,\varepsilon_{n}$ are i.i.d. Rademacher random variables. The Rademacher complexity of $\scrH$ is the expectation of empirical Rademacher complexity with respect to the distribution of $\mX^{1:n}$ defined as 
\begin{equation*}
\frakR_{n}(\scrH)\coloneq\bbE\big\{\frakR(\scrH\given\mX^{1:n})\big\}=\bbE\bigg[\sup_{h\in\scrH}\bigg\{\frac{1}{n}\sum_{i=1}^{n}\varepsilon_{i}h(\mX^{i})\bigg\}\bigg]\,.
\end{equation*}
\end{definition}

\begin{lemma}\label{lemma:generalization:weight}
Suppose Assumptions~\ref{assumption:bounded:support} and~\ref{assumption:bounded:weight} hold. Fix $t\in[0,T)$ and let $n$ be sufficiently large. Let $\scrH_{t}$ satisfy the function-class conditions of Lemma~\ref{lemma:oracle:inequality:guidance}. Then 
\begin{align*}
&\bbE\bigg[\sup_{h_t\in\scrH_t}\bigg\{\|h_t-h_t^*\|_{L^2(p_{T-t})}^2-\frac{1}{n}\sum_{i=1}^n|h_t(\mX_{T-t}^i)-h_t^*(\mX_{T-t}^i)|^2\bigg\}\bigg]\\
&~~~~~~~~~~~~~~~~~~~~~~~~~~~~~~~\leq32\bar B^2\bigg\{\frac{\mathrm{Pdim}(\scrH_t)}{n\log^{-1}n}\bigg\}^{1/2}\,.
\end{align*}
where the expectation is taken with respect to $\mX_{T-t}^{1},\ldots,\mX_{T-t}^{n}\sim^{\iid}p_{T-t}$.
\end{lemma}

\begin{proof}[Proof of Lemma~\ref{lemma:generalization:weight}]
Let $\mX_{T-t}^{1,\prime},\ldots,\mX_{T-t}^{n,\prime}$ be independent copies of $\mX_{T-t}^{1},\ldots,\mX_{T-t}^{n}$, and let $\varepsilon_1,\ldots,\varepsilon_n$ be i.i.d. Rademacher variables independent of both samples. It follows that
\begin{align}
&\bbE\bigg[\sup_{h_t\in\scrH_t}\bigg\{\|h_t-h_t^*\|_{L^2(p_{T-t})}^2
-\frac{1}{n}\sum_{i=1}^n|h_t(\mX_{T-t}^i)-h_t^*(\mX_{T-t}^i)|^2\bigg\}\bigg]\notag\\
&~~~~~~~=\bbE\bigg(\sup_{h_t\in\scrH_t}\bigg[\bbE_{\mX'}\bigg\{\frac{1}{n}\sum_{i=1}^n
|h_t(\mX_{T-t}^{i,\prime})-h_t^*(\mX_{T-t}^{i,\prime})|^2\bigg\}\notag\\
&~~~~~~~\qquad-\frac{1}{n}\sum_{i=1}^n|h_t(\mX_{T-t}^i)-h_t^*(\mX_{T-t}^i)|^2\bigg]\bigg)\notag\\
&~~~~~~~\leq\bbE\bigg(\sup_{h_t\in\scrH_t}\bigg[\frac{1}{n}\sum_{i=1}^n\bigl\{|h_t(\mX_{T-t}^{i,\prime})-h_t^*(\mX_{T-t}^{i,\prime})|^2
-|h_t(\mX_{T-t}^{i})-h_t^*(\mX_{T-t}^{i})|^2\bigr\}\bigg]\bigg)\notag\\
&~~~~~~~=\bbE\bigg(\sup_{h_t\in\scrH_t}\bigg[\frac{1}{n}\sum_{i=1}^n\varepsilon_i\bigl\{|h_t(\mX_{T-t}^{i,\prime})-h_t^*(\mX_{T-t}^{i,\prime})|^2
-|h_t(\mX_{T-t}^{i})-h_t^*(\mX_{T-t}^{i})|^2\bigr\}\bigg]\bigg)\notag\\
&~~~~~~~\leq2\bbE\bigg[\sup_{h_t\in\scrH_t}\bigg\{\frac{1}{n}\sum_{i=1}^n\varepsilon_i
|h_t(\mX_{T-t}^{i})-h_t^*(\mX_{T-t}^{i})|^2\bigg\}\bigg]
\leq8\bar B\frakR_n(\scrH_t)\,.\label{eq:lemma:generalization:weight:1}
\end{align}
Here $\bbE_{\mX'}$ averages over the independent copies. The first inequality follows from Jensen's inequality; the next equality uses exchangeability of each sample--copy pair, and the next inequality splits the supremum. The last inequality follows from the contraction inequality~\citep[Lemma 5.7, p.~93]{mohri2018foundations}, since $|h_{t}-h_{t}^{*}|\leq\bar {B}$ and subtracting the fixed function $h_{t}^{*}$ does not change the Rademacher complexity.

\par It remains to bound $\frakR_n(\scrH_t)$ in~\eqref{eq:lemma:generalization:weight:1}. Fix the design sample. Let $\delta>0$ and let $\scrH_t^\delta\subset\scrH_{t}$ be an $L^\infty(\mX_{T-t}^{1:n})$ $\delta$-cover with $|\scrH_t^\delta|=N\{\delta,\scrH_t,L^\infty(\mX_{T-t}^{1:n})\}$. For every $h_t\in\scrH_t$, choose $h_t^\delta\in\scrH_t^\delta$ such that
\[
\max_{1\leq i\leq n}|h_{t}(\mX_{T-t}^{i})-h_{t}^{\delta}(\mX_{T-t}^{i})|\leq\delta\,.
\]
Then
\[
\frac{1}{n}\sum_{i=1}^{n}\varepsilon_{i}\{h_{t}(\mX_{T-t}^{i})-h_{t}^{\delta}(\mX_{T-t}^{i})\}
\leq\frac{1}{n}\sum_{i=1}^{n}|h_{t}(\mX_{T-t}^{i})-h_{t}^{\delta}(\mX_{T-t}^{i})|\leq\delta\,.
\]
Consequently,
\begin{align*}
&\frakR(\scrH_{t}\given\mX_{T-t}^{1:n})
=\bbE_{\varepsilon}\bigg[\sup_{h_{t}\in\scrH_{t}}\bigg\{\frac{1}{n}\sum_{i=1}^n\varepsilon_i h_t(\mX_{T-t}^i)\bigg\}\bigg]\\
&~~~~~~~\leq\bbE_{\varepsilon}\bigg[\sup_{h_{t}^{\delta}\in\scrH_{t}^{\delta}}\bigg\{\frac{1}{n}\sum_{i=1}^n\varepsilon_i h_t^\delta(\mX_{T-t}^i)\bigg\}\bigg]+\delta\\
&~~~~~~~\leq\bar B\bigg(\frac{2\log|\scrH_t^\delta|}{n}\bigg)^{1/2}+\delta
=\bar B\bigg[\frac{2\log N\{\delta,\scrH_t,L^\infty(\mX_{T-t}^{1:n})\}}{n}\bigg]^{1/2}+\delta\,,
\end{align*}
where the second inequality follows from Massart's lemma~\citep[Theorem 3.7, p.~35]{mohri2018foundations}.

\par For $1\leq\mathrm{Pdim}(\scrH_t)<n$, Theorem 12.2 of~\citet[p.~167]{Anthony1999neural}, with $\delta=\bar B/\sqrt n$, gives
\begin{align*}
&\log N\{\bar {B}/\sqrt {n},\scrH_{t},L^{\infty}(\mX_{T-t}^{1:n})\}\\
&~~~~~~~\leq\mathrm{Pdim}(\scrH_{t})\log\bigg\{\frac{en^{3/2}}{\mathrm{Pdim}(\scrH_{t})}\bigg\}
\leq3\,\mathrm{Pdim}(\scrH_{t})\log n\,.
\end{align*}
It follows that
\begin{align}
&\frakR(\scrH_t\given\mX_{T-t}^{1:n})
\leq\bar B\bigg[\frac{2\log N\{\bar B/\sqrt n,\scrH_t,L^\infty(\mX_{T-t}^{1:n})\}}{n}\bigg]^{1/2}+\frac{\bar {B}}{\sqrt {n}}\notag\\
&~~~~~~~\leq\bar B\bigg\{\frac{6\,\mathrm{Pdim}(\scrH_t)\log n}{n}\bigg\}^{1/2}+\frac{\bar {B}}{\sqrt {n}}
\leq4\bar {B}\bigg\{\frac{\mathrm{Pdim}(\scrH_{t})\log n}{n}\bigg\}^{1/2}\,.\label{eq:lemma:generalization:weight:2}
\end{align}
The last inequality uses $n\geq2$, $\mathrm{Pdim}(\scrH_{t})\geq1$, and $\sqrt{6}+1/\sqrt{\log2}<4$. If $\mathrm{Pdim}(\scrH_{t})=0$, all functions coincide pointwise and the complexity is zero; if $\mathrm{Pdim}(\scrH_{t})\geq n$, the bound follows from $\frakR(\scrH_{t}\given\mX_{T-t}^{1:n})\leq\bar {B}$. Taking expectations in~\eqref{eq:lemma:generalization:weight:2} and substituting into~\eqref{eq:lemma:generalization:weight:1} completes the proof.
\end{proof}

\par By a similar argument as Lemma~\ref{lemma:generalization:weight}, we have the following generalization bounds for the gradient term.

\begin{lemma}\label{lemma:generalization:weight:grad}
Suppose Assumptions~\ref{assumption:bounded:support} and~\ref{assumption:bounded:weight} hold. Fix $t\in[0,T)$ and let $n$ be sufficiently large. Let $\scrH_{t}$ satisfy the function-class conditions of Lemma~\ref{lemma:oracle:inequality:guidance}. Then
\begin{align*}
&\bbE\bigg[\sup_{h_t\in\scrH_t}\bigg\{\|\nabla h_t\|_{L^2(p_{T-t})}^2-\frac{1}{n}\sum_{i=1}^n|\nabla h_t(\mX_{T-t}^i)|_2^2\bigg\}\bigg]\\
&~~~~~~~~~~~~~~~~~~~\leq\frac{160d\bar B^2}{\sigma_{T-t}^4}\bigg\{\frac{\mathrm{Pdim}(\nabla\scrH_t)\log n}{n}\bigg\}^{1/2}\,,
\end{align*}
where the expectation is taken with respect to $\mX_{T-t}^{1},\ldots,\mX_{T-t}^{n}\sim^{\iid}p_{T-t}$.
\end{lemma}

\begin{proof}[Proof of Lemma~\ref{lemma:generalization:weight:grad}]
Following the proof of Lemma~\ref{lemma:generalization:weight} coordinatewise, we obtain
\begin{align*}
&\bbE\bigg[\sup_{h_t\in\scrH_t}\bigg\{\|\nabla h_t\|_{L^2(p_{T-t})}^2-\frac{1}{n}\sum_{i=1}^n|\nabla h_t(\mX_{T-t}^i)|_2^2\bigg\}\bigg]\\
&~~~~\leq\sum_{k=1}^d\bbE\bigg[\sup_{h_t\in\scrH_t}\bigg\{\|D_kh_t\|_{L^2(p_{T-t})}^2-\frac{1}{n}\sum_{i=1}^n|D_kh_t(\mX_{T-t}^i)|^2\bigg\}\bigg] \leq\frac{16\bar {B}}{\sigma_{T-t}^{2}}\sum_{k=1}^d\frakR_{n}(D_{k}\scrH_{t})\\
&~~~~\leq\frac{160\bar B^2}{\sigma_{T-t}^4}\sum_{k=1}^{d}\bigg\{\frac{\mathrm{Pdim}(D_k\scrH_t)\log n}{n}\bigg\}^{1/2}
\leq\frac{160d\bar B^2}{\sigma_{T-t}^4}\bigg\{\frac{\mathrm{Pdim}(\nabla\scrH_t)\log n}{n}\bigg\}^{1/2}\,.
\end{align*}
The first inequality splits the supremum across coordinates. The second follows from symmetrization and contraction, using $\|D_{k}h_{t}\|_{\infty}\leq2\bar {B}\sigma_{T-t}^{-2}$; the third uses the covering argument leading to~\eqref{eq:lemma:generalization:weight:2}. For the covering estimate, translate the derivative range to $[0,4\bar {B}\sigma_{T-t}^{-2}]$; the finite-class bound uses the original envelope $2\bar {B}\sigma_{T-t}^{-2}$. The last inequality uses the definition of $\mathrm{Pdim}(\nabla\scrH_t)$.
\end{proof}

\par The following lemma is an extension of~\citet[Lemma 4]{Bartlett2002Rademacher}.

\begin{lemma}\label{lemma:bounded:complexity:rademacher}
Let $\vz=(z_{1},\ldots,z_{n})\in\calZ\subseteq\bbR^{n}$. Let $\xi_{1},\ldots,\xi_{n}$ be independent random variables, not necessarily identically distributed, with $|\xi_{i}|\leq K$ and $\bbE(\xi_{i})=0$ for each $1\leq i\leq n$. Then it follows that
\begin{equation*}
\bbE\bigg\{\sup_{\vz\in\calZ}\bigg(\frac{1}{n}\sum_{i=1}^{n}\xi_{i}z_{i}\bigg)\bigg\}\leq 2K\bbE\bigg\{\sup_{\vz\in\calZ}\bigg(\frac{1}{n}\sum_{i=1}^{n}\varepsilon_{i}z_{i}\bigg)\bigg\}\,,
\end{equation*}
where $\varepsilon_{1},\ldots,\varepsilon_{n}$ is a sequence of i.i.d. Rademacher variables.
\end{lemma}

\begin{proof}[Proof of Lemma~\ref{lemma:bounded:complexity:rademacher}]
The proof relies on the symmetrization technique. Let $\xi_{1}^{\prime},\ldots,\xi_{n}^{\prime}$ be independent copies of $\xi_{1},\ldots,\xi_{n}$. It follows that 
\begin{align*}
&\bbE\bigg\{\sup_{\vz\in\calZ}\bigg(\frac{1}{n}\sum_{i=1}^n\xi_i z_i\bigg)\bigg\}=\bbE\bigg[\sup_{\vz\in\calZ}\bigg\{\bbE\bigg(\frac{1}{n}\sum_{i=1}^{n}\xi_{i} z_{i}-\frac{1}{n}\sum_{i=1}^{n}\xi_{i}' z_{i}\given\xi_{1},\ldots,\xi_{n}\bigg)\bigg\}\bigg]\\
&~~~~~~~\leq\bbE\bigg[\sup_{\vz\in\calZ}\bigg\{\frac{1}{n}\sum_{i=1}^n(\xi_i-\xi_i')z_i\bigg\}\bigg]
=\bbE\bigg[\sup_{\vz\in\calZ}\bigg\{\frac{1}{n}\sum_{i=1}^n\varepsilon_i(\xi_i-\xi_i')z_i\bigg\}\bigg]\\
&~~~~~~~\leq2\bbE\bigg\{\sup_{\vz\in\calZ}\bigg(\frac{1}{n}\sum_{i=1}^n\varepsilon_i\xi_i z_i\bigg)\bigg\}
\leq2K\bbE\bigg\{\sup_{\vz\in\calZ}\bigg(\frac{1}{n}\sum_{i=1}^n\varepsilon_i z_i\bigg)\bigg\}\,.
\end{align*}
where the first equality due to $\bbE(\xi_{i}^{\prime})=0$, the first inequality holds from Jensen's inequality, and the second equality follows from the fact that distribution of $(\xi_{i}-\xi_{i}^{\prime})$ is symmetric around zero, so it has the same distribution as $\varepsilon_{i}(\xi_{i}-\xi_{i}^{\prime})$. The second inequality comes from the triangular inequality for the supremum, and we used the fact that $\xi_{i}$ and $\xi_{i}^{\prime}$ are identically distributed. The last inequality invokes Ledoux-Talagrand contraction inequality~\citep[Lemma 5.7]{mohri2018foundations} and $\max_{1\leq i\leq n}|\xi_{i}|\leq K$. This completes the proof.
\end{proof}

\begin{lemma}\label{lemma:complexity:weight}
Suppose Assumptions~\ref{assumption:bounded:support} and~\ref{assumption:bounded:weight} hold. Fix $t\in[0,T)$ and let $n$ be sufficiently large. Let $\scrH_{t}$ satisfy the function-class conditions of Lemma~\ref{lemma:oracle:inequality:guidance}. Then
\begin{align*}
&\bbE\bigg[\sup_{h_t\in\scrH_t}\bigg\{\frac{1}{n}\sum_{i=1}^n h_t(\mX_{T-t}^i)w(\mX_0^i)
-\frac{1}{n}\sum_{i=1}^{n} h_{t}(\mX_{T-t}^i)h_{t}^{*}(\mX_{T-t}^{i})\bigg\}\given\mX_{T-t}^{1:n}\bigg]\\
&~~~~~~~~~~~~\leq8\bar B^2\bigg\{\frac{\mathrm{Pdim}(\scrH_t)}{n\log^{-1}n}\bigg\}^{1/2}\,.
\end{align*}
The conditional bound holds almost surely for the i.i.d. design sample from $p_{T-t}$.
\end{lemma}

\begin{proof}[Proof of Lemma~\ref{lemma:complexity:weight}]
Define a sequence of auxilary random variables 
\begin{equation*}
\xi_{i} \coloneq w(\mX_{0}^{i})-\bbE\{w(\mX_{0}^{i})\given\mX_{T-t}^{i}\}\,.
\end{equation*}
Conditional on the full design sample, the $\xi_i$ are independent, centered and bounded by $\bar B$, because the training pairs are independent. Their conditional laws may differ. Applying Lemma~\ref{lemma:bounded:complexity:rademacher} yields 
\begin{align*}
&\bbE\bigg[\sup_{h_t\in\scrH_t}\bigg\{\frac{1}{n}\sum_{i=1}^n\xi_i h_t(\mX_{T-t}^i)\bigg\}\given\mX_{T-t}^{1:n}\bigg]\\
&~~~~~~~\leq2\bar B\bbE\bigg[\sup_{h_t\in\scrH_t}\bigg\{\frac{1}{n}\sum_{i=1}^n\varepsilon_i h_t(\mX_{T-t}^i)\bigg\}\given\mX_{T-t}^{1:n}\bigg]=2\bar B\frakR(\scrH_t\given\mX_{T-t}^{1:n})\,.
\end{align*}
where $\varepsilon_{1},\ldots,\varepsilon_{n}$ is a sequence of i.i.d. Rademacher variables. The inequality is the conditional form of Lemma~\ref{lemma:bounded:complexity:rademacher}; the equality is the definition of empirical Rademacher complexity. Finally, using~\eqref{eq:lemma:generalization:weight:2} completes the proof.
\end{proof}

\begin{lemma}\label{lemma:reg:gap:bound}
Suppose Assumptions~\ref{assumption:bounded:support} and~\ref{assumption:bounded:weight} hold, and let $0\leq t<T$. Then
\begin{equation}
\|\Delta h_{t}^{*}+\nabla h_{t}^{*}\cdot\nabla\log p_{T-t}\|_{L^{2}(p_{T-t})} \leq \frac{12d\bar{B}}{\sigma_{T-t}^{4}}\,.
\end{equation}
\end{lemma}

\begin{proof}[Proof of Lemma~\ref{lemma:reg:gap:bound}]
By applying the triangular inequality, we have
\begin{align}
\|\Delta h_{t}^{*}+\nabla h_{t}^{*}\cdot\nabla\log p_{T-t}\|_{L^{2}(p_{T-t})} &\leq \|\Delta h_{t}^{*}\|_{L^{2}(p_{T-t})}+\|\nabla h_{t}^{*}\cdot\nabla\log p_{T-t}\|_{L^{2}(p_{T-t})} \notag \\
&\leq \frac{6d\bar{B}}{\sigma_{T-t}^{4}}+\frac{2\sqrt {d}\bar{B}}{\sigma_{T-t}^{2}}\|\nabla\log p_{T-t}\|_{L^{2}(p_{T-t})}\,, \label{eq:lemma:reg:gap:bound:1}
\end{align}  
where the last inequality holds from Lemmas~\ref{lemma:section:error:guidance:regularity:2} and~\ref{lemma:section:error:guidance:regularity:3}. It remains to estimate the $L^{2}(p_{T-t})$-norm of the score $\nabla\log p_{T-t}$ in~\eqref{eq:lemma:reg:gap:bound:1}. Indeed, 
\begin{align}
&~~~~~~~\|\nabla\log p_{T-t}\|_{L^{2}(p_{T-t})}^{2} \label{eq:lemma:reg:gap:bound:2} \\
&=\int\bigg|\frac{\vx}{\sigma_{T-t}^{2}}-\frac{\mu_{T-t}}{\sigma_{T-t}^{2}}\bbE(\mX_{0}\given\mX_{T-t}=\vx)\bigg|_{2}^{2}p_{T-t}(\vx)\d\vx \leq\frac{2}{\sigma_{T-t}^{4}}\bbE\big(|\mX_{T-t}|_{2}^{2}\big)+\frac{2d\mu_{T-t}^{2}}{\sigma_{T-t}^{4}} \notag \\
&=\frac{2}{\sigma_{T-t}^{4}}\bigg\{\mu_{T-t}^{2}\bbE\big(|\mX_{0}|_{2}^{2}\big)+2\mu_{T-t}\sigma_{T-t}\bbE\big(\langle\mX_{0},\vepsilon\rangle\big)+\sigma_{T-t}^{2}\bbE\big(|\vepsilon|_{2}^{2}\big)\bigg\}+\frac{2d\mu_{T-t}^{2}}{\sigma_{T-t}^{4}} \notag \\
&=\frac{2}{\sigma_{T-t}^{4}}\bigg\{\mu_{T-t}^{2}\bbE\big(|\mX_{0}|_{2}^{2}\big)+\sigma_{T-t}^{2}\bbE\big(|\vepsilon|_{2}^{2}\big)\bigg\}+\frac{2d\mu_{T-t}^{2}}{\sigma_{T-t}^{4}}\leq\frac{6d}{\sigma_{T-t}^{4}}\,, \notag
\end{align}
where the first equality is owing to Lemma~\ref{lemma:tweedie}, the first inequality used Assumption~\ref{assumption:bounded:support}, and the second and third equalities hold from $\mX_{T-t}\stackrel{\d}{=}\mu_{T-t}\mX_{0}+\sigma_{T-t}\vepsilon$ where $\mX_{0}$ is independent $\vepsilon\sim\calN(\bzero,\mI_{d})$. The last inequality also uses Assumption~\ref{assumption:bounded:support}. Substituting~\eqref{eq:lemma:reg:gap:bound:2} into~\eqref{eq:lemma:reg:gap:bound:1} completes the proof.
\end{proof}
\subsection{Convergence rate of the Doob's guidance estimation}
\label{section:error:guidance:rate}

\begin{lemma}\label{lemma:approximation:compact}
Suppose Assumptions~\ref{assumption:bounded:support} and~\ref{assumption:bounded:weight} hold. Let $0\leq t<T$, $R\geq1$, and let $N$ be a sufficiently large integer. Set $L=\lceil C_{d}\log N\rceil$ and $S=\lceil C_{d}N^d\rceil$, where $C_{d}$ depends only on $d$ and is sufficiently large, and let
\[
\scrH_{t}=\left\{h_{t}\in N(L,S):
\begin{aligned}
&\underline {B}/2\leq h_{t}(\vx)\leq3\bar {B}/2\quad\text{for all }\vx\in\bbR^{d}\,,\\
&\max_{1\leq k\leq d}\|D_{k}h_{t}\|_{\infty}\leq4\bar {B}\sigma_{T-t}^{-2}
\end{aligned}\right\}\,.
\]
Then there exists $h_t\in\scrH_t$ such that
\[
\|h_t-h_t^*\|_{L^\infty\{B(\bzero,R)\}}\leq\frac{C\bar B R^2}{\sigma_{T-t}^4N^2}\,,\qquad
\|\nabla h_t-\nabla h_t^*\|_{L^\infty\{B(\bzero,R)\}}\leq\frac{C\bar B R}{\sigma_{T-t}^4N}\,,
\]
where $C$ depends only on $d$.
\end{lemma}

\begin{proof}[Proof of Lemma~\ref{lemma:approximation:compact}]
We first rescale the target function by $g_t^*(\vz)=h_t^*(R\vz)$. Then
\[
D_kg_t^*(\vz)=R D_kh_t^*(R\vz)\,,\qquad
D_{k\ell}^2g_t^*(\vz)=R^2D_{k\ell}^2h_t^*(R\vz)\,.
\]
By~\citet[Lemma 15]{ding2025Semi} and Proposition~\ref{proposition:regularity:h:function}, there is a ReQU network $g_{t}$, with depth at most $C_{d}\log N$ and at most $C_{d}N^d$ nonzero parameters, such that
\[
\|g_t-g_{t}^{*}\|_{L^{\infty}([-2,2]^{d})}\leq\varepsilon_{0}\,,
\qquad
\|\nabla g_{t}-\nabla g_{t}^{*}\|_{L^{\infty}([-2,2]^{d})}\leq R\varepsilon_{1}\,,
\]
where, for a constant $A_{d}\geq1$ depending only on $d$,
\[
\varepsilon_{0}=\frac{A_{d}\bar {B} R^{2}}{\sigma_{T-t}^{4}N^{2}}\,,\qquad
\varepsilon_{1}=\frac{A_{d}\bar {B} R}{\sigma_{T-t}^{4}N}\,.
\]

To keep the inputs inside this region, let $\varrho(u)=(u_{+})^{2}$ and define
\begin{align*}
&\phi(u)=\tfrac{1}{4}\{\varrho(u+1)-\varrho(1-u)\}
-\tfrac{3}{4}\{\varrho(u-1)-\varrho(-u-1)\}\\
&~~~~~~~~~~~~~~~~~~~~~~~~~+\tfrac{1}{2}\{\varrho(u-2)-\varrho(-u-2)\}\,.
\end{align*}
Direct calculation gives
\[
\phi(u)=u\quad(|u|\leq1)\,,\qquad
|\phi(u)|\leq\tfrac{3}{2}\,,\qquad
\phi'(u)=\min\{1,(2-|u|)_{+}\}\in[0,1]\,.
\]
Set $\kappa=\max\{1,\sqrt{2A_{d}\bar {B}/\underline {B}},A_{d}/2\}$. For $N\geq\kappa R/\sigma_{T-t}^{2}$, define
\[
h_{t}(\vx)=g_{t}(\vz)\,,
\]
where $\vz=(\phi(x_{1}/R),\ldots,\phi(x_{d}/R))\in[-3/2,3/2]^{d}$. We have
\begin{align*}
&\varepsilon_{0}\leq\underline {B}/2\,,\qquad
\varepsilon_{1}\leq2\bar {B}\sigma_{T-t}^{-2}\,,\qquad
\underline {B}/2\leq\underline {B}-\varepsilon_{0}
\leq h_{t}(\vx)\leq\bar {B}+\varepsilon_{0}\leq3\bar {B}/2\,,\\
&|D_{k}h_{t}(\vx)|=R^{-1}|D_{k}g_{t}(\vz)|\,|\phi'(x_{k}/R)|
\leq2\bar {B}\sigma_{T-t}^{-2}+\varepsilon_{1}
\leq4\bar {B}\sigma_{T-t}^{-2}\,.
\end{align*}
On $B(\bzero,R)$ the input map is the identity after rescaling, so
\[
\|h_{t}-h_t^*\|_{L^\infty\{B(\bzero,R)\}}\leq\varepsilon_{0}\,,
\qquad
\|\nabla h_t-\nabla h_t^*\|_{L^\infty\{B(\bzero,R)\}}\leq\varepsilon_{1}\,.
\]

Thus the asserted approximation bounds hold with $C=A_{d}$.

Finally, each coordinate input map uses six ReQU units. If $g_t$ has depth $L_{g}$ and $S_{g}$ nonzero parameters, absorbing the output of this map into the first affine layer of $g_t$ gives
\[
L_{h}\leq L_{g}+1\,,\qquad S_{h}\leq6S_{g}+12d\,.
\]
Enlarging the dimension-dependent network budget constants therefore ensures $h_t\in\scrH_t$. This completes the proof.
\end{proof}

\begin{lemma}[Approximation error]\label{lemma:approximation}
Suppose Assumptions~\ref{assumption:bounded:support} and~\ref{assumption:bounded:weight} hold. Let $0\leq t<T$ and let $N$ be a sufficiently large integer. Let $\scrH_t$ be the hypothesis class in Lemma~\ref{lemma:approximation:compact}, with $L=\lceil C_{d}\log N\rceil$ and $S=\lceil C_{d}N^d\rceil$, where $C_{d}$ depends only on $d$ and is sufficiently large. Then there exists $h_{t}\in\scrH_{t}$ such that
\[
\|h_t-h_t^*\|_{L^2(p_{T-t})}^2\leq C\frac{\bar B^2\log^4N}{\sigma_{T-t}^8N^4}\,,\qquad
\|\nabla h_t-\nabla h_t^*\|_{L^2(p_{T-t})}^2\leq C\frac{\bar B^2\log^2N}{\sigma_{T-t}^8N^2}\,,
\]
where $C$ depends only on $d$.
\end{lemma}

\begin{proof}[Proof of Lemma~\ref{lemma:approximation}]
Set $R^{2}=(4d\mu_{T-t}^{2}+8\sigma_{T-t}^{2})\log(N^{4})$. Since $R\geq1$ and $R=O_{d}(\sqrt{\log N})$, for sufficiently large $N$ we may choose $h_{t}\in\scrH_{t}$ as in Lemma~\ref{lemma:approximation:compact}. By Lemma~\ref{lemma:tail:proba:t},
\[
\pr\{|\mX_{T-t}|_{2}>R\}
\leq2^{d+1}\exp\bigg(-\frac{R^{2}}{4d\mu_{T-t}^{2}+8\sigma_{T-t}^{2}}\bigg)
=2^{d+1}N^{-4}\,.
\]
For the function approximation error, we have
\begin{align*}
&\|h_t-h_t^*\|_{L^2(p_{T-t})}^2
=\int_{|\vx|_2\leq R}\{h_t(\vx)-h_t^*(\vx)\}^2p_{T-t}(\vx)\d\vx\\
&~~~~~~~~~~~~~~~~~~~~~~~~~~~~~~~~~~~~~+\int_{|\vx|_{2}>R}\{h_{t}(\vx)-h_{t}^{*}(\vx)\}^{2}p_{T-t}(\vx)\d\vx\,.
\end{align*}
The first term is bounded by Lemma~\ref{lemma:approximation:compact}:
\[
\begin{aligned}
&\int_{|\vx|_{2}\leq R}\{h_{t}(\vx)-h_{t}^{*}(\vx)\}^{2}p_{T-t}(\vx)\d\vx \leq\sup_{|\vx|_{2}\leq R}\big\{|h_{t}(\vx)-h_{t}^{*}(\vx)|^{2}\big\}
\leq\frac{A_{d}^2\bar B^2R^4}{\sigma_{T-t}^8N^4}\,.
\end{aligned}
\]
For the second term, $0<h_{t}\leq3\bar {B}/2$ and $0<h_{t}^{*}\leq\bar {B}$ give
\[
\begin{aligned}
&\int_{|\vx|_{2}>R}\{h_t(\vx)-h_t^*(\vx)\}^2p_{T-t}(\vx)\d\vx\leq\frac{9\bar B^2}{4}\pr\{|\mX_{T-t}|_2>R\}
\leq9\cdot2^{d-1}\bar B^2N^{-4}\,.
\end{aligned}
\]
Combining the two bounds yields
\[
\|h_t-h_t^*\|_{L^2(p_{T-t})}^2
\leq\frac{A_{d}^2\bar B^2R^4}{\sigma_{T-t}^8N^4}+9\cdot2^{d-1}\bar B^2N^{-4}
\leq C\frac{\bar {B}^2\log^4N}{\sigma_{T-t}^8N^{4}}\,.
\]
Similarly, for the gradient approximation error,
\begin{align*}
&\|\nabla h_t-\nabla h_t^*\|_{L^2(p_{T-t})}^2
=\int_{|\vx|_2\leq R}|\nabla h_t(\vx)-\nabla h_t^*(\vx)|_2^2p_{T-t}(\vx)\d\vx\\
&~~~~~~~~~~~~~~~~~~~~~~~~~~~~~~~~~~~~~~~~~~~~~+\int_{|\vx|_2>R}|\nabla h_t(\vx)-\nabla h_t^*(\vx)|_2^2p_{T-t}(\vx)\d\vx\,.
\end{align*}
Lemma~\ref{lemma:approximation:compact} gives
\begin{align*}
&\int_{|\vx|_{2}\leq R}|\nabla h_{t}(\vx)-\nabla h_{t}^{*}(\vx)|_{2}^{2}p_{T-t}(\vx)\d\vx\leq\sup_{|\vx|_{2}\leq R}\big\{|\nabla h_{t}(\vx)-\nabla h_{t}^{*}(\vx)|_{2}^{2}\big\}
\\
&\qquad~~~~~~~~~~~~~~~~~~~\leq\frac{A_{d}^{2}\bar {B}^{2}R^{2}}{\sigma_{T-t}^{8}N^{2}}\,.
\end{align*}
By the derivative bounds for $\scrH_{t}$ and $h_{t}^{*}$,
\begin{align*}
&|\nabla h_{t}(\vx)-\nabla h_{t}^{*}(\vx)|_{2}^{2}
=\sum_{k=1}^{d}\{D_{k}h_{t}(\vx)-D_{k}h_{t}^{*}(\vx)\}^{2}\\
&~~~~~~~~~~~~~~~~~~~~~~~~~~~~~~~~~~~\leq d(4\bar {B}\sigma_{T-t}^{-2}+2\bar {B}\sigma_{T-t}^{-2})^{2}
=36d\bar {B}^{2}\sigma_{T-t}^{-4}\,.
\end{align*}
Hence,
\begin{align*}
&\int_{|\vx|_{2}>R}|\nabla h_{t}(\vx)-\nabla h_{t}^{*}(\vx)|_{2}^{2}p_{T-t}(\vx)\d\vx\\
&~~~~~~~~~~~~~~~~~~~~~~~~~\leq\frac{36d\bar B^2}{\sigma_{T-t}^4}\pr\{|\mX_{T-t}|_2>R\}
\leq\frac{36\cdot2^{d+1}d\bar B^2}{\sigma_{T-t}^4N^4}\,.
\end{align*}
Combining these bounds yields
\[
\|\nabla h_t-\nabla h_t^*\|_{L^2(p_{T-t})}^2
\leq\frac{A_{d}^2\bar B^2R^2}{\sigma_{T-t}^8N^2}+\frac{36\cdot2^{d+1}d\bar B^2}{\sigma_{T-t}^4N^{4}}
\leq C\frac{\bar {B}^2\log^{2}N}{\sigma_{T-t}^8N^2}\,.
\]
The last inequalities use $R^2\leq C_d\log N$, $0<\sigma_{T-t}\leq1$, and sufficiently large $N$, with $C$ depending only on $d$. This completes the proof.
\end{proof}

\begin{lemma}[Generalization error]\label{lemma:generalization}
Let $L,S\geq2$ and let $\scrH_t\subseteq N(L,S)$ be a class of ReQU networks. Then
\begin{align*}
\mathrm{Pdim}(\scrH_t)&\leq cLS\big[L+\log\{(d+1)LS\}\big]\,,\\
\mathrm{Pdim}(\nabla\scrH_t)&\leq cL^2S\big[L+\log\{(d+1)LS\}\big]\,,
\end{align*}
where $c$ is an absolute constant.
\end{lemma}

\begin{proof}[Proof of Lemma~\ref{lemma:generalization}]
Fix the depth, nonzero parameter locations, and activation types. By~\citet[Theorem 7]{Bartlett2019nearly}, the class of maps $(\vx,a)\mapsto \mathbf{1}\{h(\vx)>a\}$, with $h$ ranging over networks in $N(L,S)$ having this fixed structure, has VC-dimension at most
\[
D=\lceil CSL(L+\log S)\rceil\,.
\]
This applies to both ReQU and mixed ReLU--ReQU networks. Following the counting argument in~\citet[proof of Lemma 12]{ding2025Semi}, remove inactive units and relabel the remaining ones. Let $M$ count the possible choices of depth, nonzero parameter locations, and activation types. Since at most $S$ units remain,
\[
M\leq (L+1)2^{S}(S+1)\{(L+1)(S+d+1)^{2}\}^{S}\,,\qquad
\log M\leq CS\log\{(d+1)LS\}\,.
\]
By Sauer's lemma, shattering $m\geq D$ points requires
\[
2^{m}\leq M(em/D)^{D}\quad\Longrightarrow\quad
m\leq C(D+\log M)\,.
\]
Consequently,
\[
\mathrm{Pdim}(\scrH_t)\leq\mathrm{Pdim}\{N(L,S)\}
\leq CSL\big[L+\log\{(d+1)LS\}\big]\,.
\]
By~\citet[Lemma 13]{ding2025Semi}, each $D_{k}h_t$ is represented by a mixed ReLU--ReQU network of depth at most $c_2L$ with at most $c_3LS$ nonzero parameters. Applying the preceding bound gives
\begin{align*}
\mathrm{Pdim}(D_{k}\scrH_t)
&\leq Cc_{2}c_{3}L^2S\big[c_{2}L+\log\{c_{2}c_{3}(d+1)L^{2}S\}\big]\\
&\leq CSL^{2}\big[L+\log\{(d+1)LS\}\big]\,.
\end{align*}
Taking the maximum over $k$ completes the proof.
\end{proof}

\begin{customtheorem}{\ref{theorem:rate:guidance}}
Suppose Assumptions~\ref{assumption:bounded:support} and~\ref{assumption:bounded:weight} hold. Fix $t\in[0,T)$ and let $n$ be sufficiently large. Set
\[
\scrH_t=\left\{h_t\in N(L,S):
\begin{aligned}
&\underline {B}/2\leq h_t(\vx)\leq3\bar B/2\quad\text{for all }\vx\in\bbR^d\,,\\
&\max_{1\leq k\leq d}\|D_kh_t\|_{\infty}\leq4\bar {B}\sigma_{T-t}^{-2}
\end{aligned}\right\}\,,
\]
where $L=\lceil C_{d}\log n\rceil$ and $S=\lceil C_{d}n^{d/(d+8)}\rceil$, with $C_{d}$ depending only on $d$ and sufficiently large. Let $\hat h_t^\lambda$ be the gradient-regularized empirical risk minimizer in~\eqref{eq:erm:GP}, with $\lambda=n^{-2/(d+8)}$, and let $h_t^*$ be defined by~\eqref{eq:h:function}. Then
\[
\bbE\big\{\|\nabla\log\hat h_t^\lambda-\nabla\log h_t^*\|_{L^2(p_{T-t})}^2\big\}
\leq C\sigma_{T-t}^{-8}n^{-2/(d+8)}\log^{4} n\,,
\]
where $C$ depends only on $d$, $\bar B$, and $\underline B$, and the expectation is over the training sample.
\end{customtheorem}

\begin{proof}[Proof of Theorem~\ref{theorem:rate:guidance}]
Set $N=\lceil n^{1/(d+8)}\rceil$. For sufficiently large $n$, substituting Lemmas~\ref{lemma:approximation} and~\ref{lemma:generalization} into Lemma~\ref{lemma:oracle:inequality:guidance}, with envelope constants $\underline {B}/2$ and $2\bar {B}$, gives
\begin{align*}
&\bbE\big\{\|\hat h_t^\lambda-h_t^*\|_{L^2(p_{T-t})}^2\big\}
\leq C\frac{\log^4N}{\sigma_{T-t}^8N^4}
+C\lambda\frac{\log^2N}{\sigma_{T-t}^8N^2}\\
&~~~~~~~+C\bigg(\frac{N^d\log^2N\log n}{n}\bigg)^{1/2}
+\frac{C\lambda}{\sigma_{T-t}^4}\bigg(\frac{N^d\log^{3}N\log n}{n}\bigg)^{1/2}
+\frac{C\lambda^2}{\sigma_{T-t}^8}\,,
\end{align*}
where $C$ depends only on $d$, $\bar B$, and $\underline {B}$. Similarly,
\begin{align*}
&\bbE\big\{\|\nabla\hat h_t^\lambda-\nabla h_t^*\|_{L^2(p_{T-t})}^2\big\}
\leq\frac{C}{\lambda}\frac{\log^4N}{\sigma_{T-t}^8N^4}
+C\frac{\log^2N}{\sigma_{T-t}^8N^2}\\
&~~~~~~~+\frac{C}{\lambda}\bigg(\frac{N^d\log^2N\log n}{n}\bigg)^{1/2}
+\frac{C}{\sigma_{T-t}^4}\bigg(\frac{N^d\log^{3}N\log n}{n}\bigg)^{1/2}
+\frac{C\lambda}{\sigma_{T-t}^8}\,.
\end{align*}
Since $N\asymp n^{1/(d+8)}$, $\log N\leq\log n$, and $\lambda=n^{-2/(d+8)}$, we obtain
\begin{align}
\bbE\big\{\|\hat h_t^\lambda-h_t^*\|_{L^2(p_{T-t})}^2\big\}
&\leq C\sigma_{T-t}^{-8}n^{-4/(d+8)}\log^4n\,,\label{eq:theorem:rate:guidance:1}\\
\bbE\big\{\|\nabla\hat h_t^\lambda-\nabla h_t^*\|_{L^2(p_{T-t})}^2\big\}
&\leq C\sigma_{T-t}^{-8}n^{-2/(d+8)}\log^4n\,.\notag
\end{align}
Finally, using Lemmas~\ref{lemma:section:error:guidance:bound} and~\ref{lemma:section:error:guidance:regularity:2} and $\hat {h}_{t}^{\lambda}\geq\underline {B}/2$, we have
\begin{align*}
&\|\nabla\log\hat h_t^\lambda-\nabla\log h_t^*\|_{L^2(p_{T-t})}^2
=\bigg\|\frac{\nabla\hat h_t^\lambda-\nabla h_t^*}{\hat h_t^\lambda}
+\frac{(h_{t}^{*}-\hat {h}_{t}^{\lambda})\nabla h_t^*}{\hat h_t^\lambda h_t^*}\bigg\|_{L^2(p_{T-t})}^2\\
&~~~~~~~\leq\frac{8}{\underline B^2}\|\nabla\hat h_t^\lambda-\nabla h_t^*\|_{L^2(p_{T-t})}^2
+\frac{32d\bar B^2}{\underline B^{4}\sigma_{T-t}^4}\|\hat h_t^\lambda-h_t^*\|_{L^2(p_{T-t})}^2\,.
\end{align*}
Taking expectations and using~\eqref{eq:theorem:rate:guidance:1}, we obtain
\begin{align*}
\bbE\big\{\|\nabla\log\hat {h}_{t}^{\lambda}-\nabla\log h_{t}^{*}\|_{L^{2}(p_{T-t})}^{2}\big\}
&\leq C\big\{\sigma_{T-t}^{-8}n^{-2/(d+8)}+\sigma_{T-t}^{-12}n^{-4/(d+8)}\big\}\log^{4}n\\
&\leq C\sigma_{T-t}^{-8}n^{-2/(d+8)}\log^{4}n\,,
\end{align*}
where the last inequality holds because $n^{-2/(d+8)}\leq\sigma_{T-t}^{4}$ for sufficiently large $n$. This completes the proof.
\end{proof}

\section{Derivations in Section~\ref{section:rate:finetuning}}\label{appendix:rate:finetuning}

\subsection{Error decomposition of the controllable diffusion models}
\label{section:proof:error:decomposition}

\begin{customlemma}{\ref{lemma:error:decomposition}}
Suppose Assumptions~\ref{assumption:bounded:support},~\ref{assumption:bounded:weight}, and~\ref{assump:base:score:error} hold. Use the time grid defined in Section~\ref{section:preliminaries:inference}, with $T\geq1$ and $h\leq1/2$. Conditional on the training data, let $\what q_{T-T_0}$ be the density of the sampler~\eqref{eq:Zt:P:brownian:estimation}, whose Gaussian randomness is independent of these data. If the estimators $\hat {h}_{kh}^{\lambda}$ are positive $C^{1}$ functions and the risks below are finite, then
\begin{align*}
\kl(q_{T_{0}}\|\what{q}_{T-T_{0}})
&\lesssim \frac{\bar{B}}{\underline{B}}\sum_{k=0}^{K-1}h\bbE^{\bbP}\big\{|\nabla\log\hat{h}_{kh}^{\lambda}(\mX_{kh}^{\leftarrow})-\nabla\log h_{kh}^{*}(\mX_{kh}^{\leftarrow})|_{2}^{2}\big\} \\
&~~~~~~~~~~~~~\quad +\frac{\bar{B}}{\underline{B}}T\varepsilon_{\refer}^{2}+d\exp(-T)+\frac{d^{3}T^{2}}{\sigma_{T_{0}}^{8}K}\,,
\end{align*}
where the notation $\lesssim$ hides absolute constants.
\end{customlemma}

\begin{proof}[Proof of Lemma~\ref{lemma:error:decomposition}]
Applying~\citet[Proposition C.3]{Chen2023Improved} with the estimated score $\hat{\vs}+\nabla\log\hat {h}^{\lambda}$ and time parameters $2T,2T_{0},2h$, corresponding to their horizon, early-stopping time, and step size, respectively, and using the squared triangle inequality, we obtain
\begin{align}
\kl(q_{T_{0}}\|\what{q}_{T-T_{0}})
&\lesssim\underbrace{\kl(q_{T}\|\gamma_{d})}_{\text{initial error}}+\underbrace{\sum_{k=0}^{K-1}h\bbE^{\bbQ}\big\{|\hat{\vs}(kh,\mX_{kh}^{\leftarrow})-\nabla\log p_{T-kh}(\mX_{kh}^{\leftarrow})|_{2}^{2}\big\}}_{\text{reference score error}} \notag \\
&\quad+\underbrace{\sum_{k=0}^{K-1}h\bbE^{\bbQ}\big\{|\nabla\log\hat{h}_{kh}^{\lambda}(\mX_{kh}^{\leftarrow})-\nabla\log h_{kh}^{*}(\mX_{kh}^{\leftarrow})|_{2}^{2}\big\}}_{\text{Doob's guidance error}} \label{eq:lemma:error:kl:0} \\
&\quad+\underbrace{\sum_{k=0}^{K-1}\int_{kh}^{(k+1)h}\bbE^{\bbQ}\big\{|\nabla\log q_{T-t}(\mX_{t}^{\leftarrow})-\nabla\log q_{T-kh}(\mX_{kh}^{\leftarrow})|_{2}^{2}\big\}\dt}_{\text{discretization error}}\,,\notag
\end{align}
where $\gamma_{d}$ is the density of a standard Gaussian distribution $\calN(\bzero,\mI_{d})$.

\par\noindent{\em Step 1. Initial error in~\eqref{eq:lemma:error:kl:0}.}
Using~\citet[Lemma C.4]{Chen2023Improved}, we have 
\begin{equation}\label{eq:lemma:error:kl:1}
\kl(q_{T}\|\gamma_{d})\lesssim d\exp(-T)\,.
\end{equation}

\par\noindent{\em Step 2. Reference score error and Doob's guidance error in~\eqref{eq:lemma:error:kl:0}.}
Under Assumption~\ref{assumption:bounded:weight}, it is apparent that $\bbE^{\bbP}\{w(\mX_{T}^{\leftarrow})\}\geq\underline{B}$ and $\sup_{\vx\in\bbR^{d}}\{h^{*}(t,\vx)\}\leq\bar{B}$ for each $0\leq t\leq T-T_0$. Hence, the density ratio is uniformly bounded
\begin{equation}\label{eq:lemma:error:kl:2}
\sup_{\vx\in\bbR^{d}}\bigg\{\frac{q_{T-t}(\vx)}{p_{T-t}(\vx)}\bigg\}=\frac{\sup_{\vx\in\bbR^{d}}\{h^{*}(t,\vx)\}}{\bbE^{\bbP}\{w(\mX_{T}^{\leftarrow})\}}\leq\frac{\bar{B}}{\underline{B}}\,.
\end{equation}
For the reference score error term in~\eqref{eq:lemma:error:kl:0}, it follows for each $0\leq k\leq K-1$ that 
\begin{align*}
&\bbE^{\bbQ}\big\{|\hat{\vs}(kh,\mX_{kh}^{\leftarrow})-\nabla\log p_{T-kh}(\mX_{kh}^{\leftarrow})|_{2}^{2}\big\} \\
&~~~~~~~~~=\int|\hat{\vs}(kh,\vz)-\nabla\log p_{T-kh}(\vz)|_{2}^{2}q_{T-kh}(\vz)\d\vz \\
&~~~~~~~~~=\int|\hat{\vs}(kh,\vx)-\nabla\log p_{T-kh}(\vx)|_{2}^{2}\frac{q_{T-kh}(\vx)}{p_{T-kh}(\vx)}p_{T-kh}(\vx)\d\vx \\
&~~~~~~~~~\leq\frac{\bar{B}}{\underline{B}}\int|\hat{\vs}(kh,\vx)-\nabla\log p_{T-kh}(\vx)|_{2}^{2}p_{T-kh}(\vx)\d\vx \\
&~~~~~~~~~=\frac{\bar{B}}{\underline{B}}\bbE^{\bbP}\big\{|\hat{\vs}(kh,\mX_{kh}^{\leftarrow})-\nabla\log p_{T-kh}(\mX_{kh}^{\leftarrow})|_{2}^{2}\big\}\,,
\end{align*}
where the inequality holds from~\eqref{eq:lemma:error:kl:2} and H{\"o}lder's inequality. Consequently, 
\begin{align}
&\sum_{k=0}^{K-1}h\bbE^{\bbQ}\big\{|\hat{\vs}(kh,\mX_{kh}^{\leftarrow})-\nabla\log p_{T-kh}(\mX_{kh}^{\leftarrow})|_{2}^{2}\big\} \label{eq:lemma:error:kl:3} \\
&~~~~~~~~~\leq\frac{\bar{B}}{\underline{B}}\sum_{k=0}^{K-1}h\bbE^{\bbP}\big\{|\hat{\vs}(kh,\mX_{kh}^{\leftarrow})-\nabla\log p_{T-kh}(\mX_{kh}^{\leftarrow})|_{2}^{2}\big\}\leq T\frac{\bar{B}}{\underline{B}}\varepsilon_{\refer}^{2}\,, \notag
\end{align}
where the last inequality is owing to Assumption~\ref{assump:base:score:error}. By a similar argument, we have 
\begin{align}
&\sum_{k=0}^{K-1}h\bbE^{\bbQ}\big\{|\nabla\log\hat{h}_{kh}^{\lambda}(\mX_{kh}^{\leftarrow})-\nabla\log h_{kh}^{*}(\mX_{kh}^{\leftarrow})|_{2}^{2}\big\} \label{eq:lemma:error:kl:4} \\
&~~~~~~~~~\leq\frac{\bar{B}}{\underline{B}}\sum_{k=0}^{K-1}h\bbE^{\bbP}\big\{|\nabla\log\hat{h}_{kh}^{\lambda}(\mX_{kh}^{\leftarrow})-\nabla\log h_{kh}^{*}(\mX_{kh}^{\leftarrow})|_{2}^{2}\big\}\,. \notag
\end{align}

\par\noindent{\em Step 3. Discretization error in~\eqref{eq:lemma:error:kl:0}.}
Applying Lemmas~\ref{lemma:tweedie} and~\ref{lemma:grad:condition:expectation} to the forward process with initial distribution $q_{0}$, we obtain, for $T_{0}\leq s\leq T$,
\[
\nabla^{2}\log q_{s}(\vz)=-\frac{1}{\sigma_{s}^{2}}\mI_{d}+\frac{\mu_{s}^{2}}{\sigma_{s}^{4}}\cov(\mZ_{0}\mid\mZ_{s}=\vz)\,.
\]
Since $|\mZ_{0}|_{2}^{2}\leq d$ under Assumptions~\ref{assumption:bounded:support} and~\ref{assumption:bounded:weight},
\begin{equation}\label{eq:lemma:error:kl:5}
\sup_{T_{0}\leq s\leq T}\big[\sup_{\vz\in\bbR^{d}}\big\{|\nabla^{2}\log q_{s}(\vz)|_{\op}\big\}\big]\leq\frac{1}{\sigma_{T_{0}}^2}+\frac{d}{\sigma_{T_{0}}^{4}}\leq\frac{2d}{\sigma_{T_{0}}^{4}}\,.
\end{equation}
Thus the scores are uniformly $2d/\sigma_{T_{0}}^{4}$-Lipschitz on this interval. Applying~\citet[Lemma D.1]{Chen2023Improved} with the time rescaling $u=2s$ and $2h\leq1$ gives, for $kh\leq t\leq(k+1)h$,
\begin{equation}\label{eq:lemma:error:kl:6}
\bbE^{\bbQ}\big\{|\nabla\log q_{T-t}(\mX_{t}^{\leftarrow})-\nabla\log q_{T-kh}(\mX_{kh}^{\leftarrow})|_{2}^{2}\big\}\lesssim\frac{d^{3}}{\sigma_{T_0}^{8}}(t-kh)\,.
\end{equation}
Integrating over each interval and summing over $k$, we obtain
\begin{equation}\label{eq:lemma:error:kl:7}
\sum_{k=0}^{K-1}\int_{kh}^{(k+1)h}\bbE^{\bbQ}\big\{|\nabla\log q_{T-t}(\mX_t^{\leftarrow})-\nabla\log q_{T-kh}(\mX_{kh}^{\leftarrow})|_2^2\big\}\dt\lesssim\frac{d^{3}Kh^2}{\sigma_{T_{0}}^{8}}\leq\frac{d^{3}T^2}{\sigma_{T_0}^{8}K}\,.
\end{equation}

\par\noindent{\em Step 4. Conclusions.}
Substituting~\eqref{eq:lemma:error:kl:1},~\eqref{eq:lemma:error:kl:3},~\eqref{eq:lemma:error:kl:4},~\eqref{eq:lemma:error:kl:7} into~\eqref{eq:lemma:error:kl:0} completes the proof.
\end{proof}

\begin{corollary}\label{corollary:error:kl}
Suppose Assumptions~\ref{assumption:bounded:support},~\ref{assumption:bounded:weight}, and~\ref{assump:base:score:error} hold. Let $0<T_{0}\leq1/2$ and $0<\delta\leq e^{-1}\sigma_{T_0}^{4}$. Set the hypothesis classes $\{\scrH_{kh}\}_{k=0}^{K-1}$ as~\eqref{eq:theorem:hypothesis}, with $L,S$, and $\lambda$ chosen as in Theorem~\ref{theorem:rate:guidance}, and use the estimators~\eqref{eq:erm:GP}. Let $\what q_{T-T_0}$ be the marginal density of $\what{\mZ}_{T-T_0}^{\leftarrow}$ defined in~\eqref{eq:Zt:P:brownian:estimation}. Then it follows that
\[
\bbE\big\{\mathrm{TV}(q_{T_0},\what q_{T-T_0})^2\big\}\leq\frac{C\delta^2}{\sigma_{T_0}^8}\log^{4}\!\left\{\frac{e}{\delta}\log\!\left(\frac{\sigma_{T_0}^8}{\delta^2}\right)\right\}\,,
\]
where $C$ is a constant depending only on $d,\bar B$, and $\underline B$, the expectation is over the training sample, and
\begin{align*}
&T=\log\!\left(\frac{\sigma_{T_0}^8}{\delta^2}\right)\,,\qquad K\gtrsim\delta^{-2}\log^2\!\left(\frac{\sigma_{T_0}^8}{\delta^2}\right)\,,\qquad h=\frac{T-T_0}{K}\,,\\
&\varepsilon_{\refer}^2\lesssim\frac{\delta^2}{\sigma_{T_0}^8}\log^{-1}\!\left(\frac{\sigma_{T_0}^8}{\delta^2}\right)\,,\qquad n\gtrsim\delta^{-(d+8)}\log^{(d+8)/2}\!\left(\frac{\sigma_{T_0}^8}{\delta^2}\right)\,.
\end{align*}
The constants in the lower bounds for $K$ and $n$ are chosen sufficiently large.
\end{corollary}

\begin{proof}[Proof of Corollary~\ref{corollary:error:kl}]
The stated sample-size bound ensures that Theorem~\ref{theorem:rate:guidance} applies at every grid time.
Combining Theorem~\ref{theorem:rate:guidance} and Lemma~\ref{lemma:error:decomposition} and taking expectation over the training sample yields
\[
\begin{aligned}
\bbE\big\{\kl(q_{T_0}\|\what q_{T-T_0})\big\}\leq\frac C{\sigma_{T_0}^8}\Big\{
&\underbrace{Tn^{-2/(d+8)}\log^4 n}_{\mathrm{(i)}}+\underbrace{T\sigma_{T_0}^8\varepsilon_{\refer}^2}_{\mathrm{(ii)}}+\underbrace{\sigma_{T_0}^8e^{-T}}_{\mathrm{(iii)}}+\underbrace{T^2/K}_{\mathrm{(iv)}}\Big\}\,.
\end{aligned}
\]
The stated parameter choices ensure $T\geq2>T_0$ and $h\leq1/2$. They also give
\[
\mathrm{(i)}\lesssim\delta^2\log^{4}(eT/\delta)\,,\qquad\mathrm{(ii)}+\mathrm{(iii)}+\mathrm{(iv)}\lesssim\delta^2\,.
\]
The bound for (i) holds for all $n$ above the stated threshold, since $n^{-2/(d+8)}\log^{4}n$ is decreasing for $\log n\geq2(d+8)$. Thus we find
\[
\bbE\big\{\kl(q_{T_0}\|\what q_{T-T_0})\big\}\leq\frac{C\delta^2}{\sigma_{T_0}^8}\log^{4}\!\left\{\frac{e}{\delta}\log\!\left(\frac{\sigma_{T_0}^8}{\delta^2}\right)\right\}\,.
\]
Finally, applying Pinsker's inequality and taking expectation completes the proof.
\end{proof}

\subsection{Convergence rate of the controllable diffusion models}
\label{section:proof:rate}

\begin{customtheorem}{\ref{theorem:rate:controllable:diffusion}}
Suppose Assumptions~\ref{assumption:bounded:support},~\ref{assumption:bounded:weight}, and~\ref{assump:base:score:error} hold. Let $0<\varepsilon\leq e^{-1}$. Set the hypothesis classes $\{\scrH_{kh}\}_{k=0}^{K-1}$ as~\eqref{eq:theorem:hypothesis} with the same depth $L$ and number of non-zero parameters $S$ as Theorem~\ref{theorem:rate:guidance}, with $\lambda=n^{-2/(d+8)}$ and the per-grid minimizers~\eqref{eq:erm:GP}. Let $\what{q}_{T-T_{0}}$ be the marginal density of $\what{\mZ}_{T-T_{0}}^{\leftarrow}$ defined in~\eqref{eq:Zt:P:brownian:estimation}, and let $(\calM\circ\calT_{R})_{\sharp}\what{q}_{T-T_{0}}$ defined as~\eqref{eq:trancation:scaling}. Then it follows that 
\begin{equation}
\bbE\big[ W_{2}^{2}\{q_{0},(\calM\circ\calT_{R})_{\sharp}\what{q}_{T-T_{0}}\}\big]\leq C\varepsilon\log^{3}\bigg(\frac{1}{\varepsilon}\bigg)\,,
\end{equation}
provided that the truncation radius $R$, the terminal time $T$, the step size $h$, the number of steps $K$, the error of reference score $\varepsilon_{\refer}$, the number of samples $n$ for Doob's matching, and the early-stopping time $T_{0}$ are set, respectively, as 
\begin{align*}
&T=\log(\varepsilon^{-2})\,,\qquad T_{0}=-\tfrac{1}{2}\log(1-\varepsilon)\,,\qquad
R^{2}=2(4d\mu_{T_{0}}^{2}+8\sigma_{T_{0}}^{2})\log(2^{d+1}/\varepsilon)\,,\\
&K\gtrsim\varepsilon^{-6}\log^2(\varepsilon^{-1})\,,\qquad h=(T-T_{0})/K\,,\\
&\varepsilon_{\refer}^2\lesssim\varepsilon^2/\log(\varepsilon^{-1})\,,\qquad
n\gtrsim\varepsilon^{-3(d+8)}\log^{(d+8)/2}(\varepsilon^{-1})\,.
\end{align*}
Here $C$ and the constants in the sufficient lower bounds depend only on $d$, $\bar B$ and $\underline B$. The lower bound for $K$ is chosen large enough to ensure $h\leq1/2$. The expectation is over the training sample, and the sampler law is conditional on that sample.
\end{customtheorem} 

\begin{proof}[Proof of Theorem~\ref{theorem:rate:controllable:diffusion}]
According to the triangular inequality, we have 
\begin{align}
 W_{2}^{2}\{q_{0},(\calM\circ\calT_{R})_{\sharp}\what{q}_{T-T_{0}}\}
&\leq3\underbrace{ W_{2}^{2}(q_{0},\calM_{\sharp}q_{T_{0}})}_{\text{(i)}}+3\underbrace{ W_{2}^{2}\{\calM_{\sharp}q_{T_{0}},(\calM\circ\calT_{R})_{\sharp}q_{T_{0}}\}}_{\text{(ii)}} \notag \\
&\quad+3\underbrace{ W_{2}^{2}\{(\calM\circ\calT_{R})_{\sharp}q_{T_{0}},(\calM\circ\calT_{R})_{\sharp}\what{q}_{T-T_{0}}\}}_{\text{(iii)}}\,. \label{eq:theorem:rate:controllable:diffusion:1}
\end{align}
Here the term (i) represents the early-stopping error, the term (ii) represents the truncation error, while the term (iii) represents the error of controllable diffusion models~\eqref{eq:Zt:P:brownian:estimation}. In the rest of the proof, we bound these three errors, respectively. 

\par\noindent{\em Step 1. Bound the term (i) in~\eqref{eq:theorem:rate:controllable:diffusion:1}.}
To estimate the Wasserstein-2 distance between the target distribution $q_{0}$ and the scaled early-stopping distribution $\calM_{\sharp}q_{T_{0}}$, we begin by producing a coupling of them. Let $\mZ_{0}\sim q_{0}$, and let $\vepsilon\sim\calN(\bzero,\mI_{d})$ be independent of $\mZ_{0}$. Define $\widetilde{\mZ}_{T_{0}}\coloneq\mZ_{0}+\sigma_{T_{0}}\mu_{T_{0}}^{-1}\vepsilon$. It is apparent that $\widetilde{\mZ}_{T_{0}}\sim\calM_{\sharp}q_{T_{0}}$. Then
\begin{equation}\label{eq:theorem:rate:controllable:diffusion:2}
 W_{2}^{2}(q_{0},\calM_{\sharp}q_{T_{0}})\leq\bbE\big(|\mZ_{0}-\widetilde{\mZ}_{T_{0}}|_{2}^{2}\big)=\frac{\sigma_{T_{0}}^{2}}{\mu_{T_{0}}^{2}}\bbE\big(|\vepsilon|_{2}^{2}\big)=\frac{d\sigma_{T_{0}}^{2}}{\mu_{T_{0}}^{2}}\,.
\end{equation}

\par\noindent{\em Step 2. Bound the term (ii) in~\eqref{eq:theorem:rate:controllable:diffusion:1}.}
Let $\mZ_{T_{0}}\sim q_{T_{0}}$. According to the definition of the truncation operator $\calT_{R}$, the joint law of $(\mu_{T_{0}}^{-1}\mZ_{T_{0}},\mu_{T_{0}}^{-1}\mZ_{T_{0}}\bbone_{B(\bzero,R)}(\mZ_{T_{0}}))$ is a coupling of $\calM_{\sharp}q_{T_{0}}$ and $(\calM\circ\calT_{R})_{\sharp}q_{T_{0}}$. Therefore, 
\begin{align*}
& W_{2}^{2}\{\calM_{\sharp}q_{T_{0}},(\calM\circ\calT_{R})_{\sharp}q_{T_{0}}\}
\leq\bbE\big\{|\mu_{T_{0}}^{-1}\mZ_{T_{0}}-\mu_{T_{0}}^{-1}\mZ_{T_{0}}\bbone_{B(\bzero,R)}(\mZ_{T_{0}})|_{2}^{2}\big\} \\
&~~~~=\frac{1}{\mu_{T_{0}}^{2}}\int|\vz-\vz\bbone_{B(\bzero,R)}(\vz)|_{2}^{2}q_{T_{0}}(\vz)\d\vz =\frac{1}{\mu_{T_{0}}^{2}}\int|\vz|_{2}^{2}\bbone_{\bbR^{d}-B(\bzero,R)}(\vz)q_{T_{0}}(\vz)\d\vz \\
&~~~~\leq\frac{1}{\mu_{T_{0}}^{2}}\bbE^{1/2}\big(|\mZ_{T_{0}}|_{2}^{4}\big)\pr^{1/2}\{|\mZ_{T_{0}}|_{2}> R\} \lesssim \frac{1}{\mu_{T_{0}}^{2}}d 2^{(d+1)/2}\exp\bigg\{-\frac{R^{2}}{2(4d\mu_{T_{0}}^{2}+8\sigma_{T_{0}}^{2})}\bigg\}\,,
\end{align*}
where the second ineq holds from Cauchy-Schwarz inequality, and the last inequality is due to Lemma~\ref{lemma:expectation:t} and Corollary~\ref{corollary:tail:proba:t}. With the stated radius $R^2=2(4d\mu_{T_0}^2+8\sigma_{T_0}^2)\log(2^{d+1}/\varepsilon)$, the exponential dimension factor is absorbed, giving 
\begin{equation}\label{eq:theorem:rate:controllable:diffusion:3}
 W_{2}^{2}\{\calM_{\sharp}q_{T_{0}},(\calM\circ\calT_{R})_{\sharp}q_{T_{0}}\}\lesssim\frac{d}{\mu_{T_0}^2}\varepsilon\,.
\end{equation}

\par\noindent{\em Step 3. Bound the term (iii) in~\eqref{eq:theorem:rate:controllable:diffusion:1}.}
Let $\mZ_{T_{0}}^{R}\sim (\calT_{R})_{\sharp}q_{T_{0}}$ and $\what{\mZ}_{T_{0}}^{R}\sim (\calT_{R})_{\sharp}\what{q}_{T-T_{0}}$ be optimal coupled. This means 
\begin{equation}\label{eq:theorem:rate:controllable:diffusion:4}
 W_{2}^{2}\{(\calT_{R})_{\sharp}q_{T_{0}},(\calT_{R})_{\sharp}\what{q}_{T-T_{0}}\}=\bbE\big(|\mZ_{T_{0}}^{R}-\what{\mZ}_{T_{0}}^{R}|_{2}^{2}\big)\,.
\end{equation}
On the other hand, $\mu_{T_{0}}^{-1}\mZ_{T_{0}}^{R}\sim(\calM\circ\calT_{R})_{\sharp}q_{T_{0}}$ and $\mu_{T_{0}}^{-1}\what{\mZ}_{T_{0}}^{R}\sim(\calM\circ\calT_{R})_{\sharp}\what{q}_{T-T_{0}}$. Hence,
\begin{align}
& W_{2}^{2}\{(\calM\circ\calT_{R})_{\sharp}q_{T_{0}},(\calM\circ\calT_{R})_{\sharp}\what{q}_{T-T_{0}}\} \label{eq:theorem:rate:controllable:diffusion:5} \\
&~~~~~~~~\leq\bbE\big(|\mu_{T_{0}}^{-1}\mZ_{T_{0}}^{R}-\mu_{T_{0}}^{-1}\what{\mZ}_{T_{0}}^{R}|_{2}^{2}\big)=\frac{1}{\mu_{T_{0}}^{2}} W_{2}^{2}\{(\calT_{R})_{\sharp}q_{T_{0}},(\calT_{R})_{\sharp}\what{q}_{T-T_{0}}\}\,, \notag
\end{align}
where the equality holds from~\eqref{eq:theorem:rate:controllable:diffusion:4}. For two probability measures supported in $B(0,R)$, a maximal coupling gives $W_{2}^{2}\leq4R^{2}\mathrm{TV}(\cdot,\cdot)$, since the squared distance is at most $4R^{2}$. Together with the contraction of total variation under $\mathcal {T}_{R}$, this yields
\begin{align}
 W_{2}^{2}\{(\calT_{R})_{\sharp}q_{T_{0}},(\calT_{R})_{\sharp}\what{q}_{T-T_{0}}\}
&\leq4R^{2}\mathrm{TV}\{(\calT_{R})_{\sharp}q_{T_{0}},(\calT_{R})_{\sharp}\what{q}_{T-T_{0}}\} \notag \\
&\leq4R^{2}\mathrm{TV}(q_{T_{0}},\what{q}_{T-T_{0}})\,. \label{eq:theorem:rate:controllable:diffusion:6}
\end{align}
Combining~\eqref{eq:theorem:rate:controllable:diffusion:5} and~\eqref{eq:theorem:rate:controllable:diffusion:6} and taking expectation over the training sample, we obtain
\begin{align}
&\bbE\big[\calW_2^2\big\{(\calM\circ\calT_R)_\sharp q_{T_0},(\calM\circ\calT_R)_\sharp\what q_{T-T_0}\big\}\big]
\leq\frac{4R^2}{\mu_{T_0}^2}\bbE\big\{\mathrm{TV}(q_{T_0},\what q_{T-T_0})\big\}\notag\\
&~~~~~~~~~~~\leq\frac{4R^2}{\mu_{T_0}^2}\big[\bbE\big\{\mathrm{TV}(q_{T_0},\what {q}_{T-T_0})^2\big\}\big]^{1/2}
\leq\frac{CR^2}{\mu_{T_0}^2}\frac{\varepsilon^3}{\sigma_{T_0}^4}\log^2(\varepsilon^{-1})\,.
\label{eq:theorem:rate:controllable:diffusion:7}
\end{align}
The last inequality follows from Corollary~\ref{corollary:error:kl} with $\delta=\varepsilon^3$ and $\sigma_{T_{0}}^{2}=\varepsilon$. Its range condition holds because $\varepsilon\leq e^{-1}$, and its prescribed horizon is $T=\log(\varepsilon^{-2})$.

\par\noindent{\em Step 4. Conclusion.}
Taking expectations in~\eqref{eq:theorem:rate:controllable:diffusion:1} and inserting~\eqref{eq:theorem:rate:controllable:diffusion:2},~\eqref{eq:theorem:rate:controllable:diffusion:3}, and~\eqref{eq:theorem:rate:controllable:diffusion:7}, we obtain
\[
\bbE\big[\calW_2^2\big\{q_0,(\calM\circ\calT_R)_\sharp\what q_{T-T_0}\big\}\big]
\leq C\bigg\{\sigma_{T_0}^2+\varepsilon+\frac{\varepsilon^3}{\sigma_{T_0}^4}\log^3(\varepsilon^{-1})\bigg\}\,,
\]
where $C$ depends only on $d$, $\bar B$, and $\underline B$. Here we used $\mu_{T_{0}}^{2}=1-\varepsilon\geq1-e^{-1}$ and $R^{2}\leq C_{d}\log(\varepsilon^{-1})$. Substituting $\sigma_{T_0}^2=\varepsilon$ gives the claimed bound $C\varepsilon\log^{3}(\varepsilon^{-1})$, completing the proof.
\end{proof}

\section{Derivations in Section~\ref{section:rate:low}}\label{appendix:rate:low}

\begin{customproposition}{\ref{proposition:h:function:low}}
Suppose Assumptions~\ref{assumption:bounded:weight} and~\ref{assumption:intrinsic} hold. Then for any $t\in[0,T)$ and $\vx\in\bbR^d$, we have
\[
h_t^*(\vx)=\bar h_t^*(\mP^\top\vx)\coloneq
\bbE\big\{w(\mP\bar{\mX}_T^\leftarrow)\given\bar{\mX}_t^\leftarrow=\mP^\top\vx\big\}\,.
\]
\end{customproposition}

\begin{proof}[Proof of Proposition~\ref{proposition:h:function:low}]
By Assumption~\ref{assumption:intrinsic}, $\mX_0\stackrel{\d}{=}\mP\bar{\mX}_0$ with $\bar{\mX}_0\sim\bar p_0$. For $0<t\leq T$, the Gaussian transition density therefore gives
\[
p_t(\vx)=\int\varphi_{d}(\vx;\mu_{t}\mP\bar{\vx}_{0},\sigma_{t}^{2}\mI_{d})\bar {p}_{0}(\bar{\vx}_{0})\d\bar{\vx}_{0}\,.
\]
The orthogonal decomposition $\vx=\mP\mP^{\top}\vx+(\mI_{d}-\mP\mP^{\top})\vx$ and $\mP^{\top}\mP=\mI_{d^{*}}$ give
\[
|\vx-\mu_{t}\mP\bar{\vx}_{0}|_{2}^{2}
=|(\mI_{d}-\mP\mP^{\top})\vx|_{2}^{2}+|\mP^{\top}\vx-\mu_{t}\bar{\vx}_{0}|_{2}^{2}\,.
\]
Thus the Gaussian kernel factors as
\begin{align*}
\varphi_{d}(\vx;\mu_{t}\mP\bar{\vx}_{0},\sigma_{t}^{2}\mI_{d})
&=(2\pi\sigma_{t}^{2})^{-(d-d^{*})/2}\exp\bigg\{-\frac{|(\mI_{d}-\mP\mP^{\top})\vx|_{2}^{2}}{2\sigma_{t}^{2}}\bigg\}\\
&~~~~~~~\times\varphi_{d^{*}}(\mP^{\top}\vx;\mu_{t}\bar{\vx}_{0},\sigma_{t}^{2}\mI_{d^{*}})\,.
\end{align*}
Integrating with respect to $\bar p_{0}$ and using~\eqref{eq:forward:solution:low}, we obtain
\begin{equation}\label{eq:low:dim:density}
p_t(\vx)=(2\pi\sigma_t^2)^{-(d-d^{*})/2}
\exp\bigg\{-\frac{|(\mI_d-\mP\mP^\top)\vx|_2^2}{2\sigma_t^2}\bigg\}\bar p_t(\mP^\top\vx)\,.
\end{equation}
Dividing the kernel factorization by~\eqref{eq:low:dim:density} cancels their common Gaussian factor:
\[
\frac{\varphi_d(\vx;\mu_t\mP\bar\vx_0,\sigma_t^2\mI_d)}{p_t(\vx)}
=\frac{\varphi_{d^{*}}(\mP^\top\vx;\mu_t\bar\vx_0,\sigma_t^2\mI_{d^{*}})}{\bar {p}_{t}(\mP^{\top}\vx)}\,.
\]
Bayes' rule therefore gives
\begin{align*}
&h_{T-t}^{*}(\vx)=\bbE\big\{w(\mX_{0})\given\mX_{t}=\vx\big\}\\
&~~~~~=\int w(\mP\bar{\vx}_{0})\frac{\varphi_{d}(\vx;\mu_{t}\mP\bar{\vx}_{0},\sigma_{t}^{2}\mI_{d})}{p_{t}(\vx)}\bar p_0(\bar\vx_0)\d\bar\vx_0\\
&~~~~~=\int w(\mP\bar\vx_0)\frac{\varphi_{d^{*}}(\mP^\top\vx;\mu_t\bar\vx_0,\sigma_t^2\mI_{d^{*}})}{\bar {p}_{t}(\mP^{\top}\vx)}\bar p_0(\bar\vx_0)\d\bar\vx_0\\
&~~~~~=\bbE\big\{w(\mP\bar{\mX}_0)\given\bar{\mX}_t=\mP^\top\vx\big\}=\bbE\big\{w(\mP\bar{\mX}_T^\leftarrow)\given\bar{\mX}_{T-t}^\leftarrow=\mP^\top\vx\big\}
=\bar {h}_{T-t}^{*}(\mP^{\top}\vx)\,.
\end{align*}
The last line uses the time-reversal relation~\eqref{eq:base:reversal:low}. Replacing $T-t$ by $t$ proves the claim.
\end{proof}

\begin{customproposition}{\ref{proposition:regularity:h:function:low}}
Suppose Assumptions~\ref{assumption:bounded:weight} and~\ref{assumption:intrinsic} hold. Then for all $t\in[0,T)$ and $\bar{\vx}\in\bbR^{d^{*}}$, the following bounds hold:
\[
\underline B\leq\bar h_t^*(\bar\vx)\leq\bar B\,,\quad
\max_{1\leq k\leq d^{*}}|D_k\bar h_t^*(\bar\vx)|\leq2\sigma_{T-t}^{-2}\bar B\,,\quad
\max_{1\leq k,\ell\leq d^{*}}|D_{k\ell}^2\bar h_t^*(\bar\vx)|\leq6\sigma_{T-t}^{-4}\bar B\,,
\]
where $D_{k}$ and $D_{k\ell}^{2}$ denote the first-order and second-order partial derivatives with respect to the input coordinates, respectively.    
\end{customproposition}

\begin{proof}[Proof of Proposition~\ref{proposition:regularity:h:function:low}]
Assumptions~\ref{assumption:bounded:weight} and~\ref{assumption:intrinsic} give
\[
\supp(\bar {p}_{0})\subset[-1,1]^{d^{*}}\,,\qquad
\underline {B}\leq w(\mP\bar{\vx})\leq\bar {B}\quad\text{for }\bar{\vx}\in\supp(\bar {p}_{0})\,.
\]
Applying Lemmas~\ref{lemma:section:error:guidance:bound},~\ref{lemma:section:error:guidance:regularity:2}, and~\ref{lemma:section:error:guidance:regularity:3} in dimension $d^{*}$ to the reference law $\bar {p}_{0}$ and weight $w(\mP\,\cdot)$ proves the result.
\end{proof}

\par By a similar argument as Lemma~\ref{lemma:approximation}, we have the following approximation error bounds for low-dimensional Doob's $h$-function $\bar{h}_{t}^{*}$~\eqref{eq:doob:h:function:low}.

\begin{lemma}[Approximation error]\label{lemma:approximation:low}
Suppose Assumptions~\ref{assumption:bounded:weight} and~\ref{assumption:intrinsic} hold. Fix $t\in[0,T)$ and let $N$ be a sufficiently large integer. Set $L=\lceil C_{d^{*}}\log N\rceil$, $S=\lceil C_{d^{*}}(N^{d^*}+dd^{*})\rceil$, and
\[
\scrH_{t}=\left\{h_{t}\in N(L,S):\begin{aligned}
&\underline {B}/2\leq h_{t}(\vx)\leq3\bar {B}/2\quad\text{for all }\vx\in\bbR^{d}\,,\\
&\max_{1\leq k\leq d}\|D_{k}h_{t}\|_{\infty}\leq4\sqrt{d^{*}}\,\bar {B}\sigma_{T-t}^{-2}
\end{aligned}\right\}\,,
\]
where $C_{d^{*}}$ depends only on $d^{*}$ and is sufficiently large. Then there exists $h_{t}\in\scrH_{t}$ such that
\[
\|h_{t}-h_{t}^{*}\|_{L^{2}(p_{T-t})}^{2}\leq\frac{C\bar {B}^{2}\log^{4}N}{\sigma_{T-t}^{8}N^{4}}\,,\qquad
\|\nabla h_{t}-\nabla h_{t}^{*}\|_{L^{2}(p_{T-t})}^{2}\leq\frac{C\bar {B}^{2}\log^{2}N}{\sigma_{T-t}^{8}N^{2}}\,,
\]
where $C$ depends only on $d^{*}$.
\end{lemma}

\begin{proof}[Proof of Lemma~\ref{lemma:approximation:low}]
By Proposition~\ref{proposition:h:function:low},
\[
h_{t}^{*}(\vx)=\bar {h}_{t}^{*}(\mP^{\top}\vx)\,,\qquad
\bar {h}_{t}^{*}(\bar{\vx})=\bbE\big\{w(\mP\bar{\mX}_{0})\given\bar{\mX}_{T-t}=\bar{\vx}\big\}\,.
\]
By Proposition~\ref{proposition:regularity:h:function:low} and the support condition in Assumption~\ref{assumption:intrinsic}, the proof of Lemma~\ref{lemma:approximation} applies to $\bar {h}_{t}^{*}$ in dimension $d^{*}$. Thus, for sufficiently large $N$, there exists a network $\bar {h}_{t}$ satisfying
\[
\underline {B}/2\leq\bar {h}_{t}\leq3\bar {B}/2\,,\qquad
\max_{1\leq j\leq d^{*}}\|D_{j}\bar {h}_{t}\|_{\infty}\leq4\bar {B}\sigma_{T-t}^{-2}\,,
\]
and
\[
\|\bar {h}_{t}-\bar {h}_{t}^{*}\|_{L^{2}(\bar {p}_{T-t})}^{2}\leq\frac{C\bar {B}^{2}\log^{4}N}{\sigma_{T-t}^{8}N^{4}}\,,\qquad
\|\nabla\bar {h}_{t}-\nabla\bar {h}_{t}^{*}\|_{L^{2}(\bar {p}_{T-t})}^{2}\leq\frac{C\bar {B}^{2}\log^{2}N}{\sigma_{T-t}^{8}N^{2}}\,.
\]

The network constructed in the proof of Lemma~\ref{lemma:approximation} has the smooth input-clipping form of Lemma~\ref{lemma:approximation:compact}, with local network $g_{t}$, clipping function $\phi$, and radius $R$. Define
\[
h_{t}(\vx)=\bar {h}_{t}(\mP^{\top}\vx)
=g_{t}\bigl[\phi\{(\mP^{\top}\vx)_{1}/R\},\ldots,\phi\{(\mP^{\top}\vx)_{d^{*}}/R\}\bigr]\,.
\]
The projection changes only the input weights of the $6d^{*}$ clipping units. Each has at most $d$ weights and one bias, so
\[
L_{h}\leq L_{g}+1\leq C_{d^{*}}\log N\,,\qquad S_{h}\leq6S_{g}+6(d+1)d^{*}\leq C_{d^{*}}(N^{d^{*}}+dd^{*})\,,
\]
where $L_{g},S_{g}$ are the depth and number of nonzero parameters of $g_{t}$.
The range is preserved, and the chain rule gives
\[
|D_{k}h_{t}(\vx)|\leq|\mP^{\top}\ve_{k}|_{2}|\nabla\bar {h}_{t}(\mP^{\top}\vx)|_{2}
\leq4\sqrt{d^{*}}\,\bar {B}\sigma_{T-t}^{-2}\,,
\]
\[
\nabla h_{t}(\vx)-\nabla h_{t}^{*}(\vx)
=\mP\big\{\nabla\bar {h}_{t}(\mP^{\top}\vx)-\nabla\bar {h}_{t}^{*}(\mP^{\top}\vx)\big\}\,.
\]
Hence $h_{t}\in\scrH_{t}$. Since $\mP^{\top}\mX_{T-t}\sim\bar {p}_{T-t}$ and $|\mP\vv|_{2}=|\vv|_{2}$, we have
\begin{align*}
\|h_t-h_t^*\|_{L^2(p_{T-t})}^2
&=\|\bar {h}_{t}-\bar {h}_{t}^{*}\|_{L^2(\bar {p}_{T-t})}^{2}\,,\\
\|\nabla h_t-\nabla h_t^*\|_{L^2(p_{T-t})}^2
&=\|\nabla\bar {h}_{t}-\nabla\bar {h}_{t}^{*}\|_{L^2(\bar {p}_{T-t})}^2\,.
\end{align*}
The claimed bounds follow.
\end{proof}

\begin{customtheorem}{\ref{theorem:rate:guidance:low}}
Suppose Assumptions~\ref{assumption:bounded:weight} and~\ref{assumption:intrinsic} hold. Fix $t\in[0,T)$ and let $n$ be sufficiently large. Set
\[
\scrH_t=\left\{h_t\in N(L,S):\begin{aligned}
&\underline {B}/2\leq h_t(\vx)\leq3\bar B/2\quad\text{for all }\vx\in\bbR^d\,,\\
&\max_{1\leq k\leq d}\|D_kh_t\|_{\infty}\leq4\sqrt{d^{*}}\,\bar {B}\sigma_{T-t}^{-2}
\end{aligned}\right\}\,,
\]
where $L=\lceil C_{d^{*}}\log n\rceil$ and $S=\lceil C_{d^{*}}\{n^{d^*/(d^*+8)}+dd^{*}\}\rceil$, with $C_{d^{*}}$ depending only on $d^{*}$ and sufficiently large. Let $\hat h_t^\lambda$ be the gradient-regularized empirical risk minimizer in~\eqref{eq:erm:GP}, with $\lambda=n^{-2/(d^{*}+8)}$, and let $h_t^*$ be defined by~\eqref{eq:h:function}. Then
\[
\bbE\big\{\|\nabla\log\hat h_t^\lambda-\nabla\log h_t^*\|_{L^2(p_{T-t})}^2\big\}
\leq Cd^{2}\sigma_{T-t}^{-8}n^{-2/(d^*+8)}\log^{4} n\,,
\]
where $C$ depends only on $d^*$, $\bar B$, and $\underline B$, and the expectation is over the training sample.
\end{customtheorem}

\begin{proof}[Proof of Theorem~\ref{theorem:rate:guidance:low}]
Set $N=\lceil n^{1/(d^{*}+8)}\rceil$. Throughout the proof, $C$ depends only on $d^{*}$, $\bar {B}$, and $\underline {B}$.
By Proposition~\ref{proposition:h:function:low} and $\mP^{\top}\mP=\mI_{d^{*}}$, the chain rule gives
\[
\nabla h_{t}^{*}(\vx)=\mP\nabla\bar {h}_{t}^{*}(\mP^{\top}\vx)\,,\qquad
\Delta h_{t}^{*}(\vx)=\operatorname{tr}\big\{\mP\nabla^{2}\bar {h}_{t}^{*}(\mP^{\top}\vx)\mP^{\top}\big\}
=\Delta\bar {h}_{t}^{*}(\mP^{\top}\vx)\,.
\]
Differentiating the logarithm of~\eqref{eq:low:dim:density} yields
\[
\nabla\log p_{T-t}(\vx)
=-\frac{(\mI_{d}-\mP\mP^{\top})\vx}{\sigma_{T-t}^{2}}
+\mP\nabla\log\bar {p}_{T-t}(\mP^{\top}\vx)\,.
\]
Since $\mP^{\top}(\mI_{d}-\mP\mP^{\top})=0$, the first term is orthogonal to $\nabla h_{t}^{*}(\vx)$. Consequently,
\[
\Delta h_{t}^{*}+\nabla h_{t}^{*}\cdot\nabla\log p_{T-t}
=\big(\Delta\bar {h}_{t}^{*}+\nabla\bar {h}_{t}^{*}\cdot\nabla\log\bar {p}_{T-t}\big)\circ\mP^{\top}\,.
\]
Since $\mP^{\top}\mX_{T-t}\sim\bar {p}_{T-t}$, Lemma~\ref{lemma:reg:gap:bound} in dimension $d^{*}$ gives
\[
\|\Delta h_{t}^{*}+\nabla h_{t}^{*}\cdot\nabla\log p_{T-t}\|_{L^{2}(p_{T-t})}
\leq12d^{*}\bar {B}\sigma_{T-t}^{-4}\,.
\]
Thus the regularization-gap argument in Lemma~\ref{lemma:oracle:inequality:guidance} applies with $d$ replaced by $d^{*}$ in its bias term.

The network budgets satisfy $S\leq C d n^{d^{*}/(d^{*}+8)}$ and $L\leq C\log n$. Lemma~\ref{lemma:generalization} and $1+\log d\leq d$ yield
\begin{align*}
\big\{\mathrm{Pdim}(\scrH_{t})\log n/n\big\}^{1/2}
&\leq Cd\,n^{-4/(d^{*}+8)}\log^{3/2}n\,,\\
\big\{\mathrm{Pdim}(\nabla\scrH_{t})\log n/n\big\}^{1/2}
&\leq Cd\,n^{-4/(d^{*}+8)}\log^{2}n\,.
\end{align*}
In the proof of Lemma~\ref{lemma:oracle:inequality:guidance}, the coordinate derivative bound $4\sqrt{d^{*}}\,\bar {B}\sigma_{T-t}^{-2}$ gives the value oracle inequality a gradient generalization coefficient $C\lambda dd^{*}\bar {B}^{2}\sigma_{T-t}^{-4}$.
Combining these bounds with Lemma~\ref{lemma:approximation:low} in the two oracle inequalities, as in the proof of Theorem~\ref{theorem:rate:guidance}, gives
\begin{align*}
\bbE\big\{\|\hat {h}_{t}^{\lambda}-h_{t}^{*}\|_{L^{2}(p_{T-t})}^{2}\big\}
&\leq Cd^{2}\sigma_{T-t}^{-8}n^{-4/(d^{*}+8)}\log^{4}n\,,\\
\bbE\big\{\|\nabla\hat {h}_{t}^{\lambda}-\nabla h_{t}^{*}\|_{L^{2}(p_{T-t})}^{2}\big\}
&\leq Cd^{2}\sigma_{T-t}^{-8}n^{-2/(d^{*}+8)}\log^{4}n\,.
\end{align*}
Finally, Propositions~\ref{proposition:h:function:low} and~\ref{proposition:regularity:h:function:low} give $h_{t}^{*}\geq\underline {B}$ and $|\nabla h_{t}^{*}|_{2}\leq2\sqrt{d^{*}}\,\bar {B}\sigma_{T-t}^{-2}$. Using $\hat {h}_{t}^{\lambda}\geq\underline {B}/2$ in the log-gradient decomposition from Theorem~\ref{theorem:rate:guidance}, we obtain
\begin{align*}
\bbE\big\{\|\nabla\log\hat {h}_{t}^{\lambda}-\nabla\log h_{t}^{*}\|_{L^{2}(p_{T-t})}^{2}\big\}
&\leq Cd^{2}\big\{\sigma_{T-t}^{-8}n^{-2/(d^{*}+8)}+\sigma_{T-t}^{-12}n^{-4/(d^{*}+8)}\big\}\log^{4} n\\
&\leq Cd^{2}\sigma_{T-t}^{-8}n^{-2/(d^{*}+8)}\log^{4}n\,,
\end{align*}
where the last inequality uses $n^{-2/(d^{*}+8)}\leq\sigma_{T-t}^{4}$ for sufficiently large $n$. This completes the proof.
\end{proof}

\begin{proof}[Proof of Corollary~\ref{corollary:rate:controllable:diffusion:low}]
The stated sample-size bound ensures that Theorem~\ref{theorem:rate:guidance:low} applies at every grid time.
Under Assumptions~\ref{assumption:bounded:weight} and~\ref{assumption:intrinsic}, both $p_{0}$ and $q_{0}$ are supported in $\{\vx:|\vx|_{2}\leq\sqrt{d^{*}}\}\subseteq\{\vx:|\vx|_{2}\leq\sqrt {d}\}$. Thus the initial-moment and score-regularity bounds used in the proof of Theorem~\ref{theorem:rate:controllable:diffusion} remain valid. With the stated parameters, repeating the argument of Corollary~\ref{corollary:error:kl} using Theorem~\ref{theorem:rate:guidance:low} gives
\[
\bbE\big\{\mathrm{TV}(q_{T_{0}},\what {q}_{T-T_{0}})\big\}\leq Cd\varepsilon\log^{2}(\varepsilon^{-1})\,.
\]
Since $\mu_{T_{0}}^{2}=1-\varepsilon\geq1-e^{-1}$ and $R^{2}\leq Cd^{2}\log(\varepsilon^{-1})$, the early-stopping, truncation, and sampling estimates in the proof of Theorem~\ref{theorem:rate:controllable:diffusion} yield
\[
\begin{aligned}
&\bbE\big[ W_{2}^{2}\big\{q_{0},(\calM\circ\calT_{R})_{\sharp}\what {q}_{T-T_{0}}\big\}\big]\leq C\big\{d\varepsilon+R^{2}d\varepsilon\log^{2}(\varepsilon^{-1})\big\}\leq Cd^{3}\varepsilon\log^{3}(\varepsilon^{-1})\,.
\end{aligned}
\]
Here $C$ depends only on $d^{*}$, $\bar {B}$, and $\underline {B}$. This completes the proof.
\end{proof}

\section{Auxilary Lemmas}
\label{appendix:auxiliary:lemmas}

\begin{lemma}
\label{lemma:tail:proba:t}
Suppose Assumptions~\ref{assumption:bounded:support} and~\ref{assumption:bounded:weight} hold. Let $t>0$ and $\mX_{t}\sim p_{t}$. Then for each $\xi>0$,
\begin{equation*}
\pr\{|\mX_{t}|_{2}\geq \xi\}\leq 2^{d+1}\exp\bigg(-\frac{\xi^{2}}{4d\mu_{t}^{2}+8\sigma_{t}^{2}}\bigg)\,.
\end{equation*}
\end{lemma}

\begin{proof}[Proof of Lemma~\ref{lemma:tail:proba:t}]
By the two assumptions, $\supp(p_{0})=\supp(q_{0})$ lies in the unit hypercube, so
\begin{equation}\label{eq:lemma:section:proof:warm:start:0:1}
\bbE\bigg\{\exp\bigg(\frac{|\mu_{t}\mX_{0}|_{2}^{2}}{2d\mu_{t}^{2}}\bigg)\bigg\}\leq 2\,.
\end{equation}
Let $\vepsilon\sim\mathcal{N}(\bzero,\mI_{d})$. Then it follows that 
\begin{align}
\bbE\bigg\{\exp\bigg(\frac{|\sigma_{t}\vepsilon|_{2}^{2}}{4\sigma_{t}^{2}}\bigg)\bigg\}
&=\bbE\bigg\{\exp\bigg(\frac{|\vepsilon|_{2}^{2}}{4}\bigg)\bigg\} \notag \\
&=(2\pi)^{-d/2}\int\exp\bigg(\frac{|\vepsilon|_{2}^{2}}{4}\bigg)\exp\bigg(-\frac{|\vepsilon|_{2}^{2}}{2}\bigg)\d\vepsilon \notag \\
&=(2\pi)^{-d/2}\int\exp\bigg(-\frac{|\vepsilon|_{2}^{2}}{4}\bigg)\d\vepsilon\leq 2^{d}\,. \label{eq:lemma:section:proof:warm:start:0:2}
\end{align}
Notice that $\mX_{t}\stackrel{\d}{=}\mu_{t}\mX_{0}+\sigma_{t}\vepsilon$, where $\mX_{0}\sim p_{0}$ and $\vepsilon\sim \mathcal{N}(\bzero,\mI_{d})$ are independent. Therefore, 
\begin{align}
&\bbE\bigg\{\exp\bigg(\frac{|\mX_{t}|_{2}^{2}}{4d\mu_{t}^{2}+8\sigma_{t}^{2}}\bigg)\bigg\}
=\bbE\bigg\{\exp\bigg(\frac{|\mu_{t}\mX_{0}+\sigma_{t}\vepsilon|_{2}^{2}}{4d\mu_{t}^{2}+8\sigma_{t}^{2}}\bigg)\bigg\} \label{eq:lemma:section:proof:warm:start:0:3} \\
&~~~~~~~\leq\bbE\bigg\{\exp\bigg(\frac{|\mu_{t}\mX_{0}|_{2}^{2}}{2d\mu_{t}^{2}+4\sigma_{t}^{2}}+\frac{|\sigma_{t}\vepsilon|_{2}^{2}}{2d\mu_{t}^{2}+4\sigma_{t}^{2}}\bigg)\bigg\} \notag \\
&~~~~~~~\leq\bbE\bigg\{\exp\bigg(\frac{|\mu_{t}\mX_{0}|_{2}^{2}}{2d\mu_{t}^{2}+4\sigma_{t}^{2}}\bigg)\bigg\}\bbE\bigg\{\exp\bigg(\frac{|\sigma_{t}\vepsilon|_{2}^{2}}{2d\mu_{t}^{2}+4\sigma_{t}^{2}}\bigg)\bigg\} \notag \\
&~~~~~~~\leq\bbE\bigg\{\exp\bigg(\frac{|\mu_{t}\mX_{0}|_{2}^{2}}{2d\mu_{t}^{2}}\bigg)\bigg\}\bbE\bigg\{\exp\bigg(\frac{|\sigma_{t}\vepsilon|_{2}^{2}}{4\sigma_{t}^{2}}\bigg)\bigg\}\leq 2^{d+1}\,, \notag
\end{align}
where the the first inequality follows from Cauchy-Schwarz inequality, the second inequality holds from the independence of $\mX_{0}$ and $\vepsilon$, and the last inequality is due to~\eqref{eq:lemma:section:proof:warm:start:0:1} and~\eqref{eq:lemma:section:proof:warm:start:0:2}. Then we aim to bound the tail probability. For each $\xi>0$, we have 
\begin{align*}
&\pr\{|\mX_{t}|_{2}\geq \xi\}
=\pr\bigg\{\frac{|\mX_{t}|_{2}^{2}}{4d\mu_{t}^{2}+8\sigma_{t}^{2}}\geq\frac{\xi^{2}}{4d\mu_{t}^{2}+8\sigma_{t}^{2}}\bigg\} \\
&~~~~~=\pr\bigg\{\exp\bigg(\frac{|\mX_{t}|_{2}^{2}}{4d\mu_{t}^{2}+8\sigma_{t}^{2}}\bigg)\geq\exp\bigg(\frac{\xi^{2}}{4d\mu_{t}^{2}+8\sigma_{t}^{2}}\bigg)\bigg\} \\
&~~~~~\leq\exp\bigg(-\frac{\xi^{2}}{4d\mu_{t}^{2}+8\sigma_{t}^{2}}\bigg)\bbE\bigg\{\exp\bigg(\frac{|\mX_{t}|_{2}^{2}}{4d\mu_{t}^{2}+8\sigma_{t}^{2}}\bigg)\bigg\} \leq 2^{d+1}\exp\bigg(-\frac{\xi^{2}}{4d\mu_{t}^{2}+8\sigma_{t}^{2}}\bigg)\,,
\end{align*}
where the first inequality invokes Markov's inequality, and the last inequality is due to~\eqref{eq:lemma:section:proof:warm:start:0:3}. This completes the proof.
\end{proof}

\begin{corollary}
\label{corollary:tail:proba:t}
Suppose Assumptions~\ref{assumption:bounded:support} and~\ref{assumption:bounded:weight} hold. Let $t>0$ and $\mZ_{t}\sim q_{t}$. Then for each $\xi>0$,
\begin{equation*}
\pr\{|\mZ_{t}|_{2}\geq \xi\}\leq 2^{d+1}\exp\bigg(-\frac{\xi^{2}}{4d\mu_{t}^{2}+8\sigma_{t}^{2}}\bigg)\,.
\end{equation*}
\end{corollary}

\begin{proof}[Proof of Corollary~\ref{corollary:tail:proba:t}]
Under Assumptions~\ref{assumption:bounded:support} and~\ref{assumption:bounded:weight}, $\supp(q_{0})=\supp(p_{0})$. The same argument as Lemma~\ref{lemma:tail:proba:t} completes the proof.
\end{proof}

\begin{lemma}
\label{lemma:expectation:t}
Suppose Assumptions~\ref{assumption:bounded:support} and~\ref{assumption:bounded:weight} hold. Let $t\geq0$ and $\mZ_{t}\sim q_{t}$. Then
\begin{equation*}
\bbE\big(|\mZ_{t}|_{2}^{4}\big)\lesssim d^{2}\,.
\end{equation*}
\end{lemma}

\begin{proof}[Proof of Lemma~\ref{lemma:expectation:t}]
Let $\vepsilon\sim\mathcal{N}(\bzero,\mI_{d})$. It is straightforward that 
\begin{equation*}
\bbE\big(|\vepsilon|_{2}^{4}\big)=d(d+2)\leq(d+4)^2\,.
\end{equation*}
Since $\mZ_{t}\stackrel{\d}{=}\mu_{t}\mZ_{0}+\sigma_{t}\vepsilon$ with $\mZ_{0}\sim q_{0}$ independent of $\vepsilon$, it follows from the triangular inequality that 
\begin{align*}
\bbE\big(|\mZ_{t}|_{2}^{4}\big)
&\leq 8\mu_{t}^{4}\bbE\big(|\mZ_{0}|_{2}^{4}\big)+8\sigma_{t}^{4}\bbE\big(|\vepsilon|_{2}^{4}\big) \leq 8\{d^{2}+(d+4)^{2}\}\,,
\end{align*}
where we used the fact that $\mu_{t},\sigma_{t}\leq 1$, and $\supp(q_{0})=\supp(p_{0})$ under Assumptions~\ref{assumption:bounded:support} and~\ref{assumption:bounded:weight}. This completes the proof.
\end{proof}


\end{document}